\documentclass[twoside,11pt]{article}

\usepackage[abbrvbib, preprint]{jmlr2e}

\usepackage{amsmath,amsfonts}
\usepackage{mathtools}
\usepackage{enumerate,color,xcolor}
\usepackage{subfigure}
\usepackage{url}
\usepackage[textsize=footnotesize]{todonotes}

\usepackage{natbib}

\usepackage{graphicx}      
\usepackage{subfigure}
\usepackage{url}

\newcommand{\rev}[1]{\textcolor{black}{#1}}
\usepackage{comment}

\newcommand{\ones}{\textbf{1}}

\newcommand{\beq}{\begin{eqnarray}} 
\newcommand{\eeq}{\end{eqnarray}}
\newcommand{\beqs}{\begin{eqnarray*}} 
\newcommand{\eeqs}{\end{eqnarray*}}
\newcommand{\R}{\mathbb{R}}

\renewcommand{\k}[1]{#1^{(k)}}
\newcommand{\km}[1]{#1^{(k-1)}}
\newcommand{\kp}[1]{#1^{(k+1)}}
\ShortHeadings{Decentralized Stochastic Gradient Langevin Algorithms}{G\"{u}rb\"{u}zbalaban, Gao, Hu and Zhu}
\firstpageno{1}

\begin{document}
\title{Decentralized Stochastic Gradient Langevin\\ Dynamics and Hamiltonian Monte Carlo}

\author{\name Mert G\"{u}rb\"{u}zbalaban \email mg1366@rutgers.edu \\
       \addr Department of Management Science and Information Systems \\
       Rutgers Business School\\
       Piscataway, NJ 08854, United States of America
       \AND
       \name Xuefeng Gao* \email xfgao@se.cuhk.edu.hk \\
       \addr Department of Systems Engineering and Engineering Management\\
       The Chinese University of Hong Kong\\
       Shatin, N.T., Hong Kong, China
       \AND
       \name Yuanhan Hu* \email yh586@scarletmail.rutgers.edu \\
       \addr 
       Department of Management Science and Information Systems \\
       Rutgers Business School\\
       Piscataway, NJ 08854, United States of America
       \AND
       \name Lingjiong Zhu* \email zhu@math.fsu.edu \\
       \addr Department of Mathematics \\
       Florida State University \\
       Tallahassee, FL 32306, United States of America\\
       ~\\
       \name * \addr The authors are in alphabetical order.\\
       }

\editor{}

\maketitle

\begin{abstract}
Stochastic gradient Langevin dynamics (SGLD) and stochastic gradient Hamiltonian Monte Carlo (SGHMC) are two popular Markov Chain Monte Carlo (MCMC) algorithms for Bayesian inference that can scale to large datasets, allowing to sample from the posterior distribution of the parameters of a statistical model given the input data and the prior distribution over the model parameters. However, these algorithms do not apply to the decentralized learning setting, when a network of agents are working collaboratively to learn the parameters of a statistical model without sharing their individual data due to privacy reasons or communication constraints. We study two algorithms: Decentralized SGLD (DE-SGLD) and Decentralized SGHMC (DE-SGHMC) which are adaptations of SGLD and SGHMC methods that allow scaleable Bayesian inference in the decentralized setting for large datasets. We show that when the posterior distribution is strongly log-concave and smooth, the iterates of these algorithms converge linearly to a neighborhood of the target distribution in the 2-Wasserstein distance if their parameters are selected appropriately. We illustrate the efficiency of our algorithms on decentralized Bayesian linear regression and Bayesian logistic regression problems.
\end{abstract}




%

\section{Introduction}
Recent decades have witnessed the era of big data, and there has been an exponential growth in the amount of data collected and stored with ever-increasing rates. Since the rate at which data is generated is often outpacing our ability to analyze it in terms of computational resources at hand, there has been a lot of recent interests for developing scaleable machine learning algorithms which are efficient on large datasets. 

In the modern world, digital devices such as smart phones, tablets, wearables, sensors or video cameras are major sources of data generation. Often these devices are connected over a communication network (such as a wireless network or a sensor network) that has a high latency or a limited bandwidth.   
Because of communication constraints and privacy constraints, gathering all these data for centralized processing is often impractical or infeasible. Decentralized machine learning algorithms have received a lot of attention for such applications where agents can collaboratively learn a predictive model without sharing their own data but sharing only their local models with their immediate neighbors at some frequency to generate a global model; 
see e.g. \citet{arjevani2020ideal,he2018cola,hendrikx2019accelerated,kungurtsev2020}. 

A number of approaches for scaleable decentralized learning have been proposed in the literature such as decentralized stochastic approximation and optimization algorithms \citep{gorbunov2019optimal,nedic2020distributed,scaman2019optimal,uribe2017optimal} or decentralized maximum-likelihood estimation approaches \citep{blatt2004distributed,rabbat2004decentralized}. However, these approaches are optimization-based or in the maximum-likelihood settings, and therefore lead to point estimates for the model parameters to be learned. On the other hand, Bayesian methods allow a characterization of the full posterior distribution over the parameters, and therefore can provide a more detailed grasp of uncertainties that are part of the learning process and offer robustness to overfitting. There are a number of scaleable Bayesian methods in the literature based on variational inference methods \citep{sato2001online,hoffman2010online,lin2013online}, Bayesian coreset methods \citep{huggins2016coresets,coreset-19} and Markov Chain Monte Carlo (MCMC) based methods including
Stochastic Gradient Langevin Dynamics (SGLD) \citep{welling2011bayesian}, Stochastic Gradient Hamiltonian Monte Carlo (SGHMC) \citep{chen2014stochastic,ZXG2018} and their variants that can handle streaming data \citep{broderick2013streaming}. There are also versions of these methods such as consensus Monte Carlo methods which distribute and parallelize the computations needed for Monte Carlo sampling across many computational nodes on a cluster \citep{d-sgld,xu2014distributed,variational-mcmc,broderick2013streaming}, however none of these methods are applicable to the decentralized setting either because they need to move the data to a centralized location or because they require a global computational unit with which each learning agent is in communication or is the main thread on a multi-threaded computer which is not applicable to decentralized learning applications. In this paper, we consider two algorithms DE-SGLD and DE-SGHMC which are adaptations of the SGLD and SGHMC algorithms to the decentralized setting and show that they can be both theoretically and practically efficient for sampling from the posterior distribution when the density of the target distribution $\pi(x) \propto e^{-f(x)}$ is strongly log-concave (i.e. $f$ is strongly convex) and $f$ is smooth.


Before introducing the DE-SGLD algorithm, we consider the problem of decentralized Bayesian inference: We have $N$ agents connected over a network $\mathcal{G}=(\mathcal{V},\mathcal{E})$ where $\mathcal{V}=\{1,2,\dots,N\}$ represents the agents and $\mathcal{E}\subseteq \mathcal{V}\times \mathcal{V}$ is the set of edges; i.e. $i$ and $j$ are connected if $(i,j) \in \mathcal{E}$ where the network is undirected, i.e. $(i,j) \in \mathcal{E}$ then $(j,i) \in \mathcal{E}$. Let $A = [a_1,\dots,a_n]$ be a dataset consisting of $n$ independent and identically distributed (i.i.d.) data vectors sampled from a parametrized distribution $p(A|x)$ where the parameter $x\in \mathbb{R}^d$ has a common prior distribution $p(x)$. Due to the decentralization in the data collection, each agent $i$ possesses a subset $A_i$ of the data where $A_i = \{a_1^i, a_2^i, \dots, a_{n_i}^i \}$ and $n_i$ is the number of samples of the agent $i$. The data is held disjointly over agents; i.e. $A = \cup_i A_i$ with $A_i \cap A_j = \emptyset$ for $j\neq i$. The goal is to
sample from the posterior distribution $p(x|A) \propto p(A|x) p(x)$.
Since the data points are independent, the log-likelihood function will be additive; $\log p(A|x) = \sum_{i=1}^N \sum_{j=1}^{n_i} \log p(a_j^i | x)$. Thus, if we set
\begin{equation} 
f(x) := \sum_{i=1}^N f_i(x), \quad f_i(x) := -\sum_{j=1}^{n_i}\log p\left(a_j^i | x\right) - \frac{1}{N} \log p(x),
\label{def-bayesian-inf}
\end{equation}
the aim is to sample from the posterior distribution 
with density $\pi(x):= p(x|A)\propto e^{-f(x)}$. The functions $f_i(x)$ are called ``component functions" where $f_i(x)$ is associated to the local data of agent $i$ and is only accessible by the agent $i$. Clearly, different choices of the log-likelihood function and therefore the component functions result in different problems. In particular, this framework covers many Bayesian inference problems such as Bayesian linear regression \citep{hoff2009first}, Bayesian logistic regression \citep{hoff2009first}, Bayesian principal component analysis \citep{dubey2016variance} or Bayesian deep learning \citep{wang2016towards,polson2017}. 

Let $x_i^{(k)}$ denote the local variable of node $i$ at iteration $k$. The decentralized SGLD (DE-SGLD) algorithm (previously considered in  \citet{swenson2020distributed} in the non-convex global optimization setting) consists of a weighted averaging with the local variables $x_j^{(k)}$ of node $i$'s immediate neighbors $j\in\Omega_{i}:=\{j : (i,j) \in \mathcal{G}\}$ as well as a stochastic gradient step over the node's component function $f_i(x)$, i.e. 
\begin{equation}\label{eqn:1}
x_{i}^{(k+1)}=\sum_{j\in\Omega_{i}}W_{ij}x_{j}^{(k)}-\eta\tilde\nabla f_{i} \left(x_{i}^{(k)}\right)+\sqrt{2\eta}w_{i}^{(k+1)},
\end{equation} 
where $\eta>0$ is the stepsize, $W_{ij}$ are the entries of a doubly stochastic weight matrix $W$ with $W_{ij}>0$ only if $i$ is connected to $j$, $w_{i}^{(k)}$
are independent and identically distributed (i.i.d.) Gaussian random variables with zero mean and identity covariance matrix for every $i$ and $k$, and $\tilde\nabla f_{i} \left(x_{i}^{(k)}\right)$ is an unbiased stochastic estimate of the deterministic gradient $\nabla f_{i} \left(x_{i}^{(k)}\right)$ with a bounded variance (see \eqref{gradient:noise:i} for more details). When the number of data points $n_i$ is large, stochastic estimates $\tilde\nabla f_i(x)$ are cheaper to compute compared to actual gradients $\nabla f_i(x)$ and can for instance be estimated from a minibatch of data, i.e. from randomly selected smaller subsets of data. This allows the DE-SGLD method to be scaleable to big data settings when $n_i$ can be large. When gradients are deterministic, DE-SGLD algorithm reduces to the decentralized Langevin algorithm previously considered and studied in \citet{kungurtsev2020}. Without the Gaussian noise, the iterations are also equivalent to the decentralized stochastic gradient algorithm \citep{swenson-journal,robust-network-asg} which has its origins in the decentralized gradient descent (DGD) methods introduced in \citet{nedic2009distributed}.



\textbf{Contributions.} 
In this paper, our contributions can be summarized as follows: 

First, we give non-asymptotic performance guarantees for DE-SGLD when each of the components $f_i(x)$ is smooth and strongly convex in which case the target distribution has density $\pi(x) \propto e^{-f(x)}$ that is strongly log-concave (i.e. $f$ is strongly convex) and $f$ is smooth. More specifically, we provide an explicit upper bound on the Wasserstein distance between the target distribution $\pi(x)$ and the distribution of the iterate  $x_i^{(k)}$ of node $i$. Our results show that the distribution of the iterates $x_i^{(k)}$ converges to a neighborhood of the posterior distribution $\pi(x)$ linearly (geometrically fast in $k$) in the Wasserstein metric with a properly chosen stepsize. We also provide explicit bounds on the size of this neighborhood as a function 
of the noise level $\sigma^2$ in the stochastic gradients, the number of agents $N$ and the dimension $d$. We can also show similar results for the averaged iterates $\bar{x}^{(k)} = \frac{1}{N}\sum_{i=1}^N x_i^{(k)}$. Our proof technique relies on analyzing DE-SGLD as a perturbed version of the Euler-Maruyama discretization of the overdamped Langevin diffusion (properly defined in Section~\ref{sec:background}) and use the fact that this diffusion admits the posterior distribution with density $\pi(x) \propto e^{-f(x)}$ as the stationary distribution 
where the perturbation effect is due to the stochasticity of the gradients and due to the ``network effect" where agents are only able to communicate with their immediate neighbours. For achieving the results, we first derive a uniform $L_2$ bound on the gradients (Lemma~\ref{lem:0}) as well as a uniform $L_2$ bound on the deviation of the iterates $\bar{x}_i^{(k)}$ from their mean $\bar{x}^{(k)}$ over the agents (Lemma~\ref{lem:1}). Then, we derive an $L^{2}$ bound
on the error between the average of gradients $\frac{1}{N}\sum_{i=1}^{N}\nabla f_{i}\left(x_{i}^{(k)}\right)$ and the scaled gradient of the average $\frac{1}{N}\nabla f\left(\bar{x}^{(k)}\right)$ (Lemma~\ref{lem:2}). Finally, we control the error between the mean iterates and the discretization of the overdamped diffusion (Lemma~\ref{lem:3}) and build on the existing results which characterizes the Wasserstein distance between the overdamped diffusion and its discretization. Putting everything together, we obtain our main result for DE-SGLD (Theorem~\ref{thm:overdamped}). 

Second, we propose a new algorithm decentralized SGHMC (DE-SGHMC) which can be viewed as the decentralized version of the SGHMC algorithm. In centralized settings, it is known that SGHMC algorithm can be faster than the SGLD algorithm both in practice and in theory \citep{GGZ,chen2014stochastic}. The underlying reason is  that SGHMC is based on a discretization of the (underdamped) inertial Langevin diffusion which can converge to its equilibrium faster than the overdamped diffusion due to a momentum-based acceleration effect \citep{GGZ,Eberle}. This effect is analogous to the fact that momentum-based optimization methods can accelerate gradient descent \citep{polyak1987introduction,nesterov1983method,su2016differential}. We show that with proper choice of the stepsize and momentum parameters, the distribution of the DE-SGHMC iterates $x_i^{(k)}$ will converge to a neighborhood of the posterior distribution $\pi(x)$ linearly (in $k$) in the Wasserstein metric
(Theorem~\ref{thm:underdamped}). To our knowledge, these are first non-asymptotic performance guarantees for SGHMC methods in the decentralized setting. The approach we take is analogous to our analysis of the DE-SGLD however obtaining stability (uniform $L_2$) bounds on the iterates requires significantly more work. For this purpose, we develop a novel analysis where we show that the DE-SGHMC iterates can be viewed as a noisy version of Polyak's (deterministic) heavy-ball method \citep{polyak1987introduction}; the noise comes from stochasticity of the gradients  (which is proportional to the stepsize $\eta$), the injected Gaussian noise (which is proportional to $\sqrt{\eta}$) and the network effect where iterates can only access information from their neighbors. 
When the stepsize $\eta$ is sufficiently small, the Gaussian noise dominates the stochastic gradient noise and therefore existing analysis for stochastic heavy-ball methods \citep{can2019accelerated,kuru-privacy} are not directly applicable to obtain stability estimates (i.e. uniform $L_2$ bounds) when $\eta$ is sufficiently small. Our analysis relies on a careful choice of the Lyapunov function and obtaining sufficient conditions on the stepsize and the momentum parameters of the DE-SGHMC algorithm to guarantee stability (Lemma~\ref{lemma-l2-underdamped}). 
As a by-product, our results contribute to the growing literature about the stability of the heavy-ball methods where it has been observed repeatedly that optimization methods such as heavy-ball and Nesterov's accelerated gradient methods are more sensitive to noise in the iterations compared to gradient descent methods \citep{robust-asg,robust-network-asg,flammarion2015averaging,devolder2014first,can2019accelerated}. Recent literature focused on the amount of noise heavy-ball methods can tolerate before they diverge and on their convergence rate subject to noise and perturbations \citep{can2019accelerated,flammarion2015averaging,liu2020improved,kuru-privacy}. Our analysis in the proof of Lemma~\ref{lemma-l2-underdamped} provides sufficient conditions for noisy heavy-ball iterations to be stable when subject to noise that is on the order of the square root of the stepsize. 

Finally, we provide numerical experiments that illustrate our theory and showcase the practical performance of the DE-SGLD and DE-SGHMC algorithms: We show on Bayesian linear regression and Bayesian logistic regression tasks that our method allows each agent to sample from the posterior distribution efficiently without communicating local data.

\textbf{Related literature.} Decentralized optimization has been studied in the literature in the last few decades, at least going back to the seminal works of \citet{bertsekas1989parallel,tsitsiklis1984problems} which studied minimization of objective functions 
when the parameter vector can be decentralized. There has also been a growing literature and a lot of recent interest on decentralized optimization with first-order methods for both deterministic and stochastic optimization; See e.g. \citet{d-sgd,robust-network-asg,arjevani2020ideal,can2019decentralized,pu2020asymptotic} and also the surveys \citet{nedic2020distributed,YANG2019278}. Among the papers published in this area, \citet{swenson2020distributed,swenson-journal}
are most relevant to our paper, where the authors study a class of algorithms including DE-SGLD and show that DE-SGLD iterates with a particular decaying stepsize schedule converge in probability to the set of global minima for non-convex objectives under some assumptions. Momentum-based acceleration techniques based on heavy-ball method \citep{xin2019distributed} and Nesterov's accelerated gradient method have also been studied for solving optimization problems in the decentralized setting \citep{robust-network-asg,arjevani2020ideal,qu2016accelerated,pmlr-v108-xu20b}, we refer the readers to \citet{nedic2020distributed} for a survey in decentralized optimization. However, these papers are focused on solving optimization problems and the results do not apply to our setting where we are interested in sampling from the posterior distribution. 

There are also a number of papers for distributed Bayesian inference \rev{based on data-parallel MCMC algorithms \citep{GCWG2015,neiswanger2014asymptotically,xu2014distributed, scott2016bayes, variational-mcmc,scott2017comparing,rendell2020global,ahn2015large} where the computations are parallelized in a distributed computing environment}, however these papers are not applicable to the decentralized setting either. The variational inference methods which approximate the posterior distribution with a tractable distribution in the exponential family can be applied in the decentralized setting \citep{CH2014,lalitha2019decentralized} where agents average the parameters of their local parametrized distribution that estimates the posterior distribution, however to our knowledge, 
convergence rate guarantees to a posterior distribution for such approaches in the decentralized setting are not provided except the special case when the posterior distribution is in the exponential family \citep{lalitha2019decentralized}. \rev{There are also other parallel MCMC techniques \citep{wang2013parallelizing,neiswanger2014asymptotically,wang2015parallelizing,chowdhury2018parallel,nishihara2014parallel} which require a central node to aggregate the samples generated at each computational node to estimate the posterior distribution; these methods are also not directly applicable to the decentralized setting.}

Finally, very recently \citet{kungurtsev2020} showed that in the special case when the gradients are deterministic (i.e. when $\sigma=0$), DE-SGLD algorithm converges to the target distribution $\pi(x)$ with rate $\mathcal{O}(\frac{1}{\sqrt{k}})$ for decaying stepsize $\alpha_k = \frac{1}{k}$ in the Wasserstein metric for strongly convex and smooth $f$ with bounded gradients. Since strongly convex functions on $\mathbb{R}^d$ cannot have bounded gradients, these results are not applicable to problems we consider in this paper.
In a concurrent work,
\citet{PBGG2020} studied a Bayesian learning algorithm based on the decentralized Langevin dynamics in a non-convex setting.
They obtained theoretical convergence guarantees in KL-divergence
and evaluated the proposed algorithm on a wide variety
of machine learning tasks. \rev{In another recent work, \cite{cadena2021stochastic} proposed a modified Langevin dynamics algorithm for sensor networks. This algorithm can be implemented in a decentralized manner, where each sensor communicates with a randomly selected subset of sensors either via direct links or via multi-hop mechanism. The authors also show that when the gradient of the logarithm of the target density is bounded and Lipschitz, the proposed algorithm converges to the true centralized posterior distribution for networks where the communication delays are bounded.}

\section{Preliminaries and Background}\label{sec:background}



\paragraph{Langevin algorithms.} \textit{Langevin algorithms} are core MCMC methods in statistics that allow one to sample from a given density $\pi(x)$ of interest. The classical Langevin Monte Carlo algorithm is based on the
\emph{overdamped (or first-order) Langevin diffusion}; see e.g. \citet{Dalalyan,DM2016,DM2017,DK2017}:
\begin{equation}\label{eq:overdamped-2}
dX(t)=-\nabla f(X(t))dt+\sqrt{2}dW_{t},
\end{equation}
where $f:\mathbb{R}^{d}\rightarrow\mathbb{R}$
and $W_{t}$ is a standard $d$-dimensional Brownian motion that starts at zero
at time zero. Under some mild assumptions on $f$, the diffusion \eqref{eq:overdamped-2} admits a unique stationary distribution with the density $\pi(x) \propto e^{-f(x)}$,
also known as the \emph{Gibbs distribution} \citep{pavliotis2014stochastic}. For computational purposes, this diffusion is simulated by considering its discretization. 
Although various discretization schemes are proposed, Euler-Maruyama discretization is the simplest one:
\begin{equation}\label{discrete:overdamped}
x_{k+1}=x_{k}-\eta \nabla f(x_k)+\sqrt{2\eta}w_{k}\,,
\end{equation}
where $\eta>0$ is the stepsize parameter, and $w_k \in \R^d$ is a sequence of i.i.d. standard Gaussian random vectors $\mathcal{N}(0,I_{d})$. But then the discretized chain \eqref{discrete:overdamped} does not converge to the target $\pi$ and has a bias that needs to be properly characterized to provide performance guarantees \citep{DK2017}.\footnote{In principle, Metropolis-Hasting correction step can be employed to correct for the discretization errors, however for large-scale datasets, this correction step is computationally expensive and thus it is often not employed \citep{dalalyan2018kinetic,DK2017,teh2016consistency}. For this reason, we will not consider Metropolis-Hasting steps in our algorithms and analyses.}
There has been growing recent interest in the non-asymptotic analysis of discretized Langevin diffusions \eqref{discrete:overdamped}, motivated by applications to large-scale data analysis and Bayesian inference. The discretized Langevin diffusions admit convergence guarantees to a stationary distribution in a variety of metrics and under various assumptions on $f$; {see e.g. \citet{Dalalyan,DM2017, DM2016,BEL2015,CB2018,EH2020,DK2017,Barkhagen2021,Raginsky,xu2018global,Chau2019,Zhang2019}.}



On the other hand, one can also design sampling algorithms based on 
 the underdamped (a.k.a. inertial or kinetic) Langevin diffusion given by the SDE; see e.g. \citet{Cheng,cheng-nonconvex,dalalyan2018kinetic,GGZ,GGZ2,Ma2019,Akyildiz2020,JianfengLu,ZXG2019}:
\begin{align}
&dV(t)=-\gamma V(t)dt- \nabla f (X(t))dt+\sqrt{2\gamma }dW_{t}, \label{eq:VL}
\\
&dX(t)=V(t)dt, \label{eq:XL}
\end{align}
where $\gamma>0$ is the friction coefficient, $X(t),V(t) \in \mathbb{R}^d$ models the position and the momentum of a particle moving in a field of force (described by the gradient of $f$) plus a random (thermal) force described by the Brownian noise, and $W_{t}$ is a standard $d$-dimensional Brownian
motion that starts at zero at time zero. It is known that under some mild assumptions on $f$, the Markov process $(X(t), V(t))_{t\geq 0}$ is ergodic and admits a unique stationary distribution $\pi$ with density
$\pi(x,v) \propto \exp\left(- \left(\frac{1}{2} \Vert v\Vert^2 +f(x)\right) \right)$ \citep{pavliotis2014stochastic}.
Hence, the $x$-marginal distribution of the stationary distribution with the density $\pi(x,v)$ is exactly the invariant distribution of the overdamped Langevin diffusion. For approximate sampling, 
various discretization schemes of \eqref{eq:VL}-\eqref{eq:XL}
have been used in the literature; see e.g. 
\citet{cheng-nonconvex,teh2016consistency,chen2016bridging,carin-2015-langevin-integrators}.

\paragraph{Decentralized setting.} Agents are connected over a network $\mathcal{G}=(V,E)$
where $f_i:\mathbb{R}^d\to \mathbb{R}$ is the local objective of the agent $i$ and we assume $\mathcal{G}$ is connected. Agents can only communicate with immediate neighbors using links defined by the edge set $\mathcal{E}$. 
We associate this network with an $N\times N$ symmetric, doubly stochastic\footnote{A square matrix $A\in\mathbb{R}^{N\times N}$ is called \emph{doubly stochastic} if its entries $A_{ij}$ are non-negative and if its rows and columns all sum up to $1$, i.e. if $\sum_{j=1}^N A_{ij} = 1$ for all $i=1,2,\dots,N$ and $\sum_{i=1}^N A_{ij}=1$ for all $j=1,2,\dots,N$.} weight matrix $W$. 
We have $W_{ij}=W_{ji}>0$ if $\{i,j\}\in E$ and $i\neq j$,
and $W_{ij}=W_{ji}=0$ if $\{i,j\}\not\in E$ and $i\neq j$,
and finally $W_{ii}=1-\sum_{j\neq i}W_{ij}>0$
for every $1\leq i\leq N$.
The eigenvalues of $W$ ordered
in a descending manner satisfy:
\begin{equation}
1=\lambda_{1}^{W}>\lambda_{2}^{W}\geq\cdots\geq\lambda_{N}^{W}>-1\,,
\label{ineq-network-spectrum}
\end{equation}
with $W\ones = \ones$ where $\ones$ is a vector of length $N$ with each entry equal to one. For any connected $\mathcal{G}$, there is always a choice of $W$ that satisfies the eigenvalue conditions \citep{can2019accelerated,boyd2006randomized}. A possible choice is the Metropolis weights \citep{xiao2006space,olshevsky2017linear} where $W_{ij} = \frac{1}{\max(d_i, d_j)}$ if $(i,j) \in \mathcal{E}$ where $d_i$ is the degree (number of neighbors) of the node $i$. The mixing matrix $W$ can also be chosen in many other ways \citep{can2019decentralized,boyd2006randomized}. In this paper, we will assume that $W$ is given and fixed.

Our objective is to sample from a target distribution with density $\pi(x) \propto e^{- f(x)}$ on $\mathbb{R}^d$ where
\begin{equation}\label{defn:f}
f(x):= \sum_{i=1}^{N}f_{i}(x).
\end{equation}

The agents can only pass vectors between their neighbors (not matrices) as the communication is typically more expensive than local computations in modern applications \citep{woodruff2017distributed}. Throughout this paper, we assume $f_i \in \mathcal{S}_{\mu,L}({\mathbb{R}}^d)$ for every $i=1,2,\dots,N,$\footnote{Our results in this paper would also hold if $f_i \in \mathcal{S}_{\mu_i,L_i}({\mathbb{R}}^d)$ and one considers $\mu=\min_i \mu_i$ and $L=\max_i L_i$ in our main theorems.} where $\mathcal{S}_{\mu,L}(\mathbb{R}^{d})$ denotes the set
of functions from $\mathbb{R}^{d}$ to $\mathbb{R}$ that are
$\mu$-strongly convex and $L$-smooth, that is, {\color{black}
for any $g\in\mathcal{S}_{\mu,L}(\mathbb{R}^{d})$, 
for every $x,y\in\mathbb{R}^{d}$,
\begin{equation}\label{eq:Lip}
\frac{L}{2}\Vert x-y\Vert^{2} \geq g(x)-g(y) - \nabla g(y)^{T}(x-y) \geq \frac{\mu}{2}\Vert x-y\Vert^{2}.
\end{equation}}

\paragraph{Wasserstein distance.} 
Define $\mathcal{P}_{2}(\mathbb{R}^{d})$
as the space consisting of all the Borel probability measures $\nu$
on $\mathbb{R}^{d}$ with the finite 2nd moment
(based on the Euclidean norm).
For any two Borel probability measures $\nu_{1},\nu_{2}\in\mathcal{P}_{2}(\mathbb{R}^{d})$, 
the $2$-Wasserstein
distance $\mathcal{W}_{2}$ (see e.g. \citet{villani2008optimal}) is defined as:
$\mathcal{W}_{2}(\nu_{1},\nu_{2}):=\left(\inf\mathbb{E}\left[\Vert Z_{1}-Z_{2}\Vert^{2}\right]\right)^{1/2},$
where the infimum is taken over all joint distributions of the random variables $Z_{1},Z_{2}$ with marginal distributions
$\nu_{1},\nu_{2}$ respectively.

\paragraph{Notations.} 
For two matrices $A \in \mathbb{R}^{m \times n}$ and $B \in \mathbb{R}^{p \times q}$, we denote their Kronecker product by $A\otimes B$. 
We use $I_{d}$ to denote the $d\times d$ identity matrix; if the dimension $d$ is clear from the context we will also use $I$ to denote the identity matrix. We denote $x_{\ast}\in\mathbb{R}^{d}$ as the (unique) minimizer of $f\in\mathcal{S}_{\mu,L}$ defined in \eqref{defn:f}.
Moreover, we also denote
\begin{equation}\label{eq:x-ast-ND}
x^{\ast}=\left[x_{\ast}^{T},x_{\ast}^{T},\ldots,x_{\ast}^{T}\right]^{T}\in\mathbb{R}^{Nd}.
\end{equation}
For any random variable $X$, we use $\mathcal{L}(X)$
to denote the probability distribution of $X$. We say that the distribution $\pi(x)\propto e^{-f(x)}$ is strongly log-concave if $f(x)$ is $\mu$-strongly convex for some $\mu>0$. Given two functions $g(x)$ and $h(x)$ defined on a subset $\mathcal{D}$ of real numbers, we say $h(x) = \mathcal{O}(g(x))$ as $x\to a$ if there exist positive numbers $\delta$ and $M$ such that for all $x\in\mathcal{D}$ with $0<|x-a|<\delta$, we have $|f(x)|\leq Mg(x)$ whereas we say $h(x) = \Theta(g(x))$ if there exist positive numbers $\delta$ and $M_1, M_2$ such that for all $x\in\mathcal{D}$ with $0<|x-a|<\delta$, we have $M_1 g(x) \leq |f(x)|\leq M_2 g(x)$. The dependency to the point $a$ will be omitted if it is clear from the context. Given real scalars $x,y$, we consider the ratio $h(x,y) := \frac{x^{k}-y^{k}}{x-y}$ with the convention that 
$h(y,y) := \lim_{x\to y} h(x,y) = ky^{k-1}$.

\section{Decentralized Stochastic Gradient Langevin Dynamics}

We recall from \eqref{eqn:1} that decentralized stochastic gradient Langevin dynamics (DE-SGLD) are based on stochastic estimates $\tilde \nabla f_i(x)$ of the actual gradients $\nabla f_i(x)$. We make the following assumption throughout this paper regarding the stochastic estimates $\tilde \nabla f_i(x)$ which basically says that the gradient error is unbiased with a finite variance. This is a common assumption in the literature for analyzing stochastic optimization and stochastic-gradient MCMC algorithms; see e.g. \citet{dalalyan2018kinetic,chen2016stochastic,liu2020improved}. 
\begin{assumption}\label{assumption}
Let $x_i^{(k)}$ denote the local variable of node $i$ at iteration $k$. At iteration $k$, node $i$ has access to $\tilde \nabla f_i\left(x_i^{(k)}, z_i^{(k)}\right)$ where $z_i^{(k)}$ is a random variable independent of $\{z_j^{(t)}\}_{j=1,\dots,N, t=1,\dots,k-1}$ and $\{z_j^{(k)}\}_{j\neq i}$. To simplify the notation, we suppress the $z_i^{(k)}$ dependency and let $\tilde \nabla f_i\left(x_i^{(k)}\right)$ denote $\tilde \nabla f_i\left(x_i^{(k)}, z_i^{(k)}\right)$.
We assume the \emph{gradient noise} defined as
\begin{equation}\label{eq-grad-noise}
\xi_{i}^{(k+1)}:=\tilde{\nabla}f_{i}\left(x_{i}^{(k)}\right)-\nabla f_{i}\left(x_{i}^{(k)}\right),
\end{equation}
is unbiased with a finite second moment, i.e.,
\begin{equation}\label{gradient:noise:i}
\mathbb{E}\left[\xi_{i}^{(k+1)}\Big|\mathcal{F}_{k}\right]=0,
\qquad
\mathbb{E}\left\Vert\xi_{i}^{(k+1)}\right\Vert^{2}\leq\sigma^{2},
\end{equation}
where $\mathcal{F}_{k}$ is the natural filtration
of the iterates $x_i^{(k)}$ up to (and including) time $k$.
\end{assumption}

Based on 
\eqref{eq-grad-noise}, we rewrite the DE-SGLD iterations \eqref{eqn:1} in terms of the gradient noise $\xi_i^{(k+1)}$ as 

\begin{equation*}
x_{i}^{(k+1)}=\sum_{j\in\Omega_{i}}W_{ij}x_{j}^{(k)}-\eta \nabla f_{i}\left(x_{i}^{(k)}\right) - \eta \xi_i^{(k+1)}
+\sqrt{2\eta}w_{i}^{(k+1)},
\end{equation*}
where $\eta>0$ is the stepsize, $w_{i}^{(k)}$
are i.i.d. Gaussian noise with mean $0$ and covariance being identity matrices and $\Omega_{i}=\{j : (i,j) \in \mathcal{G}\}$ are the neighbors of the node $i$.\footnote{We adopt the convention that the node is a neighbor of itself, i.e. $(i,i) \in \mathcal{G}$.}  
By defining the column vector
\begin{equation*}
x^{(k)}:=\left[\left(x_{1}^{(k)}\right)^{T},\left(x_{2}^{(k)}\right)^{T},\ldots,\left(x_{N}^{(k)}\right)^{T}\right]^{T}\in\mathbb{R}^{Nd},
\end{equation*}
which concetenates the local decision variables into a single vector, we can express the DE-SGLD iterations further as
\begin{equation}\label{iterates:decentralized:overdamped}
x^{(k+1)}=\mathcal{W}x^{(k)}-\eta \nabla F\left(x^{(k)}\right)
-\eta\xi^{(k+1)}+\sqrt{2\eta}w^{(k+1)}, \quad \mbox{with} \quad \mathcal{W} = W \otimes I_d, 
\end{equation}
where we recall that $\otimes$ denotes the Kronecker product, $F:\mathbb{R}^{Nd}\rightarrow\mathbb{R}$ is defined as
\begin{equation}\label{eq:F}
F(x):=F(x_{1},\ldots,x_{N})
=\sum_{i=1}^{N}f_{i}(x_{i}),
\end{equation}
and
\begin{equation*}
w^{(k+1)}:=\left[\left(w_{1}^{(k+1)}\right)^{T},\left(w_{2}^{(k+1)}\right)^{T},\ldots,\left(w_{N}^{(k+1)}\right)^{T}\right]^{T}
\end{equation*}
are i.i.d. Gaussian noise with mean $0$ and with a covariance matrix given by the identity matrix. The vectors
\begin{equation*}
\xi^{(k+1)}:=\left[\left(\xi_{1}^{(k+1)}\right)^{T},\left(\xi_{2}^{(k+1)}\right)^{T},\ldots,\left(\xi_{N}^{(k+1)}\right)^{T}\right]^{T}
\end{equation*}
are the gradient noise so that
\begin{equation}\label{total:noise}
\mathbb{E}\left[\xi^{(k+1)}\Big|\mathcal{F}_{k}\right]=0,
\qquad
\mathbb{E}\left\Vert\xi^{(k+1)}\right\Vert^{2}\leq\sigma^{2}N.
\end{equation}
Let us define the average at $k$-th iteration
$\bar{x}^{(k)}:=\frac{1}{N}\sum_{i=1}^{N}x_{i}^{(k)}$.
Since $\mathcal{W}$ is doubly stochastic, we get
\begin{equation} \label{eq:avg-iterates}
\bar{x}^{(k+1)}=\bar{x}^{(k)}-\eta\frac{1}{N}\sum_{i=1}^{N}\nabla f_{i}\left(x_{i}^{(k)}\right)
-\eta\bar{\xi}^{(k+1)}
+\sqrt{2\eta}\bar{w}^{(k+1)},
\end{equation}
where 
\begin{equation} \label{eq:w-bar}
\bar{w}^{(k+1)}:=\frac{1}{N}\sum_{i=1}^{N}w_{i}^{(k+1)}\sim\frac{1}{\sqrt{N}}\mathcal{N}(0,I_{d}),
\qquad
\bar{\xi}^{(k+1)}:=\frac{1}{N}\sum_{i=1}^{N}\xi_{i}^{(k+1)},
\end{equation}
that satisfies 
\begin{equation}\label{bar:grad:noise}
\mathbb{E}\left[\bar{\xi}^{(k+1)}\Big|\mathcal{F}_{k}\right]=0,
\qquad
\mathbb{E}\left\Vert\bar{\xi}^{(k+1)}\right\Vert^{2}\leq\frac{\sigma^{2}}{N}.
\end{equation}

We now state the main result of this section, which bounds the 
average of $\mathcal{W}_{2}$ distance between the distribution of $x_{i}^{(k)}$ and the target distribution $\pi$ (that has a density proportional to $\exp(-f(x))$) over $1\leq i\leq N$. This result provides also a bound on the $\mathcal{W}_2$ distance of the node averages $\bar{x}^{(k)}$ and the target distribution $\pi$. To facilitate the presentation, we
define the second largest magnitude of the eigenvalues of
$W$ as
\begin{equation} \label{eq:gamma}
\bar{\gamma}:=\max\left\{\left|\lambda_{2}^{W}\right|,\left|\lambda_{N}^{W}\right|\right\} \in [0,1)\,,
\end{equation}
which is related to the connectivity of the graph $\mathcal{G}$. For instance, consider Metropolis weights where $W_{ij} = \frac{1}{\max(d_i, d_j)}$ if $(i,j) \in \mathcal{E}$ where $d_i$ is the degree (number of neighbors) of the node $i$ with the convention that each node is a neighbor of itself. In this case, for complete graphs with $N$ nodes where each node is connected to all the other nodes, we have $\bar{\gamma}=0$ whereas for a circular graph with $N$ nodes we have $d_i = 3$ for every $i$ and $\bar{\gamma}=\frac{1}{3} + \frac{2}{3}\cos(\frac{2\pi}{N}) = 1 - \mathcal{O}(\frac{1}{N})$ (see \citet[Example 1.1 and Example 1.5]{chung1997spectral}).

\begin{theorem} \label{thm:overdamped}
Assume $\mathbb{E}\Vert x^{(0)}\Vert^{2}<\infty$ and 
$\eta \in \big(0, \bar{\eta})$ where $\bar{\eta}:=\min(\frac{1+\lambda_N^W}{L},\frac{1}{L+\mu})$. 
Then, for every $k$, DE-SGLD iterates $x_i^{(k)}$ given by \eqref{eqn:1}  and their average $\bar{x}^{(k)}$ satisfy
\begin{align*}
\mathcal{W}_{2}\left(\mathcal{L}\left(\bar{x}^{(k)}\right),\pi\right) 
&\leq (1-\mu\eta)^{k}\left(\left(\mathbb{E}\Vert \bar{x}^{(0)}-x_{\ast}\Vert^{2}\right)^{1/2}+\sqrt{2\mu^{-1}dN^{-1}}\right)\nonumber\\
&\quad  
+
\left(\bar{\gamma}^2 \frac{\left(1-\eta\mu\left(1-\frac{\eta L}{2}\right)\right)^{k}-\bar{\gamma}^{2k}}
{\left(1-\eta\mu\left(1-\frac{\eta L}{2}\right)\right)-\bar{\gamma}^{2}}\right)^{1/2}
\frac{2L}{\sqrt{N}}\left(\mathbb{E}\left\Vert x^{(0)}\right\Vert^{2}\right)^{1/2} 
+ \sqrt{\eta} E_1, 
\end{align*}
and
\begin{align*}
&    E_1 := \frac{1.65L}{\mu}\sqrt{dN^{-1}} + \frac{\sigma}{\sqrt{\mu(1-\frac{\eta L}{2})N}}\\
&\qquad
+\left(\frac{\eta}{\mu(1-\frac{\eta L}{2})}+\frac{(1+\eta L)^{2}}{\mu^{2}(1-\frac{\eta L}{2})^{2}}\right)^{1/2}
\cdot
\left(\frac{4L^{2}D^{2}\eta}{N(1-\bar{\gamma})^{2}}
+\frac{4L^{2}\sigma^{2}\eta}{(1-\bar{\gamma}^{2})}
+\frac{8L^{2}d}{(1-\bar{\gamma}^{2})}
\right)^{1/2}\,.
\end{align*}
Furthermore,
\begin{align}
&\frac{1}{N}\sum_{i=1}^{N}\mathcal{W}_{2}\left(\mathcal{L}\left(x_{i}^{(k)}\right),\pi\right)\nonumber
\\
&\leq (1-\mu\eta)^{k}\left(\left(\mathbb{E}\Vert \bar{x}^{(0)}-x_{\ast}\Vert^{2}\right)^{1/2}+\sqrt{2\mu^{-1}dN^{-1}}\right) + \frac{2\bar{\gamma}^{k}}{\sqrt{N}}\left(\mathbb{E}\left\Vert x^{(0)}\right\Vert^{2}\right)^{1/2}\nonumber
\\
&\quad +\left(\bar{\gamma}^2 \frac{\left(1-\eta\mu\left(1-\frac{\eta L}{2}\right)\right)^{k}-\bar{\gamma}^{2k}}
{\left(1-\eta\mu\left(1-\frac{\eta L}{2}\right)\right)-\bar{\gamma}^{2}}\right)^{1/2}
\frac{2L}{\sqrt{N}}\left(\mathbb{E}\left\Vert x^{(0)}\right\Vert^{2}\right)^{1/2} + \sqrt{\eta} E_2+ \eta E_3, \label{ineq-thm-overdamped-to-prove}
\end{align}
with $E_{2}:=E_{1}+\frac{2\sqrt{2d}}{\sqrt{1-\bar{\gamma}^{2}}}$ and $E_3 := \frac{2D}{\sqrt{N}(1-\bar{\gamma})}
+\frac{2\sigma}{\sqrt{1-\bar{\gamma}^{2}}}$,
where $x_{\ast}$ is the minimizer of $f$,
$\bar{x}^{(0)}=\frac{1}{N}\sum_{i=1}^{N}x_{i}^{(0)}$,
$D$ is defined in \eqref{eqn:D1},
$\mathcal{L}\left(x_{i}^{(k)}\right)$ denotes the law of $x_{i}^{(k)}$ 
and $\pi$ is the Gibbs distribution with probability density function proportional to $\exp(-f(x))$.
\end{theorem}

\begin{remark}\label{remark-desgld}
We observe that in the setting of Theorem~\ref{thm:overdamped}, the asymptotic error with respect to the target distribution in 2-Wasserstein satisfies 
$ \limsup_{k\to\infty}\mathcal{W}_{2}\left(\mathcal{L}\left(\bar{x}^{(k)}\right),\pi\right) = \mathcal{O}\left(\sqrt{\eta}\right),$
where $\mathcal{O}(\cdot)$ hides other constants ($d$, $\mu$, $L$, $\sigma, N$ and $\bar{\gamma}$). 
This shows that the asymptotic error can be made arbitrarily smaller by choosing $\eta>0$ small enough. In particular, for sufficiently small $\eta$, it is easy to check that $\left(1-\eta\mu\left(1-\frac{\eta L}{2}\right)\right) \geq \bar{\gamma}^2$ and consequently from Theorem~\ref{thm:overdamped}, 
\begin{align}
\mathcal{W}_{2}\left(\mathcal{L}\left(\bar{x}^{(2K)}\right),\pi\right) &\leq \left(1-\eta\mu\left(1-\frac{\eta L}{2}\right)\right)^{K}\psi_0 +  \mathcal{O}\left(\sqrt{\eta} \right) \label{ineq wass with constant step} \\
&\leq e^{-\eta\mu \left(1-\frac{\eta L}{2}\right)K} a_0 +  \mathcal{O}\left(\sqrt{\eta} \right)
\label{ineq wass part 2}
\end{align}
for some $a_0$ (that depends on the initialization $x^{(0)}, d, \mu, L, \sigma,  \bar{\gamma}$ and $N$) where $K$ is the iteration budget. Given $K$,
if we choose
$\eta = \frac{c}{\mu K}$ for some constant $c$, then the right-hand side of \eqref{ineq wass with constant step} becomes $\Theta(1)$ as $K\to \infty$; this is because $(1-\frac{c}{K})^K \to e^{-c}=\Theta(1)$ as $K\to\infty$ for any constant $c>0$. This is not desirable, as ideally, we want the Wasserstein error bound (right-hand side of \eqref{ineq wass with constant step}) go to zero if the iteration budget $K\to\infty$. This can be achieved by choosing a stepsize such as 
$\eta = \frac{c\log\sqrt{K}}{\mu K}$ for a constant $c>1$. Then, given $c>1$ fixed, if $K$ is large enough satisfying $K\geq \bar{K}$ with $\bar{K}=\max(e, \frac{a^2}{e})$ where $a :=\frac{c(L+\mu)}{2\mu(1+\lambda)}$, 
then the stepsize $\eta = \frac{c\log\sqrt{K}}{\mu K}$ satisfies the assumptions of Theorem~\ref{thm:overdamped} 
(this follows simply from the inequality 
$\log(K) \leq 1 + \frac{K-e}{\sqrt{e K}}$ for $K\geq e$).
Consequently, from \eqref{ineq wass part 2}, we obtain
\beq\mathcal{W}_{2}\left(\mathcal{L}\left(\bar{x}^{(2K)}\right),\pi\right) = 
\mathcal{O}\left(\frac{1}{{(\sqrt{K})}^c}
+
\frac{\sqrt{c\log(K)}}{\sqrt{K}}
\right)
= 
\mathcal{O}\left(
\frac{\sqrt{\log(K)}}{\sqrt{K}}\right)\,,
\label{eq-overdamped-bound-in-iter-number}
\eeq
where the last $\mathcal{O}(\cdot)$ term hides constants that depends on $x^{(0)}, d, \mu, L, \sigma,  \bar{\gamma}, N$ and $c$. 
This shows that to sample from a distribution that is $\varepsilon$ close to the target in the 2-Wasserstein distance, it suffices to have $\mathcal{O}(\frac{1}{\varepsilon^2})$ iterations of DE-SGLD, ignoring logarithmic factors. The appearance of logarithmic factors in the iteration complexity as well as in \eqref{eq-overdamped-bound-in-iter-number} is related to the fact that constant stepsize is used, 
and similar logarithmic factors also appear even in centralized SGLD methods with constant stepsize (see \citet[Theorem~1 and Section 2]{DK2017}). It is possible to avoid the logarithmic terms by employing a time-varying stepsize similar to \citet[Theorem~2]{DK2017}.
\end{remark}

\begin{remark}\label{remark-spectral-gap} 
\rev{The upper bound given for the 2-Wasserstein distances to the target $\pi$ in Theorem~\ref{thm:overdamped} is monotonically increasing in the parameter $\bar{\gamma}$. To see this, consider the function $H(x,y) := {\frac{x^k - y^k}{x-y} y}={\sum_{i=0}^{k-1} x^i y^{k-i}}$ for $x\in (0,1), y\in (0,1)$ with the convention that $H(y,y):= {ky^k}$. For given $x$ fixed, the partial derivative $\partial_y H(x,y) = \sum_{i=0}^{k-1} (k-i) x^i y^{k-1-i}>0$. Therefore $H$ is monotonically decreasing in $y$, so is the function $\sqrt{H}$. If we set $y=\bar{\gamma}^2$ and $x=1 - \eta \mu (1-\frac{\eta L}{2})$, the third term that appears in the bound \eqref{ineq-thm-overdamped-to-prove}
is an affine function of $\sqrt{H}$ and hence monotonically increasing in $\bar{\gamma}^2$ and in $\bar{\gamma}$. Finally, after a straightforward computation it can be seen that the remaining terms $E_1, E_2$ and $E_3$ that appear in the bound \eqref{ineq-thm-overdamped-to-prove} are also monotonically increasing in $\bar{\gamma}$. It follows from this argument that closer $\bar{\gamma}$ to zero, better connectivity properties the network has (with $\bar{\gamma} = 0$ for complete graphs that are fully-connected) and the Wasserstein distance to the target becomes (smaller) better. Hence, roughly speaking, the parameter $\bar{\gamma}$ determines the additional cost of the distributed algorithm (i.e. increased bias and variance) when there is not full connectivity among the nodes.}
\end{remark}

\begin{remark}
\rev{In Assumption~\ref{assumption}, we assumed that the variance of the gradient noise is bounded. 
It is a reasonable assumption in many applications including linear regressions with stochastic gradients estimated using minibatches, since one can show that if the stepsize $\eta>0$ is small enough the variance of the gradients for DE-SGLD will stay bounded and satisfy our assumptions on the gradient noise (Assumption~\ref{assumption}) with an analysis similar to \citet[Section K]{universally-optimal-sgd}. We will illustrate this point in detail 
in Appendix~\ref{sec:gradient:noise:assump}.}
\end{remark}

\subsection{Proof of Theorem~\ref{thm:overdamped}}\label{proof:overdamped}

To facilitate the analysis, let us define $x_{k}$ from the iterates:
\begin{equation} \label{eq:x_k}
x_{k+1}=x_{k}-\eta \frac{1}{N} \nabla f(x_{k})+\sqrt{2\eta}\bar{w}^{(k+1)},
\end{equation}
where $x_{0}= \bar{x}_0=\frac{1}{N}\sum_{i=1}^{N}x_{i}^{(0)}$ and $\bar{w}^{(k+1)}$ is defined in \eqref{eq:w-bar}.  
This is an Euler-Maruyama discretization (with stepsize $\eta$)
of the continuous-time overdamped Langevin diffusion:  
\begin{equation}\label{eq:over-N}
d X_{t}=-  \frac{1}{N} \nabla f(X_{t})dt+\sqrt{2N^{-1}}dW_{t},
\end{equation}
where $W_{t}$ is a standard $d$-dimensional Brownian motion.

To bound the average of $\mathcal{W}_{2}$ distance between $\mathcal{L}\left(x_{i}^{(k)}\right)$ and $\pi$ over $1\leq i\leq N$, the main idea of our proof technique is to bound the following three terms: (1) the $L^2$ distance between $x_{i}^{(k)}$ and their average (mean) $\bar{x}^{(k)}=\frac{\sum_{i=1}^N x_{i}^{(k)}}{N}$ for $1\leq i\leq N$; (2) the $L^2$ distance between the average iterate $\bar x^{(k)}$ and iterates $x_k$ obtained from Euler-Maruyama discretization of overdamped SDE; and (3) the $\mathcal{W}_{2}$ distance between between $\mathcal{L}\left(x_k\right)$ and $\pi$, i.e. the convergence of Euler-Maruyama discretization of the overdamped SDE. The next subsections are devoted to controlling each of these three terms.

\subsubsection{Uniform $L^2$ bounds between $x_{i}^{(k)}$ and their average}

We first state a key lemma which provides $L^2$ bounds on the gradients $\nabla F\left(x^{(k)}\right)$ that are uniform in $k$, where $F$ is defined in \eqref{eq:F}. Recall from \eqref{eq:x-ast-ND} that $x_{\ast}\in\mathbb{R}^{d}$ denotes the unique minimizer of $f(x)$, and
$x^{\ast}=\left[x_{\ast}^{T},x_{\ast}^{T},\ldots,x_{\ast}^{T}\right]^{T}$
is an $Nd$-dimensional vector. We view DE-SGLD as a decentralized gradient descent (DGD) method subject to stochastic gradient and Gaussian noise, and our analysis is inspired by the proof techniques of \citet{Yuan16} for analyzing DGD methods. The proof of this lemma is provided in the Appendix.

\begin{lemma}\label{lem:0}
Under the assumptions of Theorem~\ref{thm:overdamped},
we have,
\begin{equation*}
\mathbb{E}\left\Vert\nabla F\left(x^{(k)}\right)\right\Vert^{2}
\leq D^{2}, \quad \mbox{for any } k,
\end{equation*}
where
\begin{equation}\label{eqn:D1}
D^{2}:=4L^{2}\mathbb{E}\left\Vert x^{(0)}-x^{\ast}\right\Vert^{2}
+8L^{2}\frac{C_{1}^{2}\eta^{2}N}{(1-\bar{\gamma})^{2}}
+\frac{2L^{2}(\eta\sigma^{2}N+2dN)}{\mu(1+\lambda_{N}^{W}-\eta L)}
+4\left\Vert \nabla F\left(x^{\ast}\right)\right\Vert^{2}.
\end{equation}
Here, $x^* \in \mathbb{R}^{Nd}$ is given in \eqref{eq:x-ast-ND}, $\bar{\gamma}$ is defined by \eqref{eq:gamma} and
\begin{equation}\label{def-C1}
 \quad C_1 := \bar{C}_1\cdot \left(1 + \frac{2(L+\mu)}{\mu   }\right),
 \,\text{where}\quad
 \bar{C}_1 := \sqrt{2L\sum_{i=1}^N \left(f_i\left(0\right) -f_i^*\right)}, \quad f_i^* := \min_{x\in\mathbb{R}^d} f_i(x).
\end{equation}
\end{lemma}


It is clear from the DE-SGLD iterations that the deviations between the iterates $x_i^{(k)}$ and their means $\bar{x}^{(k)}$ depend on the magnitude of the gradients $\nabla F(x^{(k)})$, the stepsize as well as the magnitude of the injected Gaussian noise. Building on Lemma~\ref{lem:0} which gives us a control over the second moment of the gradients, in the next result we provide uniform $L_2$ bounds between the iterates $x_i^{(k)}$ and their means. The proof can be found in the Appendix.

\begin{lemma}\label{lem:1}
Under the assumptions of Theorem~\ref{thm:overdamped}, for any $k$, we have
\begin{equation*}
\sum_{i=1}^{N}\mathbb{E}\left\Vert x_{i}^{(k)}-\bar{x}^{(k)}\right\Vert^{2}
\leq
4\bar{\gamma}^{2k}\mathbb{E}\left\Vert x^{(0)}\right\Vert^{2}
+\frac{4D^{2}\eta^{2}}{(1-\bar{\gamma})^{2}}
+\frac{4\sigma^{2}N\eta^{2}}{(1-\bar{\gamma}^{2})}
+\frac{8dN\eta}{(1-\bar{\gamma}^{2})},
\end{equation*}
where $D$ is defined in \eqref{eqn:D1} and $\bar{\gamma}$ is given in \eqref{eq:gamma}. 
\end{lemma}

Note that we can deduce from \eqref{eq:avg-iterates} that
\begin{equation}\label{eq:avg-error}
\bar{x}^{(k+1)}=\bar{x}^{(k)}-\eta \frac{1}{N} \nabla f\left(\bar{x}^{(k)}\right)
+\eta\mathcal{E}_{k+1}-\eta\bar{\xi}^{(k+1)}
+\sqrt{2\eta}\bar{w}^{(k+1)},
\end{equation}
where
\begin{equation}\label{eq-gradient-error}
\mathcal{E}_{k+1}
:=
\frac{1}{N}\sum_{i=1}^{N} \left[ \nabla f_i\left(\bar{x}^{(k)}\right) - \nabla f_{i}\left(x_{i}^{(k)}\right)\right].
\end{equation}
We observe that the average iterate $\bar{x}^{(k)}$ in \eqref{eq:avg-error} follows a gradient descent dynamics subject to gradient errors and Gaussian noise, if we view $\mathcal{E}_k$ as a gradient error term. Since $\nabla f_i$ is Lipschitz by our assumptions, the gradient error \eqref{eq-gradient-error} can be controlled based on Lemma~\ref{lem:1}. In particular, as a corollary of Lemma~\ref{lem:1}, we obtain the following result; the proof is given in the Appendix for the sake of completeness.

\begin{lemma}\label{lem:2}
Under the assumptions of Theorem~\ref{thm:overdamped},
for any $k$, we have
\begin{equation*}
\mathbb{E}\left\Vert\mathcal{E}_{k+1}\right\Vert^{2}
\leq
\frac{4L^{2}\bar{\gamma}^{2k}}{N}\mathbb{E}\left\Vert x^{(0)}\right\Vert^{2}
+\frac{4L^{2}D^{2}\eta^{2}}{N(1-\bar{\gamma})^{2}}
+\frac{4L^{2}\sigma^{2}\eta^{2}}{(1-\bar{\gamma}^{2})}
+\frac{8L^{2}d\eta}{(1-\bar{\gamma}^{2})},
\end{equation*} 
where $\mathcal{E}_{k+1}$ is defined in \eqref{eq-gradient-error}.
\end{lemma}

\subsubsection{$L^2$ distance between the mean and the discretized overdamped SDE}

Recall the iterates $x_{k}$ defined in \eqref{eq:x_k} which is 
an Euler-Maruyama discretization of the continuous-time overdamped Langevin SDE in \eqref{eq:over-N} with stepsize $\eta$, and the mean $\bar{x}^{(k)}$ in \eqref{eq:avg-error}.
Since the $L^{2}$ bound of the error term $\mathcal{E}_{k+1}$
can be controlled as in Lemma~\ref{lem:2}, 
we will show that the mean $\bar{x}^{(k)}$ and $x_{k}$ are close to each other
in $L^{2}$ distance. Indeed, we have the following estimate:

\begin{lemma}\label{lem:3}
Under the assumptions of Theorem~\ref{thm:overdamped},
for every $k$,
\begin{align*}
\mathbb{E}\left\Vert\bar{x}^{(k)}-x_{k}\right\Vert^{2}\nonumber
&\leq
\eta\left(\frac{\eta}{\mu(1-\frac{\eta L}{2})}+\frac{(1+\eta L)^{2}}{\mu^{2}(1-\frac{\eta L}{2})^{2}}\right)
\left(\frac{4L^{2}D^{2}\eta}{N(1-\bar{\gamma})^{2}}
+\frac{4L^{2}\sigma^{2}\eta}{(1-\bar{\gamma}^{2})}
+\frac{8L^{2}d}{(1-\bar{\gamma}^{2})}\right)
\nonumber
\\
&\qquad\qquad\qquad
+\frac{\eta\sigma^{2}}{\mu(1-\frac{\eta L}{2})N}
+\frac{\bar{\gamma}^{2k}-
\left(1-\eta\mu\left(1-\frac{\eta L}{2}\right)\right)^{k}}
{\bar{\gamma}^{2}-1+\eta\mu\left(1-\frac{\eta L}{2}\right)}
\frac{4L^{2}\bar{\gamma}^{2}}{N}\mathbb{E}\left\Vert x^{(0)}\right\Vert^{2}.
\end{align*}
\end{lemma}

\subsubsection{$\mathcal{W}_2$ distance between the iterates and the Gibbs distribution}

Bounds on the $\mathcal{W}_{2}$ distance between the Euler-Maruyama discretization
$x_{k}$ of the overdamped Langevin diffusion
and Gibbs distribution $\pi$ has been established in the literature. We note that the function $\frac{1}{N}f$ is $\frac{\mu}{N}$-strongly convex
and $\frac{L}{N}$-smooth, and we state Theorem~4 in \citet{DK2017} as follows. 

\begin{lemma}[Theorem~4 in \citet{DK2017}]\label{lem:DK}
For any $\eta 
\in (0,\frac{2N}{L+\mu}]$, we have
\begin{equation*}
\mathcal{W}_{2}\left(\mathcal{L}(x_{k}),\pi\right)
\leq
(1-\mu\eta)^{k}\mathcal{W}_{2}\left(\mathcal{L}(x_{0}),\pi\right)
+\frac{1.65L}{\mu}\sqrt{\eta dN^{-1}}.
\end{equation*}
\end{lemma}
\rev{The proof of this lemma is based on the so-called ``synchronous coupling" technique to control the $\mathcal{W}_2$ distances, see \citet{DK2017} for details.} Next, we bound the $L^{2}$ distance between the minimizer of $f$ and Gibbs distribution $\pi$; the proof is provided in Appendix~\ref{sec:additional}.

\begin{lemma}\label{bound:Gibbs}
Let $x_{\ast}$ be the unique minimizer of $f(x)$. Then, we have
$\mathbb{E}_{X\sim\pi}\Vert X-x_{\ast}\Vert^{2}
\leq\frac{2dN^{-1}}{\mu}$.
\end{lemma}
Putting all the pieces together, the stage is set for the proof of Theorem~\ref{thm:overdamped}.

\subsubsection{Proof of Theorem~\ref{thm:overdamped}} 
Since $x_{0}=\frac{1}{N}\sum_{i=1}^{N}x_{i}^{(0)}$, 
we have $\mathbb{E}\Vert x_{0}\Vert^{2}<\infty$. By Lemma~\ref{bound:Gibbs},
\begin{align*}
\mathcal{W}_{2}\left(\mathcal{L}(x_{0}),\pi\right)
&\leq
\left(\mathbb{E}\Vert x_{0}-x_{\ast}\Vert^{2}\right)^{1/2}
+\left(\mathbb{E}_{X\sim\pi}\Vert X-x_{\ast}\Vert^{2}\right)^{1/2}
\\
&\leq\left(\mathbb{E}\Vert x_{0}-x_{\ast}\Vert^{2}\right)^{1/2}+\sqrt{2\mu^{-1}dN^{-1}}.
\end{align*}
Under our assumptions on the stepsize $\eta$, we have clearly $\eta \in (0,\frac{2N}{L+\mu}]$ as $N\geq 1$. Therefore Lemma~\ref{lem:DK} is applicable. More specifically, 
it follows from Lemma~\ref{lem:DK} that,
\begin{equation*}
\mathcal{W}_{2}\left(\mathcal{L}(x_{k}),\pi\right)
\leq
(1-\mu\eta)^{k}\left(\left(\mathbb{E}\Vert x_{0}-x_{\ast}\Vert^{2}\right)^{1/2}+\sqrt{2\mu^{-1}dN^{-1}}\right)
+\frac{1.65L}{\mu}\sqrt{\eta dN^{-1}}.
\end{equation*}
Moreover, it follows from Lemma~\ref{lem:3} that
\begin{align*}
&\mathcal{W}_{2}\left(\mathcal{L}\left(\bar{x}^{(k)}\right),\mathcal{L}(x_{k})\right)
\\
&\leq
\left(\mathbb{E}\left\Vert\bar{x}^{(k)}-x_{k}\right\Vert^{2}\right)^{1/2}
\\
&\leq
\eta^{1/2}\left(\frac{\eta}{\mu(1-\frac{\eta L}{2})}+\frac{(1+\eta L)^{2}}{\mu^{2}(1-\frac{\eta L}{2})^{2}}\right)^{1/2}
\cdot
\left(\frac{4L^{2}D^{2}\eta}{N(1-\bar{\gamma})^{2}}
+\frac{4L^{2}\sigma^{2}\eta}{(1-\bar{\gamma}^{2})}
+\frac{8L^{2}d}{(1-\bar{\gamma}^{2})}
\right)^{1/2}
\\
&\quad
+\frac{\sqrt{\eta}\sigma}{\sqrt{\mu(1-\frac{\eta L}{2})N}}
+\left(\frac{\bar{\gamma}^{2k}-
\left(1-\eta\mu\left(1-\frac{\eta L}{2}\right)\right)^{k}}
{\bar{\gamma}^{2}-1+\eta\mu\left(1-\frac{\eta L}{2}\right)}\right)^{1/2}
\frac{2L\bar{\gamma}}{\sqrt{N}}\left(\mathbb{E}\left\Vert x^{(0)}\right\Vert^{2}\right)^{1/2}.
\end{align*}
Hence, we conclude that
\begin{align}
\mathcal{W}_{2}\left(\mathcal{L}\left(\bar{x}^{(k)}\right),\pi\right) 
&\leq (1-\mu\eta)^{k}\left(\left(\mathbb{E}\Vert \bar{x}^{(0)}-x_{\ast}\Vert^{2}\right)^{1/2}+\sqrt{2\mu^{-1}dN^{-1}}\right)\nonumber\\
&\quad +\left(\frac{\left(1-\eta\mu\left(1-\frac{\eta L}{2}\right)\right)^{k}-\bar{\gamma}^{2k}}
{\left(1-\eta\mu\left(1-\frac{\eta L}{2}\right)\right)-\bar{\gamma}^{2}}\right)^{1/2}
\frac{2L\bar{\gamma}}{\sqrt{N}}\left(\mathbb{E}\left\Vert x^{(0)}\right\Vert^{2}\right)^{1/2} + \sqrt{\eta} E_1, \end{align}
with
\begin{align*}
&    E_1 := \frac{1.65L}{\mu}\sqrt{dN^{-1}} + \frac{\sigma}{\sqrt{\mu(1-\frac{\eta L}{2})N}}\\
&\qquad
+\left(\frac{\eta}{\mu(1-\frac{\eta L}{2})}+\frac{(1+\eta L)^{2}}{\mu^{2}(1-\frac{\eta L}{2})^{2}}\right)^{1/2}
\cdot
\left(\frac{4L^{2}D^{2}\eta}{N(1-\bar{\gamma})^{2}}
+\frac{4L^{2}\sigma^{2}\eta}{(1-\bar{\gamma}^{2})}
+\frac{8L^{2}d}{(1-\bar{\gamma}^{2})}
\right)^{1/2}\,.
\end{align*}
Finally, by the Cauchy-Schwarz inequality,
\begin{align}
\frac{1}{N}\sum_{i=1}^{N}\mathcal{W}_{2}\left(\mathcal{L}\left(x_{i}^{(k)}\right),\mathcal{L}\left(\bar{x}^{(k)}\right)\right)
&\leq\sqrt{\frac{1}{N}\sum_{i=1}^{N}\mathcal{W}_{2}^{2}\left(\mathcal{L}\left(x_{i}^{(k)}\right),\mathcal{L}\left(\bar{x}^{(k)}\right)\right)}\nonumber
\\
&\leq\sqrt{\frac{1}{N}\sum_{i=1}^{N}\mathbb{E}\left\Vert x_{i}^{(k)}-\bar{x}^{(k)}\right\Vert^{2}}.\label{by:cauchy:schwarz}
\end{align}
Also, by Lemma~\ref{lem:1}, we have
\begin{align*}
\sqrt{\frac{1}{N}\sum_{i=1}^{N}\mathbb{E}\left\Vert x_{i}^{(k)}-\bar{x}^{(k)}\right\Vert^{2}}
&\leq
\left(\frac{4\bar{\gamma}^{2k}}{N}\mathbb{E}\left\Vert x^{(0)}\right\Vert^{2}
+\frac{4D^{2}\eta^{2}}{N(1-\bar{\gamma})^{2}}
+\frac{4\sigma^{2}\eta^{2}}{(1-\bar{\gamma}^{2})}
+\frac{8d\eta}{(1-\bar{\gamma}^{2})}\right)^{1/2}
\\
&\leq
\frac{2\bar{\gamma}^{k}}{\sqrt{N}}\left(\mathbb{E}\left\Vert x^{(0)}\right\Vert^{2}\right)^{1/2}
+\frac{2D\eta}{\sqrt{N}(1-\bar{\gamma})}
+\frac{2\sigma\eta}{\sqrt{1-\bar{\gamma}^{2}}}
+\frac{2\sqrt{2d\eta}}{\sqrt{1-\bar{\gamma}^{2}}}.
\end{align*}
The inequality \eqref{ineq-thm-overdamped-to-prove} then follows from the triangular inequality for the $2$-Wasserstein distance. 
This completes the proof.
\hfill $\Box$

\section{Decentralized Stochastic Gradient Hamiltonian Monte Carlo}

We introduce the following algorithm which we call decentralized stochastic gradient Hamiltonian Monte Carlo (DE-SGHMC): For each agent $i=1, \ldots, N,$
\begin{align}
&v_{i}^{(k+1)}=v_{i}^{(k)}-\eta\left[\gamma v_{i}^{(k)}+\tilde{\nabla} f_{i}\left(x_{i}^{(k)}\right)\right]+\sqrt{2\gamma\eta}w_{i}^{(k+1)},
\label{eq:under-vik}\\
&x_{i}^{(k+1)}=\sum_{j\in\Omega_{i}}W_{ij}x_{j}^{(k)}+\eta v_{i}^{(k+1)}, \label{eq:under-xik}
\end{align}
starting from the initializations $x_i^{(0)}, v_i^{(0)} \in \mathbb{R}^d$, where $\eta>0$ is the stepsize, $w_{i}^{(k+1)}$
are i.i.d. Gaussian noise with mean $0$ and covariance being $d-$dimensional identity matrices. We note that in this section, we are abusing the notation for simplicity of the presentation and using $x_i^{(k)}$ to denote the DE-SGHMC iterates instead of DE-SGLD iterates. This algorithm is a natural adaptation of the SGHMC algorithm to the decentralized setting:
If the term $\sum_{j\in\Omega_{i}}W_{ij}x_{j}^{(k)}$ is replaced by $x_i^{(k)}$, then the resulting dynamics at each node reduces to SGHMC which is a discretization of the underdamped Langevin diffusion given in \eqref{eq:VL}-\eqref{eq:XL} (see e.g. \citet{GGZ}).

Note that the gradient noise
$\xi_{i}^{(k+1)}:=\tilde{\nabla}f_{i}(x_{i}^{(k)})-\nabla f_{i}(x_{i}^{(k)})$
satisfies Assumption~\ref{assumption} 
so that 
$\xi^{(k+1)}:=\left[(\xi_{1}^{(k+1)})^{T},\ldots,(\xi_{N}^{(k+1)})^{T}\right]^{T}$
satisfies \eqref{total:noise} and $\bar{\xi}^{(k+1)}:=\frac{1}{N}\sum_{i=1}^{N}\xi_{i}^{(k+1)}$ satisfies
\eqref{bar:grad:noise}.
By defining the column vectors
\begin{align*}
&x^{(k)}:=\left[\left(x_{1}^{(k)}\right)^{T},\left(x_{2}^{(k)}\right)^{T},\ldots,\left(x_{N}^{(k)}\right)^{T}\right]^{T}\in\mathbb{R}^{Nd},
\\
&v^{(k)}:=\left[\left(v_{1}^{(k)}\right)^{T},\left(v_{2}^{(k)}\right)^{T},\ldots,\left(v_{N}^{(k)}\right)^{T}\right]^{T}\in\mathbb{R}^{Nd},
\end{align*}
where $v_{i}^{(k)}$ and $x_{i}^{(k)}$ satisfy \eqref{eq:under-vik}--\eqref{eq:under-xik}, we can rewrite the DE-SGHMC as follows:
\begin{align} 
&v^{(k+1)}=v^{(k)}-\eta\left[\gamma v^{(k)}+\nabla F\left(x^{(k)}\right)+\xi^{(k+1)}\right]+\sqrt{2\gamma\eta}w^{(k+1)}, \label{eq:underdamped}
\\
&x^{(k+1)}=\mathcal{W}x^{(k)}+\eta v^{(k+1)},  \label{eq:underdamped2}
\end{align}
where $\mathcal{W} = W \otimes I_d$ and $F:\mathbb{R}^{Nd}\rightarrow\mathbb{R}$ is defined as
$F(x):=F(x_{1},\ldots,x_{N})
=\sum_{i=1}^{N}f_{i}(x_{i})$,
$w^{(k+1)}$ are i.i.d. Gaussian noise with mean $0$ and covariance being $Nd-$dimensional identity matrix.
Let us define the average at $k$-th iteration as:
\begin{equation} \label{eq:under-avg}
\bar{x}^{(k)}:=\frac{1}{N}\sum_{i=1}^{N}x_{i}^{(k)},
\qquad
\bar{v}^{(k)}:=\frac{1}{N}\sum_{i=1}^{N}v_{i}^{(k)}.
\end{equation}
Since $\mathcal{W}$ is doubly stochastic, we get
\begin{align*}
&\bar{v}^{(k+1)}=\bar{v}^{(k)}-\eta\gamma\bar{v}^{(k)}-\eta\frac{1}{N}\sum_{i=1}^{N}\nabla f_{i}\left(x_{i}^{(k)}\right)
-\eta\bar{\xi}^{(k+1)}
+\sqrt{2\gamma \eta}\bar{w}^{(k+1)},
\\
&\bar{x}^{(k+1)}=\bar{x}^{(k)}+\eta\bar{v}^{(k+1)},
\end{align*}
where $\bar{\xi}^{(k+1)}:=\frac{1}{N}\sum_{i=1}^{N}\xi_{i}^{(k+1)}$ and $\bar{w}^{(k+1)}:=\frac{1}{N}\sum_{i=1}^{N}w_{i}^{(k+1)}\sim\frac{1}{\sqrt{N}}\mathcal{N}(0,I_{d})$.

We now state the main result of this section which bounds the average of $\mathcal{W}_{2}$ distance between the distribution of the node iterates $x_{i}^{(k)}$ and the target distribution $\pi$. 
The result shows that if the parameters $\eta$ and $\gamma$ are suitably chosen, then this distance decays geometrically fast (in $k$) to a level of $\mathcal{O}(\eta)$.
This result also bounds the $\mathcal{W}_2$ distance of the node averages $\bar{x}^{(k)}$ and the target distribution $\pi$. \rev{The main idea of the proof is to analyze DE-SGHMC as a perturbed heavy-ball method (see Section~\ref{proof:underdamped} and the proof of Lemma~\ref{lemma-l2-underdamped}) which appears to be a new technique to analyze SGHMC methods}. Recall $\bar{\gamma}=\max\left\{\left|\lambda_{2}^{W}\right|,\left|\lambda_{N}^{W}\right|\right\} \in [0,1)$
from \eqref{eq:gamma},
and $x_*$ is the minimizer of $f(x)$. 

\begin{theorem}\label{thm:underdamped}
Assume $\mathbb{E}\Vert x^{(0)}\Vert^{2}$ and $\mathbb{E}\Vert v^{(0)}\Vert^{2}$ are finite. Let $\eta$ be given satisfying
 \beq \eta^2 \in \bigg(0, \frac{1+\lambda_N^W}{2(L+\mu)}\bigg].
    \label{ineq-step-cond}
 \eeq
Then, we can can choose $\gamma \in (0,\frac{1}{\eta}]$ such that $\beta := 1-\gamma \eta \in [0,1)$ and satisfies the inequality
\begin{align} 
\beta \leq \bar{\beta}:= \min \left(\frac{1+\lambda_N^W - 4\eta^2\mu}{4}, \eta^3 \sqrt{c_1\mu^3 \frac{(1+\lambda_N^W)}{64 }}\right)\,,
\label{ineq-momentum-cond}
\end{align}
where
\begin{align}
c_1 := \frac{1}{2}\frac{\eta^2\mu}{(1+\beta) + (1-\beta)\left(\frac{\eta^2\mu}{1-\lambda_N^W + \eta^2 L }\right)}\,,
\nonumber
\end{align}
and for every $k$, DE-SGHMC iterates $x_i^{(k)}$ given by \eqref{eq:under-xik} and their average $\bar{x}^{(k)}$ satisfy 
\begin{align}
\mathcal{W}_{2}\left(\mathcal{L}\left(\bar{x}^{(k)}\right),\pi\right) 
&\leq
\left(1-\mu\eta^{2}\right)^{k}\left(\left(\mathbb{E}\left\Vert \bar{x}^{(0)}-x_{\ast}\right\Vert^{2}\right)^{1/2}+\sqrt{2\mu^{-1}dN^{-1}}\right)
\nonumber
\\
&\quad
+\left(\bar{\gamma}^2\frac{
\left(1-\eta^{2}\mu\left(1-\frac{\eta^{2} L}{2}\right)\right)^{k}-\bar{\gamma}^{2k}}
{\left(1-\eta^{2}\mu\left(1-\frac{\eta^{2} L}{2}\right)\right)-\bar{\gamma}^{2}}\right)^{1/2}
\frac{2L}{\sqrt{N}}\left(\mathbb{E}\left\Vert x^{(0)}\right\Vert^{2}\right)^{1/2}  + \eta E_4, 
\end{align}
with
\begin{align*}
    E_4 &:=\sqrt{2}\left(\frac{\eta^{2}}{\mu(1-\frac{\eta^{2} L}{2})}+\frac{(1+\eta^{2} L)^{2}}{\mu^{2}(1-\frac{\eta^{2} L}{2})^{2}}\right)^{1/2}
\\
&\qquad\qquad
\cdot
\left[\left(\frac{\beta^{2}c_{5}}{\eta^{4}N}
+\frac{2L^{2}c_{5}}{N(1-\bar{\gamma})^{2}}\right)^{1/2}
+\left(\frac{(\sqrt{1-\beta}-1)^{2}}{\eta^{4}}\frac{d}{N}\right)^{1/2}\right]
\\
&\qquad\qquad\qquad
+\frac{1.65L}{\mu}\sqrt{dN^{-1}} + \frac{\sigma}{\sqrt{\mu(1-\frac{\eta^{2}L}{2})N}}
= \mathcal{O}(1)\,,
\end{align*}
and
\begin{align}
&\frac{1}{N}\sum_{i=1}^{N}\mathcal{W}_{2}\left(\mathcal{L}\left(x_{i}^{(k)}\right),\pi\right)\nonumber
\\
&\leq
\left(1-\mu\eta^{2}\right)^{k}\left(\left(\mathbb{E}\left\Vert \bar{x}^{(0)}-x_{\ast}\right\Vert^{2}\right)^{1/2}+\sqrt{2\mu^{-1}dN^{-1}}\right)
+ \frac{\sqrt{2}\bar{\gamma}^{k}}{\sqrt{N}}\left(\mathbb{E}\left\Vert x^{(0)}\right\Vert^{2}\right)^{1/2} \nonumber
\\
&\quad
+\left(\frac{
\left(1-\eta^{2}\mu\left(1-\frac{\eta^{2} L}{2}\right)\right)^{k}-\bar{\gamma}^{2k}}
{\left(1-\eta^{2}\mu\left(1-\frac{\eta^{2} L}{2}\right)\right)-\bar{\gamma}^{2}}\right)^{1/2}
\frac{2L\bar{\gamma}}{\sqrt{N}}\left(\mathbb{E}\left\Vert x^{(0)}\right\Vert^{2}\right)^{1/2}  + \eta E_5, \label{ineq-W2-bound-underdamped}
\end{align}
with $E_5:=E_4+\frac{\sqrt{2c_{5}}}{\sqrt{N}(1-\bar{\gamma})}=\mathcal{O}(1)$,
and $\beta= \mathcal{O}(\eta^4)$ where $\mathcal{O}(\cdot)$ hides the constants that depend on $d$, $\mu$, $L$, $\sigma$ and $\bar{\gamma}$ and $N$, $\mathcal{L}\left(\bar{x}^{(k)}\right)$ denotes the law of $\bar{x}^{(k)}$ and $\pi$
denotes the Gibbs distribution with probability density function proportional to $e^{-f(x)}$, 
and $c_{5}$ is defined in Lemma~\ref{lemma-l2-underdamped}. 
\end{theorem} 


\begin{remark}\label{remark-sghmc}
We observe that in the setting of Theorem~\ref{thm:underdamped}, the asymptotic error with respect to the target distribution satisfies 
$ \limsup_{k\to\infty}\frac{1}{N}\sum_{i=1}^{N}\mathcal{W}_{2}\left(\mathcal{L}\left(x_{i}^{(k)}\right),\pi\right) = \mathcal{O}\left(\eta\right),$
where $\mathcal{O}(\cdot)$ hides other constants ($d$, $\mu$, $L$, $\sigma, N$ and $\bar{\gamma}$). 
In particular, for $\eta$ small enough, it is easy to check that $\left(1-\eta^2\mu\left(1-\frac{\eta^2 L}{2}\right)\right) \geq \bar{\gamma}^2$ and consequently from Theorem~\ref{thm:underdamped}, defining $\alpha:=\eta^2$, we obtain 
\begin{align}
\mathcal{W}_{2}\left(\mathcal{L}\left(\bar{x}^{(2K)}\right),\pi\right) &\leq \left(1-\alpha\mu\left(1-\frac{\alpha L}{2}\right)\right)^{K}b_0 +  \mathcal{O}\left(\sqrt{\alpha} \right) \label{ineq wass underdamped with constant step} \\
&\leq e^{-\alpha\mu \left(1-\frac{\alpha L}{2}\right)K} b_0 +  \mathcal{O}\left(\sqrt{\alpha} \right)
\label{ineq wass underdamped part 2}
\end{align}
for some $b_0$ (that depends on the initialization $x^{(0)}, d, \mu, L, \sigma,  \bar{\gamma}$ and $N$) where $K$ is the iteration budget. We observe that this bound in $\alpha$ is similar to that of DE-SGLD case analyzed in \eqref{ineq wass with constant step}--\eqref{ineq wass part 2} if we were to replace $\eta$ in \eqref{ineq wass with constant step}--\eqref{ineq wass part 2} by $\alpha$. By following the same argument as in Remark~\ref{remark-desgld}, if we choose $\alpha = \eta^2 =  \frac{c\log\sqrt{K}}{\mu K}$, where the constant $c>1$, then we obtain 
$$ \mathcal{W}_{2}\left(\mathcal{L}\left(\bar{x}^{(2K)}\right),\pi\right) = \mathcal{O}\left(\frac{\sqrt{\log(K)}}{\sqrt{K}}\right).
$$
We conclude that in order to sample fwerom a distribution that is $\varepsilon$ close to the target in the 2-Wasserstein distance, it suffices to have $\mathcal{O}(\frac{1}{\varepsilon^2})$ iterations of DE-SGHMC, ignoring logarithmic factors. This iteration complexity bound is of the same order with that we obtained for DE-SGLD (see Remark~\ref{remark-desgld}). However, in practice we have seen that DE-SGHMC outperformed DE-SGLD in some cases (see Section~\ref{subsec-logistic-real}). \rev{We also note that, with a similar analysis to that in Remark \ref{remark-spectral-gap}, it can be shown that all the terms appearing in the performance bounds is monotonically increasing as a function of $\bar{\gamma}$ in the setting of Theorem~\ref{thm:underdamped} except the constant $c_5$ whose dependency to $\bar\gamma$ is more complicated to determine within our analysis}.
\end{remark} 


\subsection{Proof of Theorem~\ref{thm:underdamped}}\label{proof:underdamped}
To facilitate the analysis, we introduce the iterates $(x_k)$ (with slight abuse of notations):
 \begin{align} \label{eq:under-over}
 x_{k+1}=x_{k}-\eta^{2}\frac{1}{N}\nabla f(x_{k})+\sqrt{2}\eta\bar{w}^{(k+1)},
\end{align}
where $\bar{w}^{(k+1)}$ is the Gaussian noise given in \eqref{eq:w-bar} and $x_0=\bar{x}^{(0)}$.
 This is an Euler-Maruyama discretization (with stepsize $\eta^2$) of the continuous-time overdamped Langevin diffusion:
 \begin{align*}
 dX_{t}=-\frac{1}{N}\nabla f(x_{k})dt+\sqrt{2N^{-1}}dW_{t},
 \end{align*}
where $W_{t}$ is a standard $d$-dimensional Brownian motion. We also define iterates $(\tilde{x}_k)$:
\begin{align}\label{eq:tilde-x-k}
\tilde{x}_{k+1}=\tilde{x}_{k}-\eta^{2}\frac{1}{N}\nabla f(\tilde{x}_{k})+\sqrt{2\gamma\eta}\eta\bar{w}^{(k+1)},
\end{align}
where $\tilde{x}_0= x_0=\bar{x}^{(0)}$.

We recall that SGHMC can be viewed as a discretization of the kinetic (inertial) Langevin SDE \eqref{eq:VL}--\eqref{eq:XL}. It is also known that as the friction coefficient $\gamma\to\infty$, the paths of this SDE becomes more and more similar to the paths of the overdamped Langevin SDE (see e.g. \citet{leimkuhler2016computation}). However, this is not the case when $\gamma>0$ is small enough. Therefore, the stepsize $\eta$ is small enough and the friction coefficient $\gamma$ is large enough, it is reasonable to expect that the node averages $\bar{x}^{(k)}$ of DE-SGHMC given by \eqref{eq:under-avg} track the overdamped SDE dynamics. In the setting of Theorem~\ref{thm:underdamped}, we consider such a case when the stepsize $\eta$ is small enough and $\gamma \approx \frac{1}{\eta}$ (see \eqref{ineq-step-cond}-\eqref{ineq-momentum-cond}). We will next show that node averages $\bar{x}^{(k)}$ and the discretized overdamped SDE iterates will be close to each other in the 2-Wasserstein metric and that the iterates ${x}_i^{(k)}$ will remain close to their average $\bar{x}^{(k)}$ in $L_2$ distance. We note that in general, the optimal choice of $\gamma$ is not known in the decentralized setting; and it is only known in the centralized setting for special cases: For centralized Langevin dynamics with deterministic gradients and $\mu$-strongly convex quadratic objectives, it is recently shown that the choice of $\gamma=2\sqrt{\mu}$ optimizes the convergence rate to the stationary distribution in the 2-Wasserstein distance \citep{GGZ2}. Studying the convergence of DE-SGHMC iterates for other choices of the friction coefficient $\gamma$ will be left as a future work, as our current proof techniques do not allow arbitrary choice of $\gamma$.

In the proof of Theorem~\ref{thm:underdamped}, to bound the $\mathcal{W}_{2}$ distance between the average of $\mathcal{L}\left(x_{i}^{(k)}\right)$ and $\pi$
over $1\leq i\leq N$, we follow a similar approach to the analysis of DE-SGLD where
the idea is to bound the following four terms: (1) the $L^2$ distance between $x_{i}^{(k)}$ and the average iterate $\bar x^{(k)}$;
(2) the $L^2$ distance between the average iterate $\bar x^{(k)}$ and iterates $\tilde{x}_k$ in \eqref{eq:tilde-x-k};
(3) the $L^2$ distance between the iterates $\tilde{x}_k$ and the iterates $x_{k}$ in \eqref{eq:under-over}; (4) the $\mathcal{W}_{2}$ distance between $\mathcal{L}\left(x_k\right)$ and $\pi$, i.e. the convergence of overdamped Langevin dynamics. For analyzing the first term, 
we first present a technical lemma \rev{(Lemma~\ref{lemma-l2-underdamped})} on uniform $L^2$ bounds on the iterates $v^{(k)}, x^{(k)}$ in \eqref{eq:underdamped}--\eqref{eq:underdamped2}. 
The result will be used in the proof of Lemma~\ref{lem:1:under}. The proof idea is to analyze DE-SGHMC as a perturbed heavy-ball method. Momentum-based first order methods such as heavy-ball methods are less robust to noise compared to gradient descent methods (see e.g. \citet{can2019accelerated,kuru-privacy,flammarion2015averaging,mohammadi2020robustness,devolder2014first}), and achieving this result requires significantly more work 
compared to the analogous result we obtained for DE-SGLD. The proof of this result (and all the other lemmas) are given in the Appendix.

{We first provide uniform $L^2$ bounds on the iterates $v^{(k)}, x^{(k)}$ in \eqref{eq:underdamped}--\eqref{eq:underdamped2}
in the following lemma.} 

\begin{lemma}\label{lemma-l2-underdamped} 
Under the assumptions of Theorem~\ref{thm:underdamped},
there exist constants $c_4$ and $c_5$ (that do not depend on $\eta$ or $\gamma$) that can be made explicit such that 
\begin{align} 
&\sup_{k\geq 1} \mathbb{E} \left[\left\| \k{x} + \frac{\beta}{1-\beta} \left( \k{x} - \km{x}\right)\right\|^2\right] \leq c_4\,,
\label{ineq-unif-l2}
\\
&\sup_{k\geq 1} \max\left(\mathbb{E} \left\|v^{(k)}\right\| ^2, \mathbb{E} \left\|x^{(k)}\right\| ^2\right) \leq c_5\,.
 \label{ineq-unif-l2-iter}
\end{align}
\end{lemma}


With this lemma, we can bound the deviation of $x_{i}^{(k)}$ in \eqref{eq:under-xik} from the mean $\bar{x}^{(k)}$ in \eqref{eq:under-avg}. We state the result in the next subsection.

\subsubsection{Uniform $L^2$ bounds on the deviation from the mean}

\begin{lemma}\label{lem:1:under}
Under the assumptions of Theorem~\ref{thm:underdamped}, 
for any $k$, we have
\begin{equation*}
\sum_{i=1}^{N}\mathbb{E}\left\Vert x_{i}^{(k)}-\bar{x}^{(k)}\right\Vert^{2}
\leq
2\bar{\gamma}^{2k}\mathbb{E}\left\Vert x^{(0)}\right\Vert^{2}
+\frac{2c_{5}\eta^{2}}{(1-\bar{\gamma})^{2}},
\end{equation*}
where $c_{5}$ is defined in Lemma~\ref{lemma-l2-underdamped} and
$\bar{\gamma}=\max\left\{\left|\lambda_{2}^{W}\right|,\left|\lambda_{N}^{W}\right|\right\} \in [0,1).$ 
\end{lemma}

Note that we have
\begin{align}
&\bar{v}^{(k+1)}=\bar{v}^{(k)}-\gamma\eta\bar{v}^{(k)}-\eta\frac{1}{N}\nabla f\left(\bar{x}^{(k)}\right)
+\eta\mathcal{E}_{k+1}
-\eta\bar{\xi}^{(k+1)}
+\sqrt{2\gamma\eta}\bar{w}^{(k+1)}, \label{eq:mean-v}
\\
&\bar{x}^{(k+1)}=\bar{x}^{(k)}+\eta\bar{v}^{(k+1)}, \label{eq:mean-x}
\end{align}
where
$\mathcal{E}_{k+1}
:=\frac{1}{N}\nabla f\left(\bar{x}^{(k)}\right)-\frac{1}{N}\sum_{i=1}^{N}\nabla f_{i}\left(x_{i}^{(k)}\right)$. As a corollary of Lemma~\ref{lem:1:under}, we have the following estimate.

\begin{lemma}\label{lem:2:under}
Under the assumptions of Theorem~\ref{thm:underdamped},
for any $k$, we have
\begin{align}\label{bound:E}
\mathbb{E}\left\Vert\mathcal{E}_{k+1}\right\Vert^{2}
\leq
\frac{2L^{2}\bar{\gamma}^{2k}}{N}\mathbb{E}\left\Vert x^{(0)}\right\Vert^{2}
+\frac{2L^{2}c_{5}\eta^{2}}{N(1-\bar{\gamma})^{2}}.
\end{align}
\end{lemma}

\subsubsection{$L^2$ distance between the mean and discretized overdamped SDE} 
Given the dynamics of the average iterate $(\bar{v}^{(k)}, \bar{x}^{(k)})$ in \eqref{eq:mean-v}--\eqref{eq:mean-x}, we next show $\bar{x}^{(k)}$ is close to the iterates $\tilde{x}_k$ in \eqref{eq:tilde-x-k}, which is close to the iterates $x_{k}$ in \eqref{eq:under-over} obtained from an Euler-Maruyama discretization of an overdamped Langevin SDE.
By plugging \eqref{eq:mean-v} into \eqref{eq:mean-x}, we get
\begin{equation}\label{eq:becomes}
\bar{x}^{(k+1)}=\bar{x}^{(k)}+\eta\bar{v}^{(k)}-\gamma\eta^{2}\bar{v}^{(k)}-\eta^{2}\frac{1}{N}\nabla f\left(\bar{x}^{(k)}\right)
+\eta^{2}\mathcal{E}_{k+1}
- \eta^2 \bar{\xi}^{(k+1)}
+\sqrt{2\gamma\eta}\eta\bar{w}^{(k+1)}.
\end{equation}
By \eqref{eq:mean-v}, we get
$\bar{v}^{(k)}=\frac{\bar{x}^{(k)}-\bar{x}^{(k-1)}}{\eta}$,
so that \eqref{eq:becomes} becomes:
\begin{equation*}
\bar{x}^{(k+1)}=\bar{x}^{(k)}-\eta^{2}\frac{1}{N}\nabla f\left(\bar{x}^{(k)}\right)
+\beta\left(\bar{x}^{(k)}-\bar{x}^{(k-1)}\right)
+\eta^{2}\mathcal{E}_{k+1}
- \eta^2 \bar{\xi}^{(k+1)}
+\sqrt{2(1-\beta)}\eta\bar{w}^{(k+1)},
\end{equation*}
where we recall that $\beta=1-\gamma\eta$. Also recall that we define $\tilde{x}_{k}$ from the iterates:
\begin{align}
\tilde{x}_{k+1}=\tilde{x}_{k}-\eta^{2}\frac{1}{N}\nabla f(\tilde{x}_{k})+\sqrt{2(1-\beta)}\eta\bar{w}^{(k+1)},
\label{def-tilde-x-k}
\end{align}
where $\tilde{x}_{0}=\frac{1}{N}\sum_{i=1}^{N}x_{i}^{(0)}$. 
We have the following estimate.

\begin{lemma}\label{lem:3:under}
Under the assumptions of Theorem~\ref{thm:underdamped},
we have for every $k$,
\begin{align*}
\mathbb{E}\left\Vert\bar{x}^{(k)}-\tilde{x}_{k}\right\Vert^{2}\nonumber
&\leq
2\left(\frac{\eta^{2}}{\mu(1-\frac{\eta^{2} L}{2})}+\frac{(1+\eta^{2} L)^{2}}{\mu^{2}(1-\frac{\eta^{2} L}{2})^{2}}\right)
\left(\frac{\beta^{2}c_{5}}{\eta^{2}N}
+\frac{2L^{2}c_{5}\eta^{2}}{N(1-\bar{\gamma})^{2}}
\right)
+\frac{\eta^{2}\sigma^{2}}{\mu(1-\frac{\eta^{2}L}{2})N}
\nonumber
\\
&\qquad\qquad
+\frac{\bar{\gamma}^{2k}-
\left(1-\eta^{2}\mu\left(1-\frac{\eta^{2} L}{2}\right)\right)^{k}}
{\bar{\gamma}^{2}-1+\eta^{2}\mu\left(1-\frac{\eta^{2} L}{2}\right)}
\frac{4L^{2}\bar{\gamma}^{2}}{N}\mathbb{E}\left\Vert x^{(0)}\right\Vert^{2},
\end{align*}
where the constant $c_{5}$ is as in Lemma~\ref{lemma-l2-underdamped}.
\end{lemma}

Next, recall the iterates $x_k$ defined in \eqref{eq:under-over}:
\begin{align*}
x_{k+1}=x_{k}-\eta^{2}\frac{1}{N}\nabla f(x_{k})+\sqrt{2}\eta\bar{w}^{(k+1)},
\end{align*}
where $x_{0}=\tilde{x}_{0}=\bar{x}^{(0)}$.
This is an Euler-Maruyama discretized version 
of the continuous-time overdamped Langevin diffusion
with stepsize $\eta^{2}$.
Since $\beta=1-\gamma\eta$ is small (see \eqref{ineq-momentum-cond}), 
we will then show that $\tilde{x}_{k}$ and $x_{k}$ are close to each other in $L^{2}$ distance.
Indeed, we have the following estimate.

\begin{lemma}\label{lem:4:under}
Under the assumptions of Theorem~\ref{thm:underdamped},
we have for every $k$,
\begin{align*}
\mathbb{E}\left\Vert\tilde{x}_{k}-x_{k}\right\Vert^{2}\nonumber
\leq
2\left(\frac{\eta^{2}}{\mu(1-\frac{\eta^{2} L}{2})}+\frac{(1+\eta^{2} L)^{2}}{\mu^{2}(1-\frac{\eta^{2} L}{2})^{2}}\right)
\left(\frac{(\sqrt{1-\beta}-1)^{2}}{\eta^{2}}\frac{d}{N}\right).
\end{align*}
\end{lemma}



\subsubsection{Proof of Theorem~\ref{thm:underdamped}}

Since $x_{0}=\frac{1}{N}\sum_{i=1}^{N}x_{i}^{(0)}$, 
we have $\mathbb{E}\Vert x_{0}\Vert^{2}<\infty$.
By assumption we have also $\eta^{2}\leq\frac{2N}{\mu+L}$. Then, it follows from Lemma~\ref{lem:DK} and Lemma~\ref{bound:Gibbs} that
for $x_k$ defined in \eqref{eq:under-over} we have
\begin{equation*}
\mathcal{W}_{2}\left(\mathcal{L}(x_{k}),\pi\right)
\leq
\left(1-\mu\eta^{2}\right)^{k}\left(\left(\mathbb{E}\Vert x_{0}-x_{\ast}\Vert^{2}\right)^{1/2}+\sqrt{2\mu^{-1}dN^{-1}}\right)
+\frac{1.65L}{\mu}\sqrt{\eta^{2} dN^{-1}}.
\end{equation*}
Moreover, it follows from Lemma~\ref{lem:3:under} that
\begin{align*}
&\mathcal{W}_{2}\left(\mathcal{L}\left(\bar{x}^{(k)}\right),\mathcal{L}(\tilde{x}_{k})\right)
\\
&\leq
\left(\mathbb{E}\left\Vert\bar{x}^{(k)}-\tilde{x}_{k}\right\Vert^{2}\right)^{1/2}
\\
&\leq
\sqrt{2}\left(\frac{\eta^{2}}{\mu(1-\frac{\eta^{2} L}{2})}+\frac{(1+\eta^{2} L)^{2}}{\mu^{2}(1-\frac{\eta^{2} L}{2})^{2}}\right)^{1/2}
\left(\frac{\beta^{2}c_{5}}{\eta^{2}N}
+\frac{2L^{2}c_{5}\eta^{2}}{N(1-\bar{\gamma})^{2}}\right)^{1/2}
\\
&\quad
+\frac{\eta\sigma}{\sqrt{\mu(1-\frac{\eta^{2}L}{2})N}}
+\left(\frac{\bar{\gamma}^{2k}-
\left(1-\eta^{2}\mu\left(1-\frac{\eta^{2} L}{2}\right)\right)^{k}}
{\bar{\gamma}^{2}-1+\eta^{2}\mu\left(1-\frac{\eta^{2} L}{2}\right)}\right)^{1/2}
\frac{2L\bar{\gamma}}{\sqrt{N}}\left(\mathbb{E}\left\Vert x^{(0)}\right\Vert^{2}\right)^{1/2},
\end{align*}
whereas it follows from Lemma~\ref{lem:4:under} that
\begin{align*}
\mathcal{W}_{2}\left(\mathcal{L}(\tilde{x}_{k}),\mathcal{L}(x_{k})\right)
&\leq
\left(\mathbb{E}\left\Vert\tilde{x}_{k}-x_{k}\right\Vert^{2}\right)^{1/2}
\\
&\leq
\sqrt{2}\left(\frac{\eta^{2}}{\mu(1-\frac{\eta^{2} L}{2})}+\frac{(1+\eta^{2} L)^{2}}{\mu^{2}(1-\frac{\eta^{2} L}{2})^{2}}\right)^{1/2}
\left(\frac{(\sqrt{1-\beta}-1)^{2}}{\eta^{2}}\frac{d}{N}\right)^{1/2}.
\end{align*}
Hence, we conclude that
\begin{align}
\mathcal{W}_{2}\left(\mathcal{L}\left(\bar{x}^{(k)}\right),\pi\right)
&\leq
\left(1-\mu\eta^{2}\right)^{k}\left(\left(\mathbb{E}\left\Vert \bar{x}^{(0)}-x_{\ast}\right\Vert^{2}\right)^{1/2}+\sqrt{2\mu^{-1}dN^{-1}}\right)
\nonumber
\\
&\quad
+\left(\frac{
\left(1-\eta^{2}\mu\left(1-\frac{\eta^{2} L}{2}\right)\right)^{k}-\bar{\gamma}^{2k}}
{\left(1-\eta^{2}\mu\left(1-\frac{\eta^{2} L}{2}\right)\right)-\bar{\gamma}^{2}}\right)^{1/2}
\frac{2L\bar{\gamma}}{\sqrt{N}}\left(\mathbb{E}\left\Vert x^{(0)}\right\Vert^{2}\right)^{1/2}  + \eta E_4, 
\end{align}
with
\begin{align*}
    E_4 &:=\sqrt{2}\left(\frac{\eta^{2}}{\mu(1-\frac{\eta^{2} L}{2})}+\frac{(1+\eta^{2} L)^{2}}{\mu^{2}(1-\frac{\eta^{2} L}{2})^{2}}\right)^{1/2}
\\
&\qquad\qquad\cdot
\left[\left(\frac{\beta^{2}c_{5}}{\eta^{4}N}
+\frac{2L^{2}c_{5}}{N(1-\bar{\gamma})^{2}}\right)^{1/2}
+\left(\frac{(\sqrt{1-\beta}-1)^{2}}{\eta^{4}}\frac{d}{N}\right)^{1/2}\right]
\\
&\qquad\qquad\qquad+ \frac{1.65L}{\mu}\sqrt{dN^{-1}} + \frac{\sigma}{\sqrt{\mu(1-\frac{\eta^{2}L}{2})N}}\,.
\end{align*}
Finally, by \eqref{by:cauchy:schwarz}, we have
\begin{align*}
\frac{1}{N}\sum_{i=1}^{N}\mathcal{W}_{2}\left(\mathcal{L}\left(x_{i}^{(k)}\right),\mathcal{L}\left(\bar{x}^{(k)}\right)\right)
\leq
\sqrt{\frac{1}{N}\sum_{i=1}^{N}\mathbb{E}\left\Vert x_{i}^{(k)}-\bar{x}^{(k)}\right\Vert^{2}}.
\end{align*}
On the other hand, by Lemma~\ref{lem:1:under},
\begin{align*}
\sqrt{\frac{1}{N}\sum_{i=1}^{N}\mathbb{E}\left\Vert x_{i}^{(k)}-\bar{x}^{(k)}\right\Vert^{2}}
&\leq
\left(\frac{2\bar{\gamma}^{2k}}{N}\mathbb{E}\left\Vert x^{(0)}\right\Vert^{2}
+\frac{2c_{5}\eta^{2}}{N(1-\bar{\gamma})^{2}}\right)^{1/2}
\\
&\leq
\frac{\sqrt{2}\bar{\gamma}^{k}}{\sqrt{N}}\left(\mathbb{E}\left\Vert x^{(0)}\right\Vert^{2}\right)^{1/2}
+\frac{\sqrt{2c_{5}}\eta}{\sqrt{N}(1-\bar{\gamma})}.
\end{align*}
We then obtain \eqref{ineq-W2-bound-underdamped} by applying the triangular inequality for the $2$-Wasserstein distance. Finally, since $\beta$ satisfies the inequality \eqref{ineq-momentum-cond}, we have $\beta = \mathcal{O}(\eta^4)$ as $\eta \to 0$ and this implies that 
$E_4 = \mathcal{O}(1)$ 
and $E_5 = \mathcal{O}(1)$ as claimed. This completes the proof.
\hfill $\Box$

\section{Numerical Experiments}\label{sec:numerical}
We present our numerical results in this section. We conduct several experiments to validate our theory and investigate the performance of DE-SGLD and DE-SGHMC. We focus on applying our methods to Bayesian linear regression and Bayesian logistic regression problems. In our experiments,  each agent has its own data in the form of i.i.d. samples. We will consider three different network architectures: (a) Fully-connected network (b) Circular network (c) A disconnected network with no edges 
as illustrated in Figure~\ref{fig:network}. Fully-connected network structure corresponds to the complete graph where all the nodes are connected to each other whereas for the circular graph, each node can communicate with only ``left" and ``right" neighbors. Disconnected graph corresponds to the case when nodes do not communicate at all with each other.  
The disconnected network is considered as a baseline case for comparison purposes to see how the individual agents would perform without sharing any information among themselves.

\rev{Before we proceed to the numerical experiments, we remark that the following examples
all satisfy the assumptions in our paper. We will have a discussion on this
in the Appendix in detail. In particular, Appendix~\ref{sec:gradient:noise:assump}
shows that the variance of the gradient noise is bounded, and 
Appendix~\ref{sec:assump:numerical} shows that the gradient of the component functions are Lipschitz. }

\begin{figure}[t]
\centering
    \subfigure[Fully-connected]{
    \includegraphics[width=0.3\columnwidth]{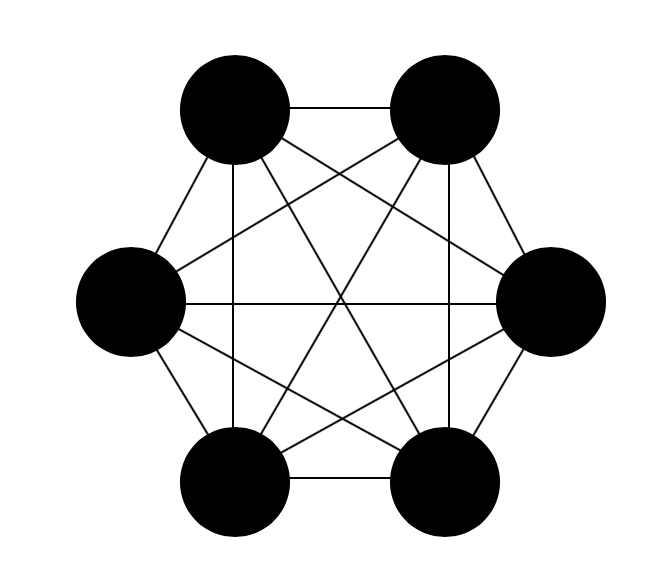}
    }
    \hfill
    \subfigure[Circular]{
    \includegraphics[width=0.3\columnwidth]{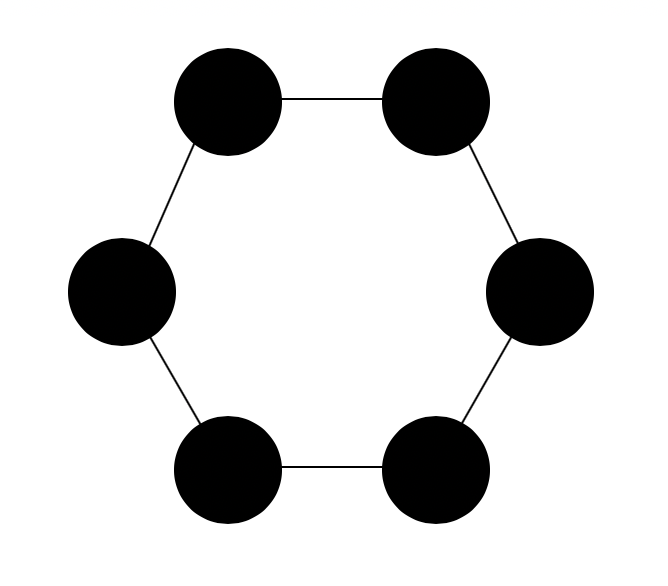}
    }
    \hfill
    \subfigure[Disconnected]{
    \includegraphics[width=0.3\columnwidth]{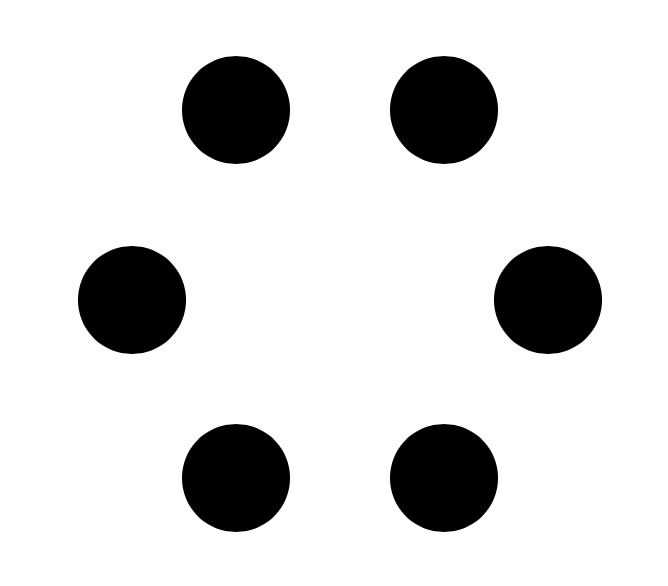}
    }
\caption{Illustration of the network architectures.}
\label{fig:network}
\end{figure}

\subsection{Bayesian linear regression}

In this section, we present our experiments on the Bayesian linear regression problem, where our main goal is to validate Theorems~\ref{thm:overdamped} and  \ref{thm:underdamped} in a basic setting and show that each agent can sample from the posterior distribution up to an error tolerance with constant stepsize. In this set of experiments, we first generate data for each agent by simulating the model:
\begin{equation}
    \delta_j \sim \mathcal{N}\left(0,\xi^2\right),\quad  X_j \sim \mathcal{N}(0,I),\quad  y_j=x^TX_j+\delta_j,
\label{num:lin_gen}
\end{equation}
where the noise term $\delta_j$ are i.i.d. scalars with $\xi=1$, $x \in \mathbb{R}^2$, and the prior distribution of $x$ follows $\mathcal{N}(0,\lambda I)$ where we take $\lambda=10$ in the experiments. 
For the Bayesian linear regression, we can derive the posterior distribution as:
\begin{equation*}
    \pi(x) \sim \mathcal{N}(m,V),\quad m= \left(\Sigma^{-1}+X^TX/\xi^2\right)^{-1}\left(X^Ty/\xi^2\right),\quad V = \left(X^TX/\xi^2+\Sigma^{-1}\right)^{-1},
\end{equation*} 
where $\Sigma=\lambda I$ is the covariance matrix of the prior distribution of $x$, $X=[X_1^T,X_2^T,\dots]^T$ and $Y=[y_1,y_2,\dots]^T$ are the matrices containing all data points. We simulate 5,000 data points and partition them randomly among the $N=100$ agents so that each agent will have the same number of data points. Each agent has access to its own data but not to other agents' data. The posterior distribution $\pi(x) \propto e^{-f(x)}$ is of the form $f(x)=\sum_{i=1}^N f_i(x)$
with
$$f_i(x) := -\sum_{j=1}^{n_i}\log p\left(y_j^i | x,X_j^i\right) - \frac{1}{N} \log p(x) = 
\sum_{j=1}^{n_i} \left(y_j^i - x^T X_j^i\right)^2 + \frac{1}{2\lambda N}\|x\|^2,
$$
where
$$p\left(y_j^i | x,X_j^i\right)=\frac{1}{\sqrt{2\pi \xi^2}}e^{-\frac{1}{2\xi^2} (y_j^i-x^TX_j^i)^2}, \quad p(x)\propto
e^{-\frac{1}{2\lambda}\|x\|^2}, $$
and agent $i$ has $n_i=50$ data points $\{(X_j^i,y_j^i)\}_{j=1}^{n_i}$. 

In the first experiment, we report the performance of the DE-SGLD method on the fully-connected, circular and disconnected networks in Figure~\ref{fig:1exp}. We tune the stepsize $\eta$ to the dataset where we take $\eta=0.009$.  We consider the case when gradients are deterministic (i.e. $\sigma=0$). In this case, it can be seen that $x_i^{(k)}$ follows a Gaussian distribution, i.e.  $x_i^{(k)}\sim \mathcal{N}\left(m_i^{(k)},\Sigma_i^{(k)}\right)$ for some mean vector $m_i^{(k)}$ and covariance matrix $\Sigma_i^{(k)}$. Based on 100 independent runs, we estimate the parameters $m_i^{(k)}$ and $\Sigma_i^{(k)}$ and then compute the 2-Wasserstein distance with respect to the posterior distribution $\pi(x) \sim \mathcal{N}(m,V)$ based on the explicit formula \citep{givens1984class} which characterizes the 2-Wasserstein distance between any two Gaussian distributions. This allows us to plot the 2-Wasserstein distance to the stationary distribution for each agent and for the distribution of the average $\bar{x}^{(k)}=\sum_{i=1}^N x_i^{(k)}/N$ over the iterations in Figure~\ref{fig:1exp}. We observe that for both complete and circular graphs all the agents will converge to the posterior distribution up to an error tolerance. In the case of the disconnected network, we observe that individual agents do relatively worse compared to the fully-connected and circular network cases; as they do not leverage any information about their neighbors' data points. For the disconnected case, the nodes averages $\bar{x}^{(k)}$ (which is neither computed nor accessible by agents) is closer to the target distribution than the individual iterates $x_i^{(k)}$ as it contains information from each agent; however the performance of the node averages $\bar{x}^{(k)}$ is still worse compared to the performance of node averages for the fully-connected case as expected. We observe that the experiments in the fully-connected network converges faster than the circular network. This behavior is predicted by Theorem~\ref{thm:overdamped}.
Since the fully-connected network has a larger \emph{spectral gap} $1-\bar{\gamma}$ compared to the circular network (see the paragraph after \eqref{eq:gamma}), our performance bounds for the fully-connected network is better compared to the circular network case.\footnote{This is a consequence of the fact that our upper bounds given for the 2-Wasserstein distances to the target $\pi$ in Theorem~\ref{thm:overdamped} and Theorem~\ref{thm:underdamped} are both monotonically increasing in $\bar{\gamma}$ (see Remark \ref{remark-spectral-gap} and Remark \ref{remark-sghmc}).
}

\begin{figure}[t]
\centering
    \subfigure[Fully-connected]{
    \includegraphics[width=0.3\columnwidth]{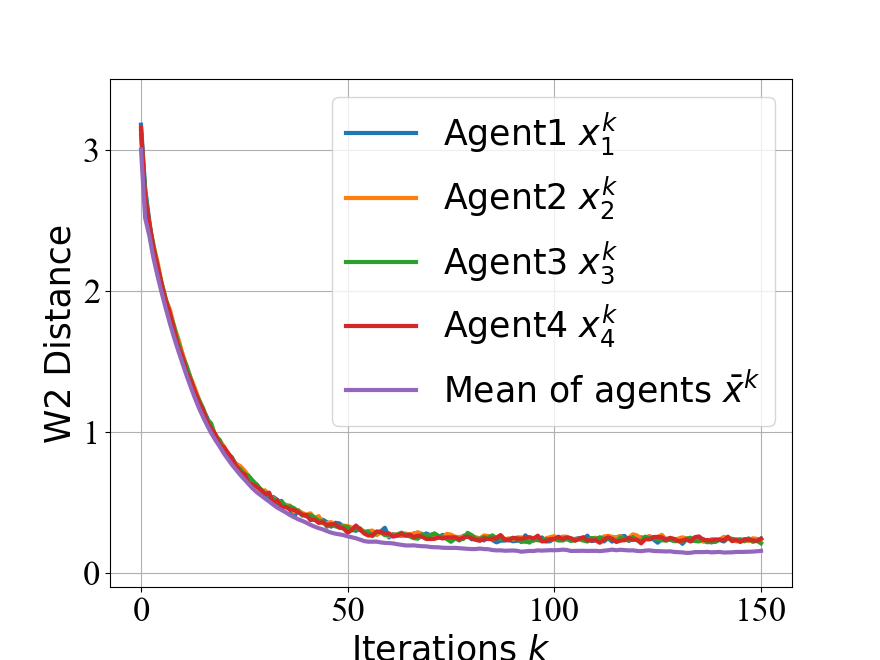}
    }
    \hfill
    \subfigure[Circular]{
    \includegraphics[width=0.3\columnwidth]{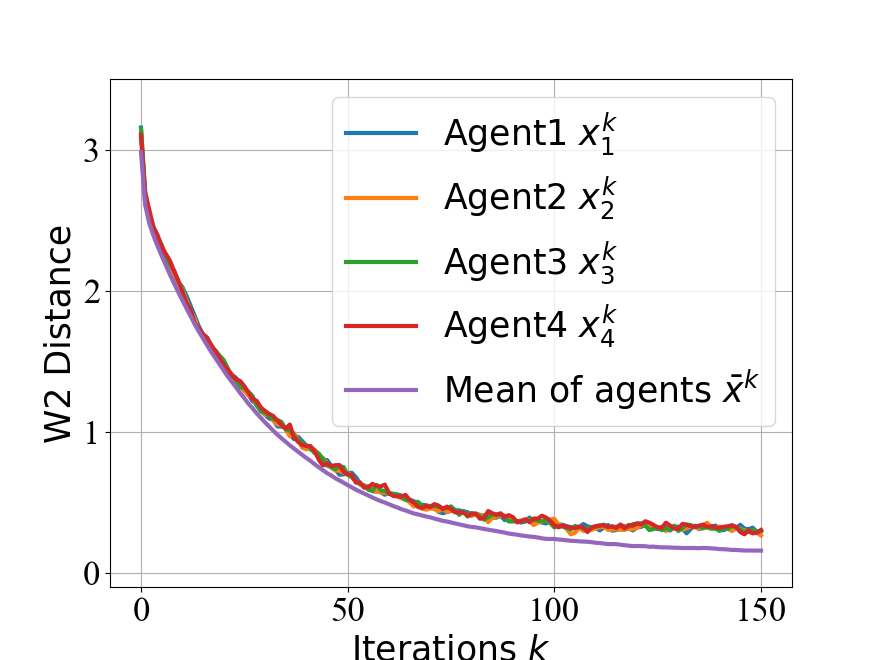}
    }
    \hfill
    \subfigure[Disconnected]{
    \includegraphics[width=0.3\columnwidth]{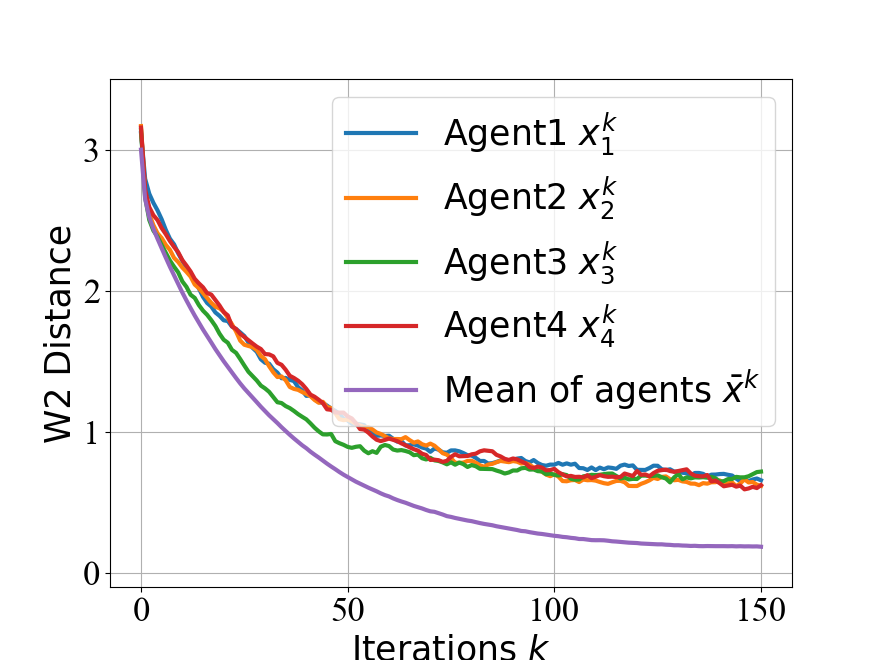}
    }
\caption{Performance of DE-SGLD for Bayesian regression on different network structures with $N=100$ agents. The results of the first 4 agents $x_i^{k}$ and the node averages $\bar{x}^k=\sum_{i=1}^N x_i^{(k)}/N$ are reported. 
} 
\label{fig:1exp}
\end{figure}

In the next experiment, we investigate the performance of the DE-SGHMC method on the same data set with (the same) three network structures. The stepsize $\eta$ and the friction coefficient $\gamma$ are tuned to the dataset where we take $\eta=0.1$ and $\gamma=7$. The results are displayed in Figure~\ref{fig:2exp}. The results are qualitatively similar to the DE-SGLD case. The convergence of DE-SGHMC is fastest for the fully-connected case and is the slowest for the disconnected case. 

\begin{figure}[t]
\centering
    \subfigure[Fully-connected]{
    \includegraphics[width=0.3\columnwidth]{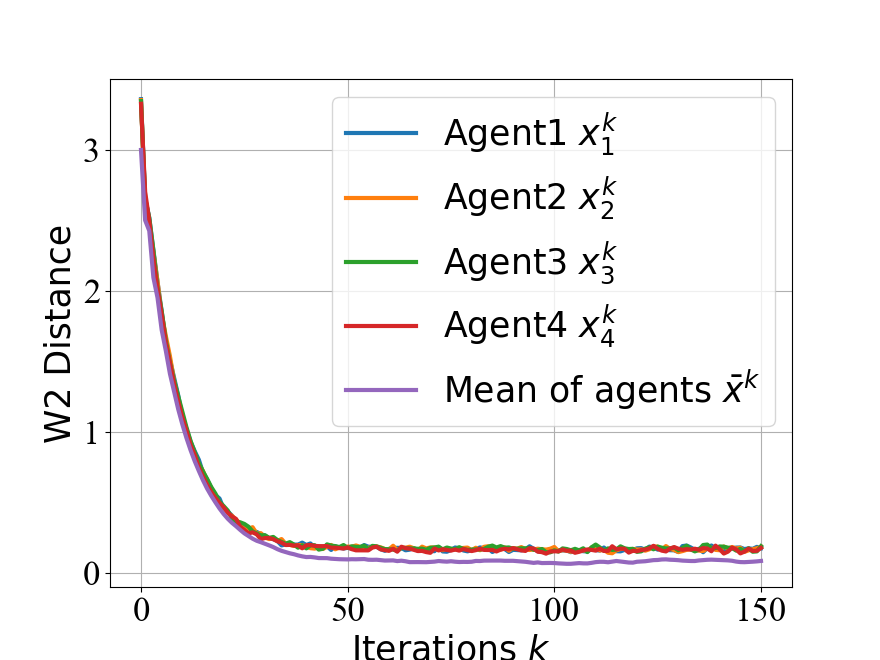}
    }
    \hfill
    \subfigure[Circular]{
    \includegraphics[width=0.3\columnwidth]{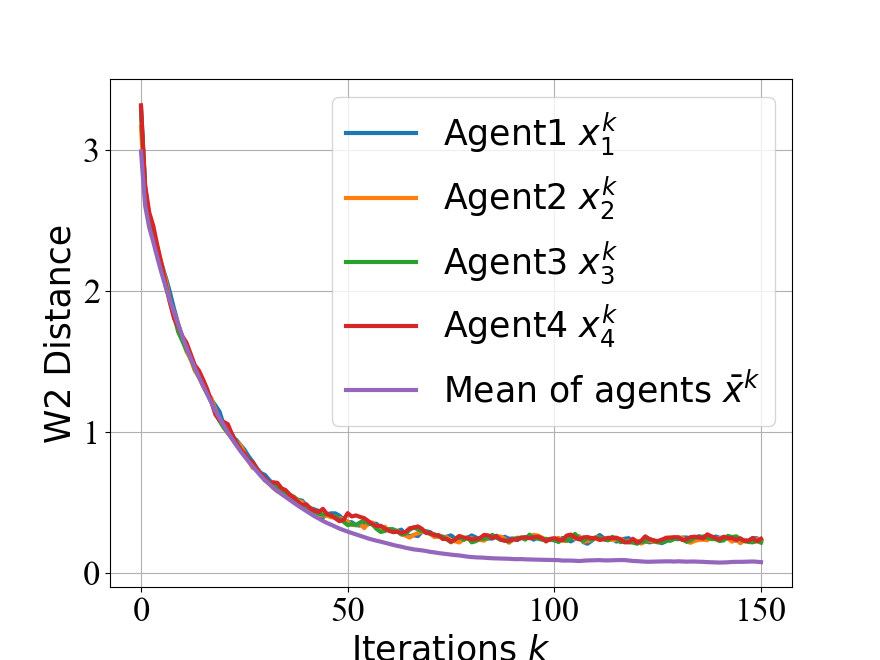}
    }
    \hfill
    \subfigure[Disconnected]{
    \includegraphics[width=0.3\columnwidth]{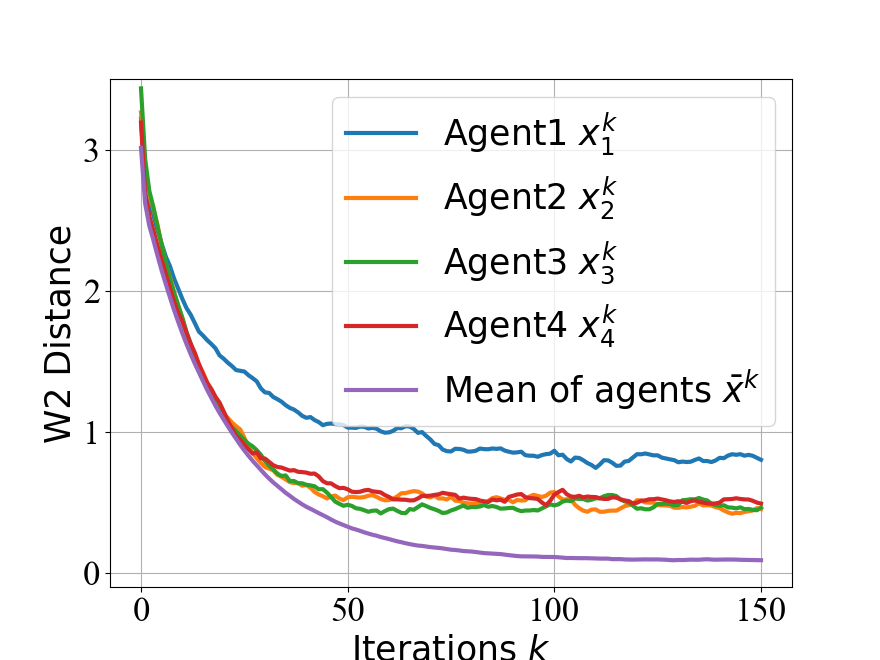}
    }
\caption{Performance of DE-SGHMC method for Bayesian regression on different network structures. The stepsize $\eta$ and the friction coefficient $\gamma$ are tuned to the dataset where we take $\eta=0.1$ and $\gamma=7$. 
} 
\label{fig:2exp}
\end{figure}

\begin{figure}[t]
\centering
    \subfigure[Batch size]{
    \includegraphics[width=0.3\columnwidth]{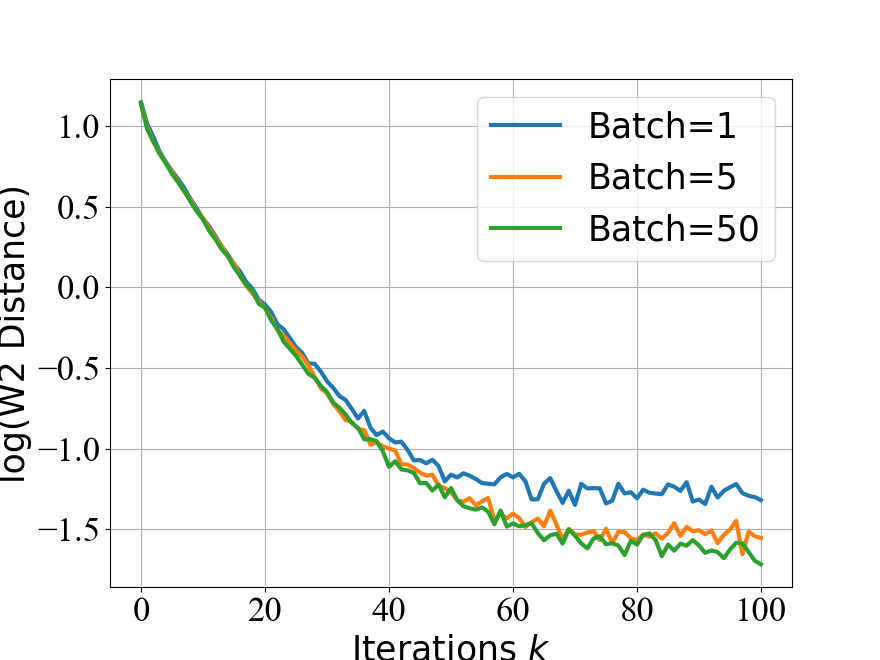}
    \label{fig:3_batch}
    }
    \hfill
    \subfigure[Stepsize]{
    \includegraphics[width=0.3\columnwidth]{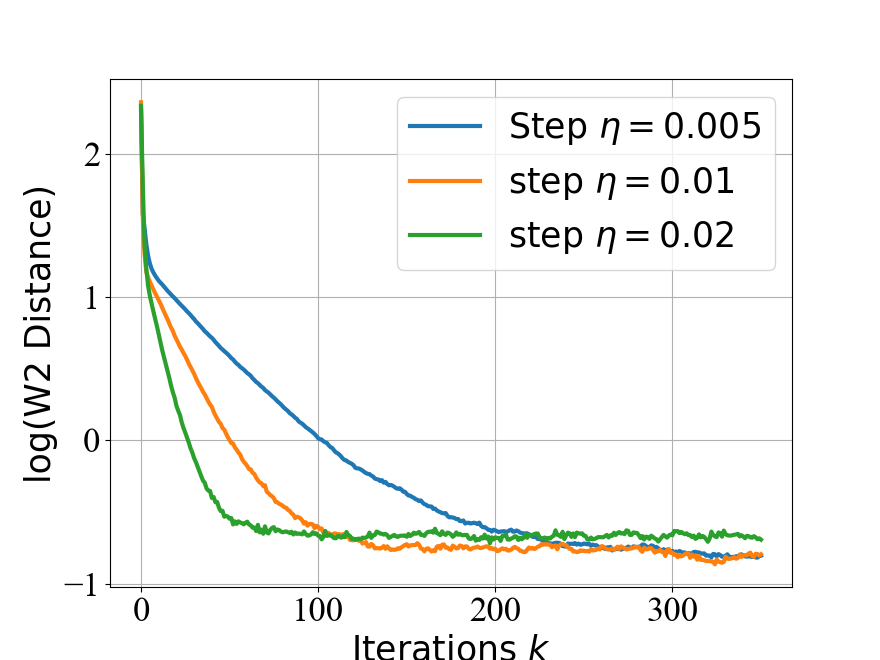}
    \label{fig:3_step}
    }
    \hfill
    \subfigure[Network structure]{
    \includegraphics[width=0.3\columnwidth]{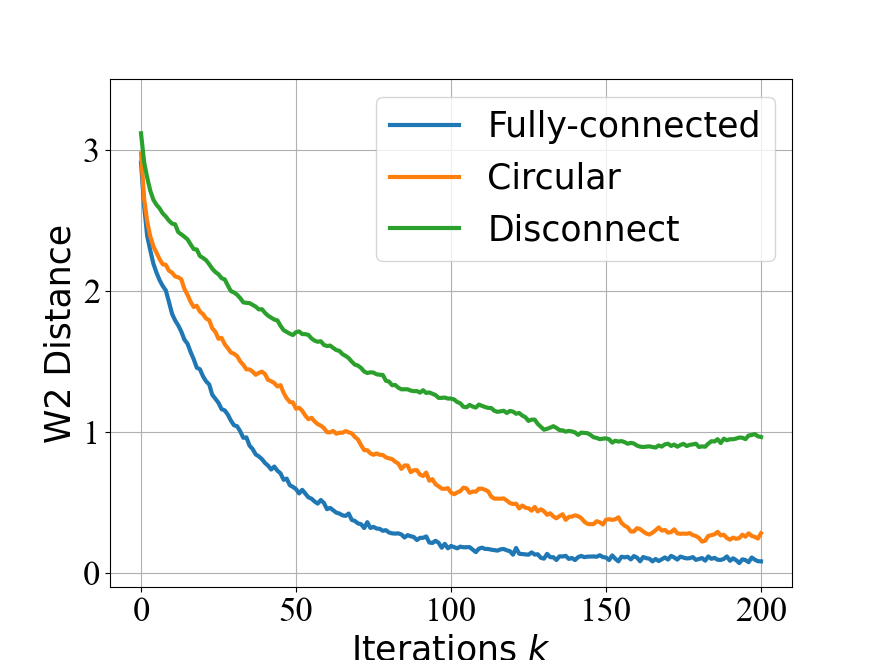}
    \label{fig:3_network}
    }
\caption{Performance of DE-SGLD method for Bayesian regression under different settings. Figures are based on one randomly picked agent. The y-axis is presented in a logarithmic scale in \subref{fig:3_batch} and \subref{fig:3_step}
. 
} 
\label{fig:3exp}
\end{figure}

In the next set of experiments, we investigate the effect of changing stepsize, batch size and the network structure on the speed of convergence where we stick to the DE-SGLD method for this set of experiments. We measure the 2-Wasserstein distance to the target $\pi$ with a similar approach as before by fitting a Gaussian distribution $\mathcal{N}(m_i^{(k)}, \Sigma_i^{(k)})$ to the empirical distribution of $x_i^{(k)}$ over 100 independent runs. The results are shown in Figure~\ref{fig:3exp}. Both Figure~\ref{fig:3_batch} and Figure~\ref{fig:3_step} are based on the fully-connected network architecture. In Figure~\ref{fig:3_batch}, we fix the stepsize to $\eta=0.009$ and vary the batch sizes (the number of data points sampled with replacement to estimate the gradient). We conclude that different batch sizes affect the asymptotic error the iterates have with respect to the 2-Wasserstein distance. Larger batch sizes reduce the amount of noise (i.e. the upper bound $\sigma^2$ on the gradient noise) and therefore lead to smaller asymptotic error as predicted by Theorem~\ref{thm:overdamped}. In Figure~\ref{fig:3_step}, we used stochastic gradients with batch size $b=25$ while we varied the stepsize. The result clearly demonstrates the trade-off between the convergence rate and the asymptotic accuracy; for larger stepsize the algorithm converges faster to an asymptotic error region but the accuracy becomes worse as predicted by Theorem~\ref{thm:overdamped} (see also Remark~\ref{remark-desgld}). In Figure~\ref{fig:3_network} we report the effect of network structure with a constant stepsize $\eta=0.008$ and batch size $b=25$ where we report the performance of a randomly picked agent. The fastest convergence is observed for the fully-connected network. 
For the disconnected network, each agent will converge to a stationary distribution based on its own data rather than the posterior distribution based on the whole data set; therefore the asymptotic error in 2-Wasserstein distance will be bounded away from zero. 

\subsection{Bayesian logistic regression}\label{subsec-bayesian-logistic}

In Bayesian logistic regression, we are given a dataset of input-output pairs $A = \{a_j\}_{j=1}^n$ where $a_j = (X_j, y_j)$,  $X_j \in R^d$ are the features and $y_j \in \{0,1\}$ are the binary labels. We assume that $X_j$ are independent, and that the probability distribution of the output $y_j$ given features $X_j$ and the regression coefficients
$x\in \mathbb{R}^d$ is given by
    \beq \mathbb{P}(y_j=1 | X_j, x) = \frac{1}{1+e^{-x^T X_j}}.
    \label{eq-logistic-proba}
    \eeq
The prior distribution $p(x)$ is often taken as a Gaussian distribution $\mathcal{N}(0,\lambda I)$ for some $\lambda>0$ (see e.g. \citet{chatterji2018theory,dubey2016variance,zou2018subsampled}). If each agent $i$ possesses a subset $A_i$ of the data where $A_i = \{(X_j^i, y_j^i) \}_{i=1}^{n_i}$, then the goal in Bayesian logistic regression is to sample from $\pi(x) \propto e^{-f(x)}$ with $f(x)=\sum_i f_i(x)$ where
\begin{equation} 
f_i(x) := -\sum_{j=1}^{n_i}\log p\left(y_j^i=1 | X_j^i,x\right) - \frac{1}{N} \log p(x)= \sum_{j=1}^{n_i}\log \left(1+e^{-x^T X_j^i}\right) + \frac{1}{2N\lambda}\|x\|^2
\end{equation}
is strongly convex and smooth.
We first test our algorithms on synthetic data where we simulate \eqref{eq-logistic-proba} by
\begin{equation*}
   {X}_j \sim \mathcal{N}(0,20I),\quad 
    p_j \sim \mathbb{}\mathcal{U}(0,1),\quad 
    y_j= \begin{cases} 1 &\mbox{if } p_j \leq \frac{1}{1+e^{-x^T X_j}} \\
0 & \mbox{otherwise }  \end{cases},
\end{equation*}
where $\mathcal{U}(0,1)$ is the uniform distribution on $[0,1]$,  $x = [x_1, x_2, x_3]^T\in \mathbb{R}^3$ and the prior distribution of $x$ follows $\mathcal{N}(0,\lambda I)$ where we take $\lambda=10$ in the experiments. 
Similar to the case of Bayesian linear regression, we separate the data points approximately equally among all the agents, where we take $N=6$.
Each agent can access to one part of the data set. 
Unlike Bayesian linear regression, where the posterior distribution admits an explicit formula, the posterior distribution $\pi(x)$ of Bayesian logistic regression does not admit an explicit formula. In principle, one can approximate the stationary distribution by running the algorithm over many runs and compute the Wasserstein distance between this approximate distribution and the empirical distribution of the iterates and report this distance as a performance measure. However, this is not practical. Instead, we resort to another performance measure for each agent, which is the distribution of the accuracy over the whole data set where accuracy is defined as the ratio of the correctly predicted labels. This ratio is relatively simpler to compute and serves as a measure correlated with the goodness of fit to the training data. For this purpose, we run the DE-SGLD method multiple times and for each realization of the $k$-th iterate $x_i^{(k)}$ at node $i$, we classify the whole data set based on $x_i^{(k)}$ and calculate the accuracy over $n=1,000$ data points. Over 100 independent runs of the DE-SGLD algorithm with batch size $b=32$, 
we
estimate the distribution of the accuracy for each agent at step $k$. We report the mean and the standard deviation of the accuracy in 
\begin{figure}[t]
\centering
    \subfigure[Fully-connected]{
    \includegraphics[width=0.3\columnwidth]{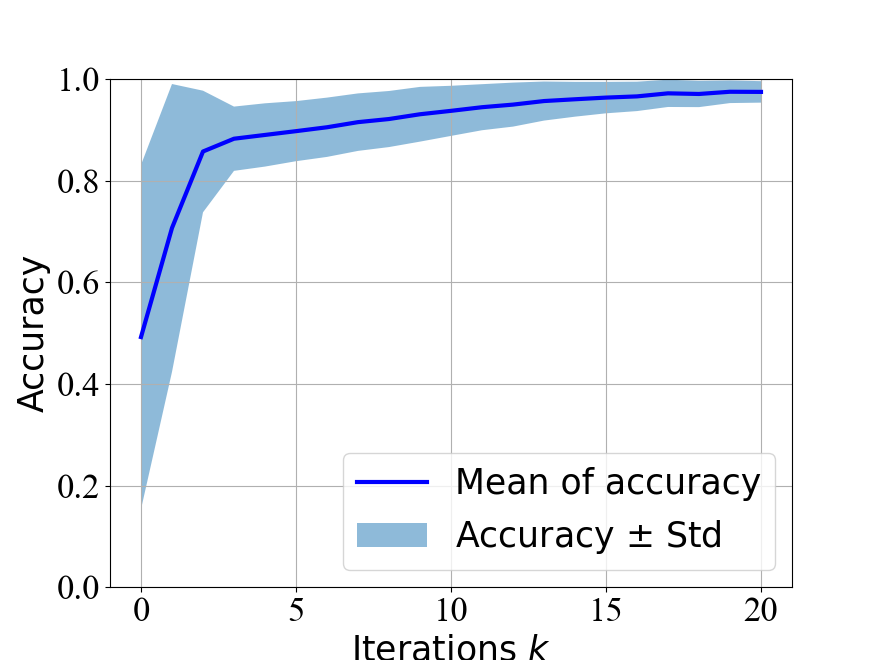}
    }
    \hfill
    \subfigure[Circular]{
    \includegraphics[width=0.3\columnwidth]{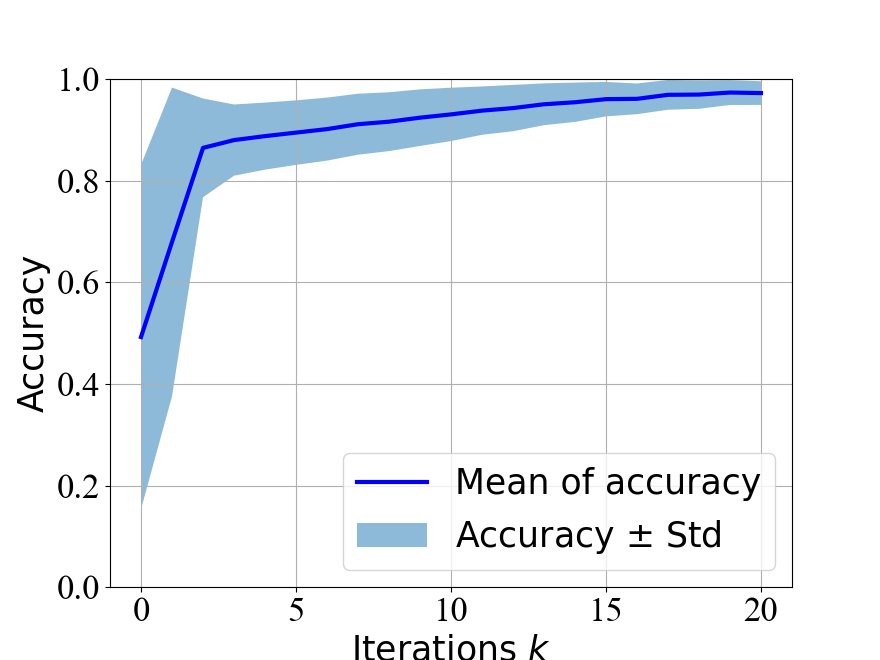}
    }
    \hfill
    \subfigure[Disconnected]{
    \includegraphics[width=0.3\columnwidth]{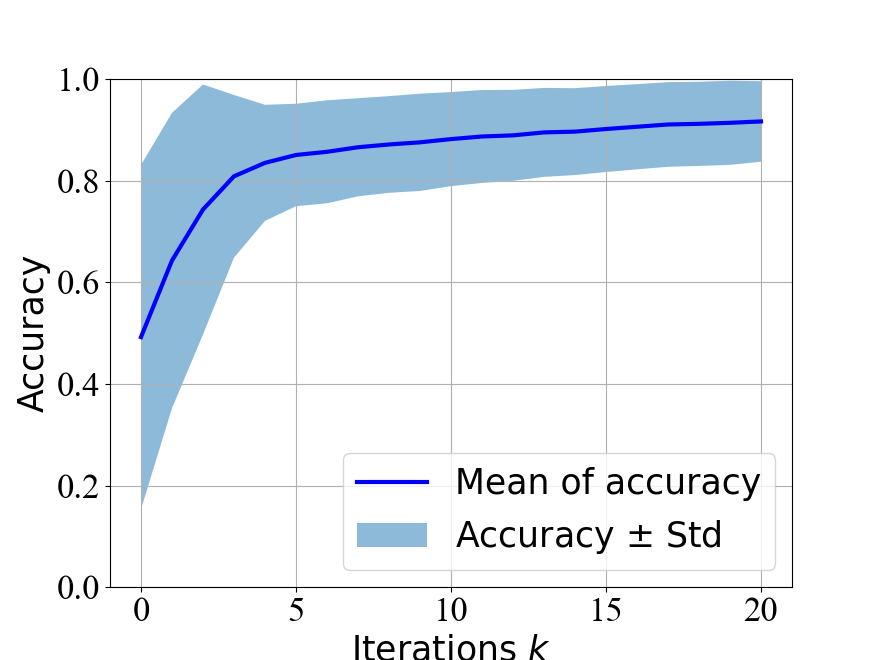}
    }
    \hfill
\caption{The plots show the the accuracy over the data set versus number of iterations for the DE-SGLD method on different network structures.  Figures are based on one randomly picked agent. Here, the stepsizes are tuned to the dataset where we take $\eta=0.0003$. We use the stochastic gradient with batch size $b=32$ in the experiments.}
\label{fig:4exp}
\end{figure}
Figure~\ref{fig:4exp}. We can clearly observe that the DE-SGLD method works well for both fully-connected and circular networks for Bayesian logistic regression, which supports our theory. In the right panel of Figure~\ref{fig:4exp}, we show the results of the DE-SGLD method for the disconnected network. The performance on the disconnected network is worse compared to the fully-connected and circular network settings as expected. 


In our next set of experiments, we investigate the DE-SGHMC method in Figure~\ref{fig:6exp} where we take $\eta=0.02$ and $\gamma=30$ after tuning these parameters to the dataset. We use the batch size $b=32$ in this set of experiments. We see that the performance of DE-SGHMC for fully-connected and circular networks is also better compared to the disconnected setting as expected. 

\begin{figure}[t]
\centering
    \subfigure[Fully-connected]{
    \includegraphics[width=0.3\columnwidth]{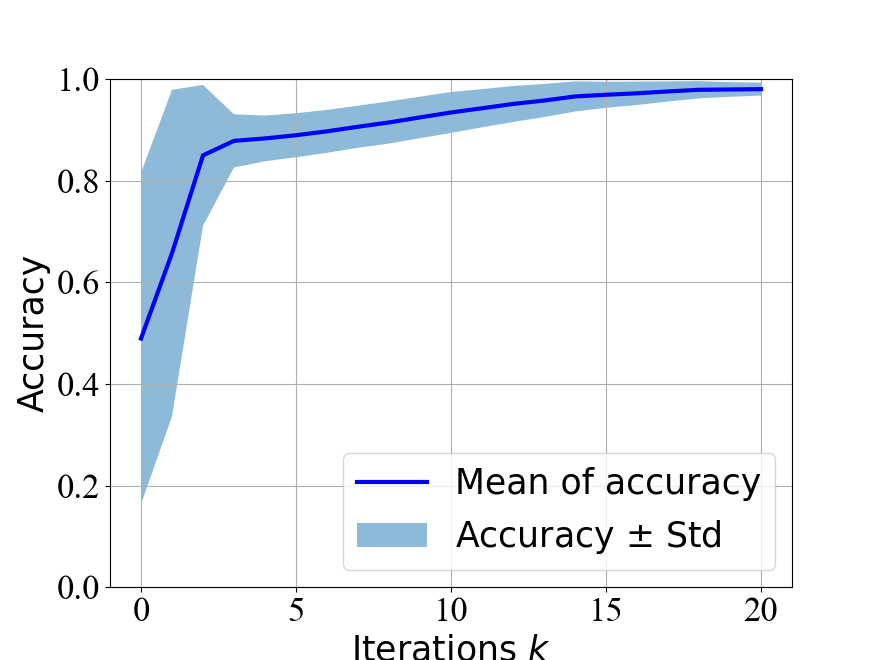}
    }
    \hfill
    \subfigure[Circular]{
    \includegraphics[width=0.3\columnwidth]{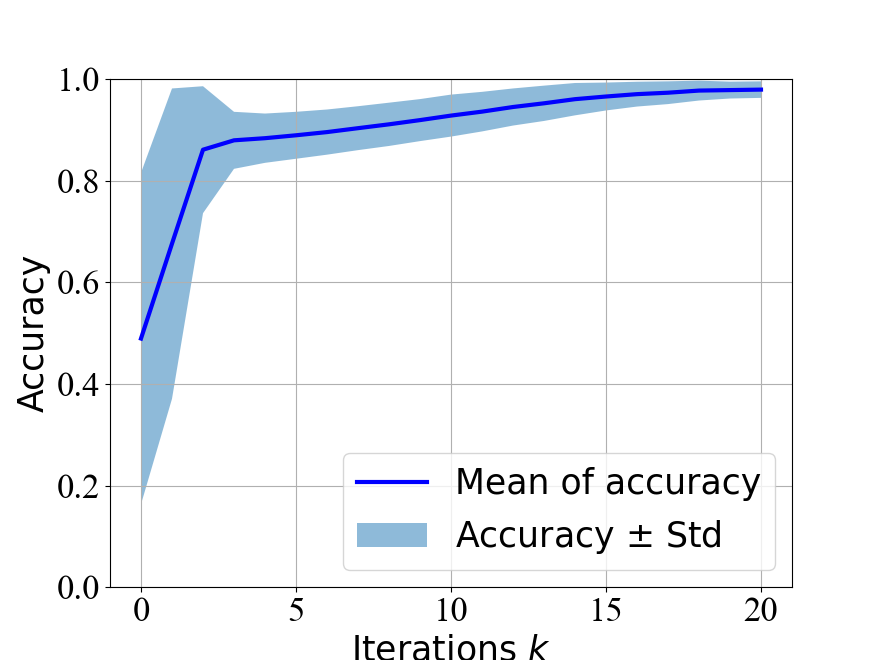}
    }
    \hfill
    \subfigure[Disconnected]{
    \includegraphics[width=0.3\columnwidth]{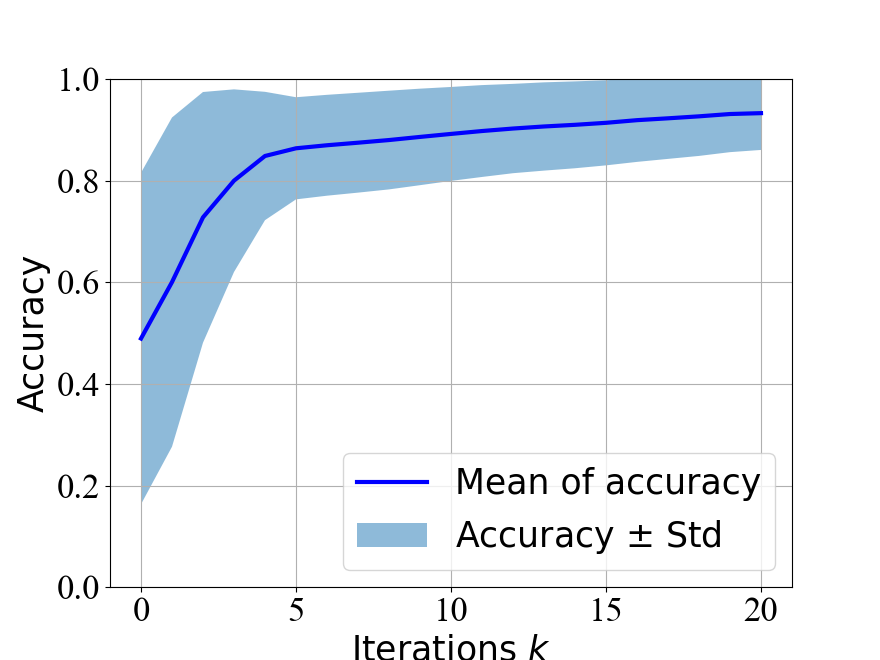}
    }
    \hfill
\caption{The plots show the accuracy over the data set versus number of iterations for the DE-SGHMC method on different network structures.  Figures are based on one randomly picked agent. Here, the stepsize $\eta$ and the friction coefficient $\gamma$ are tuned to the dataset where we take $\eta=0.02$ and $\gamma=30$. We use the stochastic gradient with batch size $b=32$ in the experiments.}
\label{fig:6exp}
\end{figure}



\begin{figure}[t]
\centering
    \subfigure[Fully-connected]{
    \includegraphics[width=0.3\columnwidth]{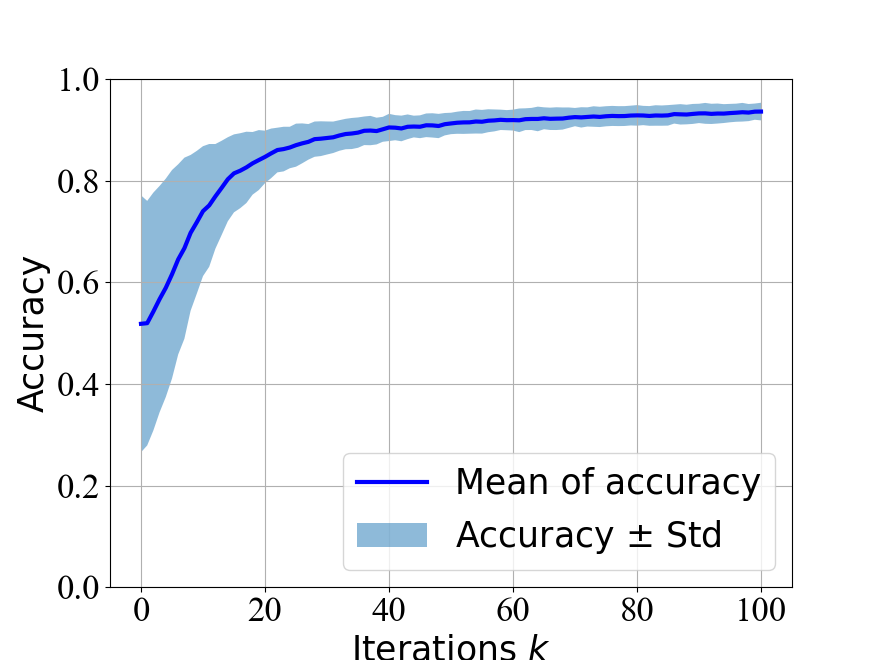}
    }
    \hfill
    \subfigure[Circular]{
    \includegraphics[width=0.3\columnwidth]{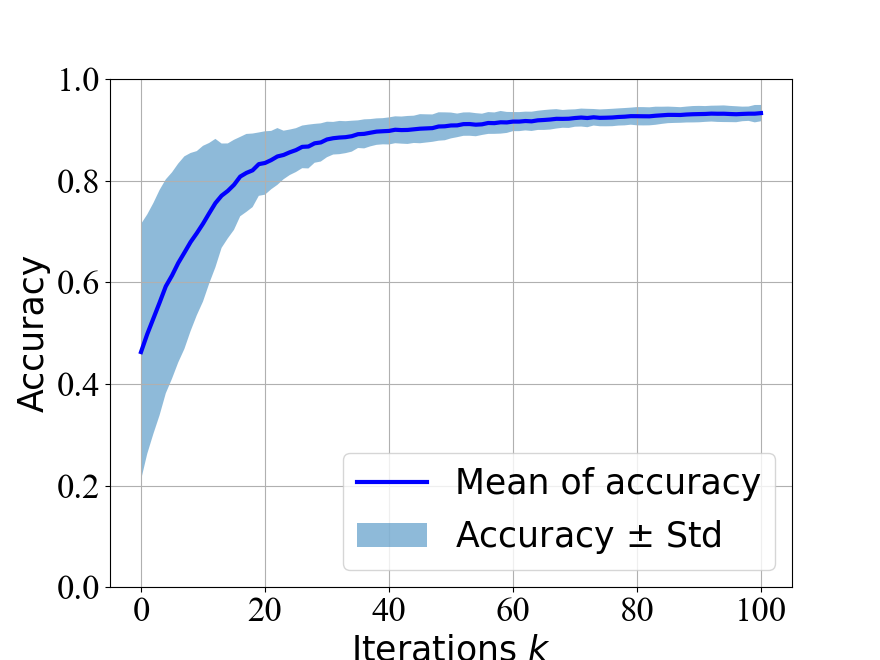}
    }
    \hfill
    \subfigure[Disconnected]{
    \includegraphics[width=0.3\columnwidth]{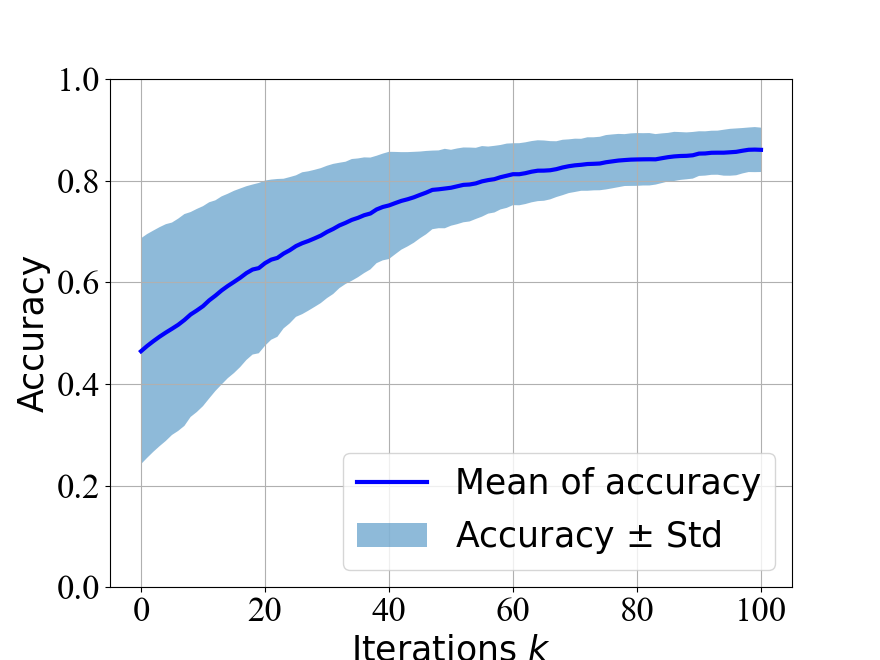}
    }
\caption{The plots show the accuracy over the data set versus number of iterations for the DE-SGLD method on different network structures over Breast Cancer data set. Figures are based on one randomly picked agent. Here, the stepsizes are chosen as $\eta=0.0008$. We use batch size $b=32$ in the experiments.}
\label{fig:8exp}
\end{figure}

\begin{figure}[th!]
\centering
    \subfigure[Fully-connected]{
    \includegraphics[width=0.3\columnwidth]{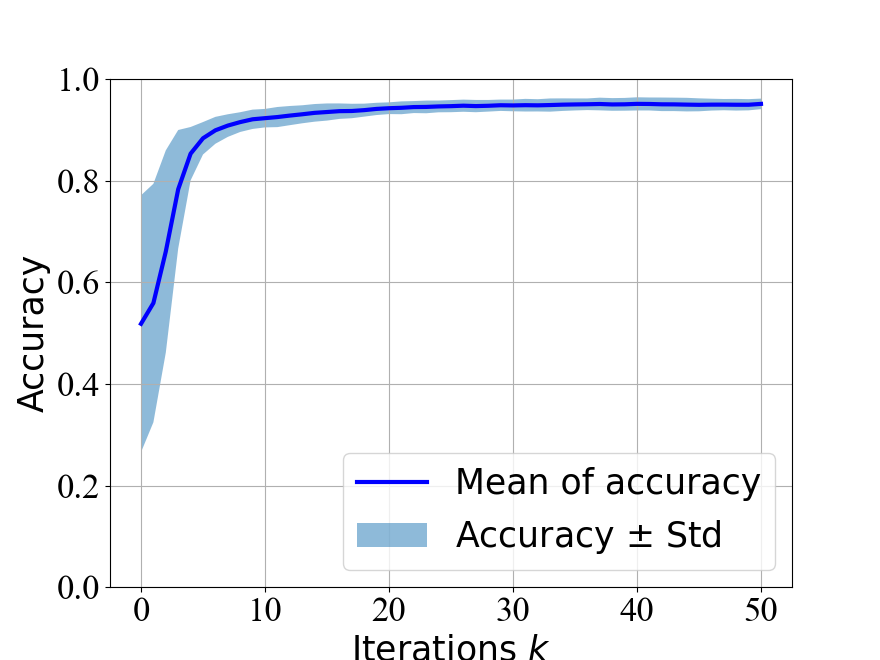}
    }
    \hfill
    \subfigure[Circular]{
    \includegraphics[width=0.3\columnwidth]{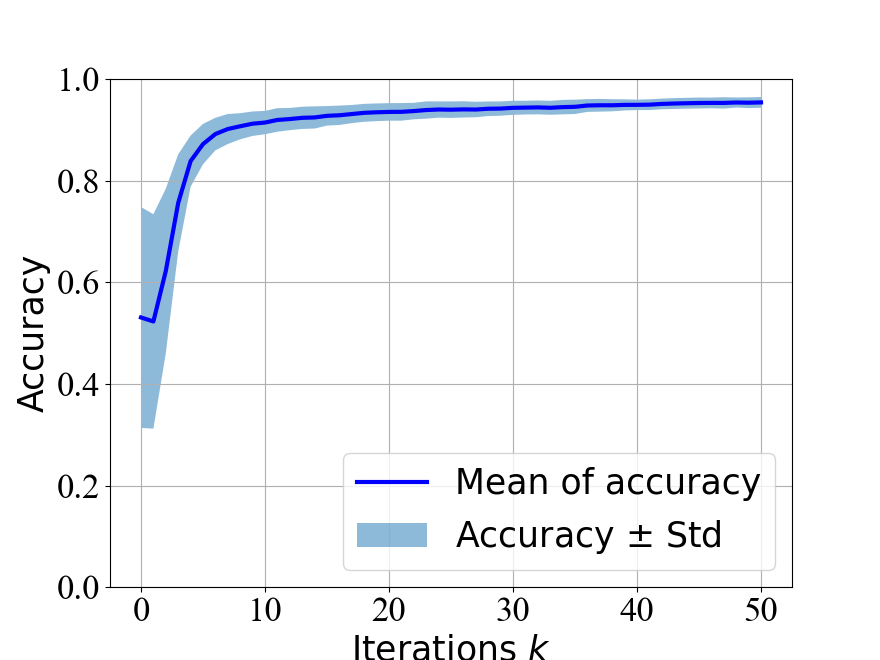}
    }
    \hfill
    \subfigure[Disconnected]{
    \includegraphics[width=0.3\columnwidth]{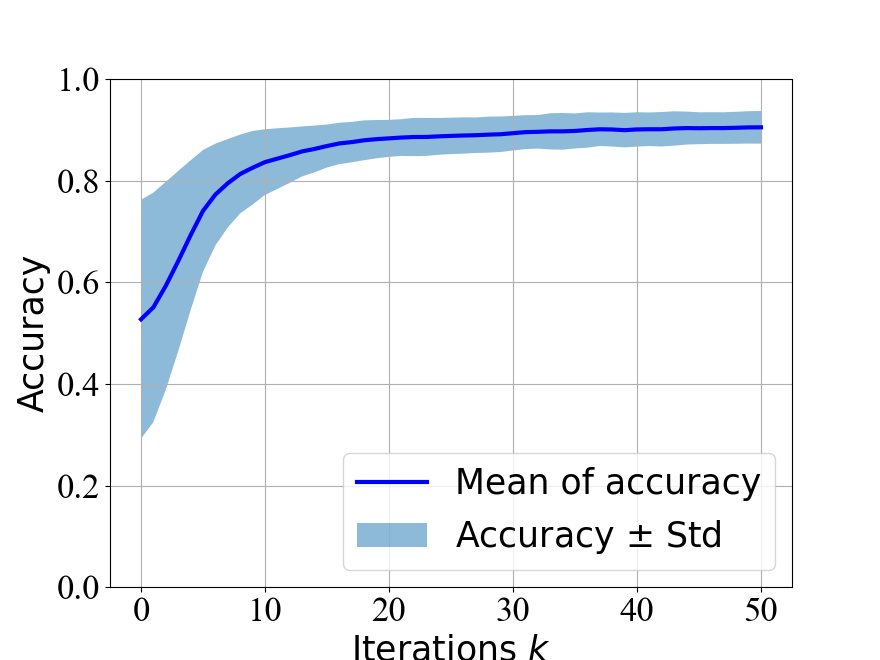}
    }
\caption{The plots show the accuracy over the data set versus number of iterations for the DE-SGHMC method on different network structures over Breast Cancer data set. Figures are based on one randomly picked agent. Here, the stepsize $\eta$ and the friction coefficient $\gamma$ are well tuned to the data set so that we take $\eta=0.05$, $\gamma=10$. We use batch size $b=32$ in the experiments.}
\label{fig:9exp}
\end{figure}

\begin{figure}[t]
\centering
    \subfigure[Fully-connected over train]{
    \includegraphics[width=0.3\columnwidth]{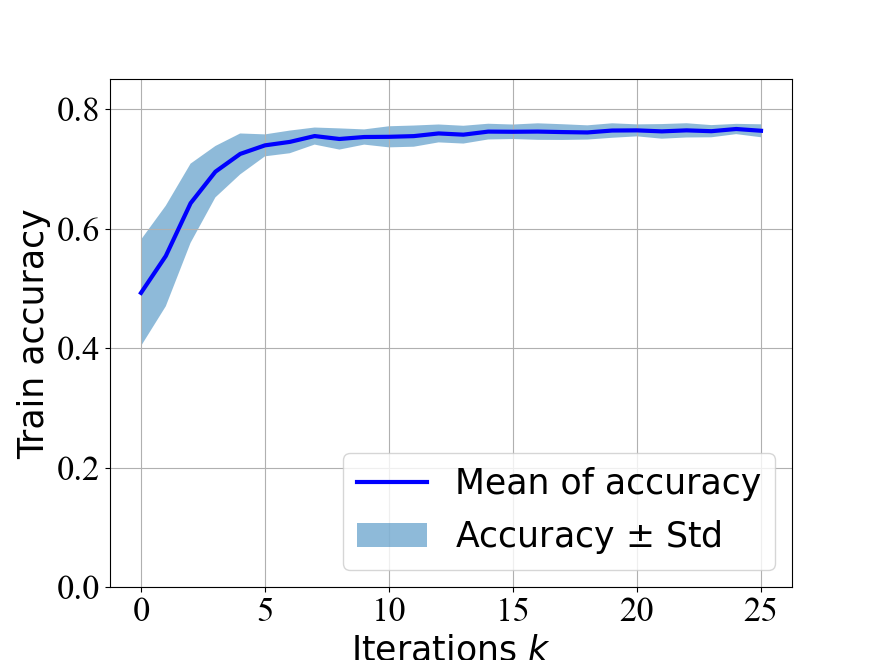}
    }
    \hfill
    \subfigure[Circular over train]{
    \includegraphics[width=0.3\columnwidth]{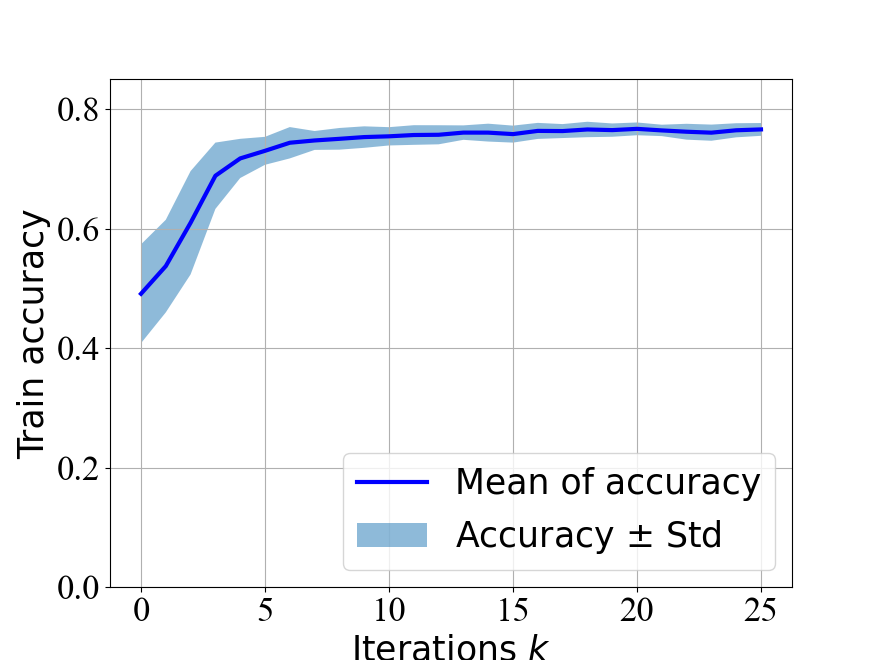}
    }
    \hfill
    \subfigure[Disconnected over train]{
    \includegraphics[width=0.3\columnwidth]{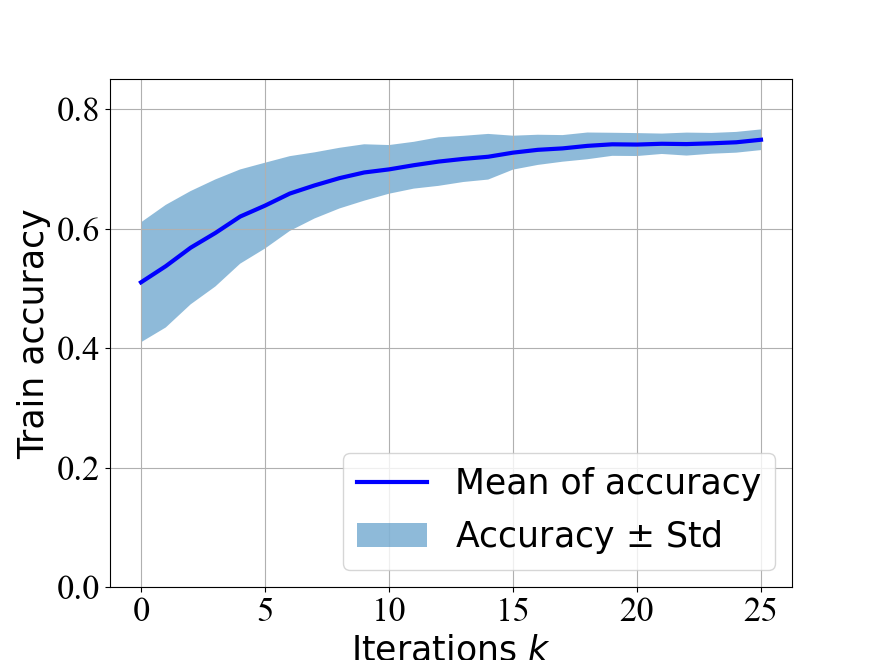}
    }
    \subfigure[Fully-connected over test]{
    \includegraphics[width=0.3\columnwidth]{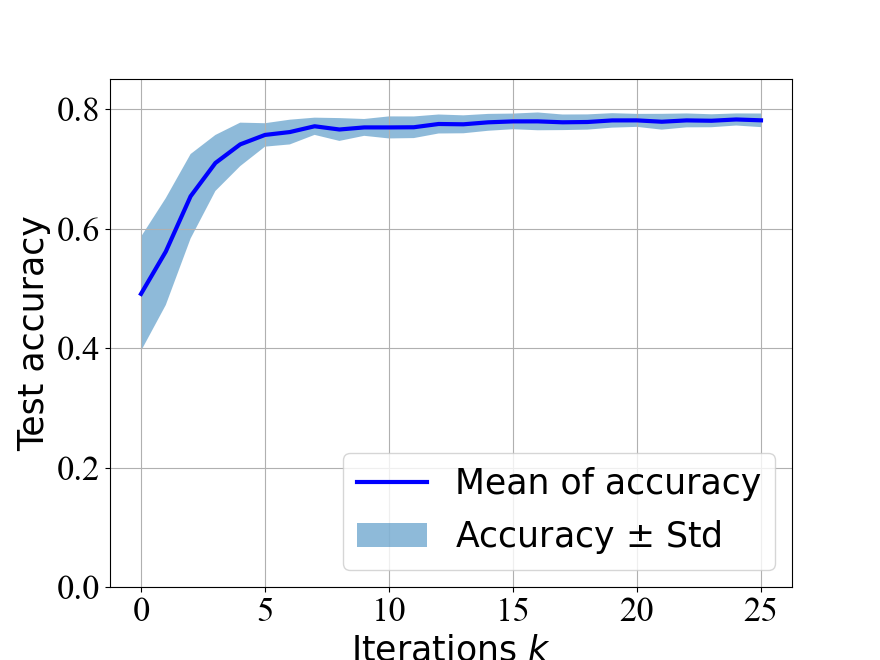}
    }
    \hfill
    \subfigure[Circular over test]{
    \includegraphics[width=0.3\columnwidth]{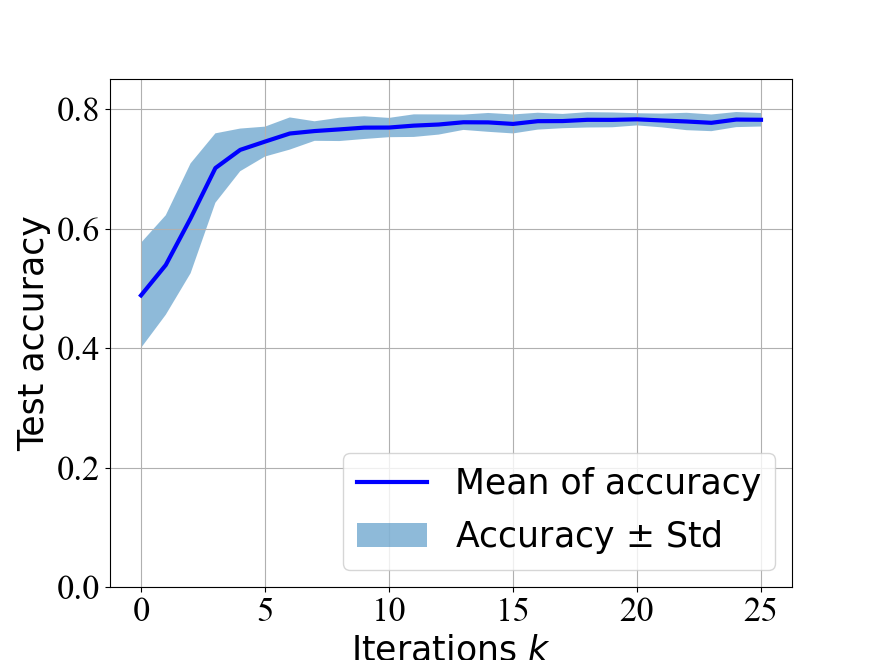}
    }
    \hfill
    \subfigure[Disconnected over test]{
    \includegraphics[width=0.3\columnwidth]{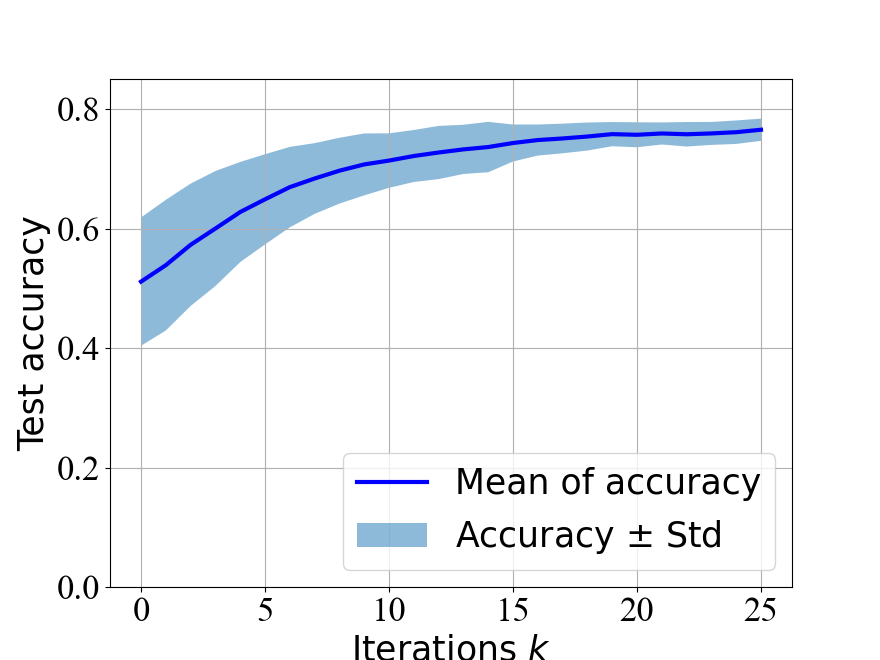}
    }
\caption{The plots show the ``distribution of the accuracy over the data set" versus number of iterations for the DE-SGLD method on different network structures over the training data and the test data of Telescope data set. Figures are based on one randomly picked agent. Here, the stepsizes are chosen as $\eta=0.008$. We use batch size $b=100$ in the experiments.}
\label{fig:10exp}
\end{figure}

\begin{figure}[th!]
\centering
    \subfigure[Fully-connected over train]{
    \includegraphics[width=0.3\columnwidth]{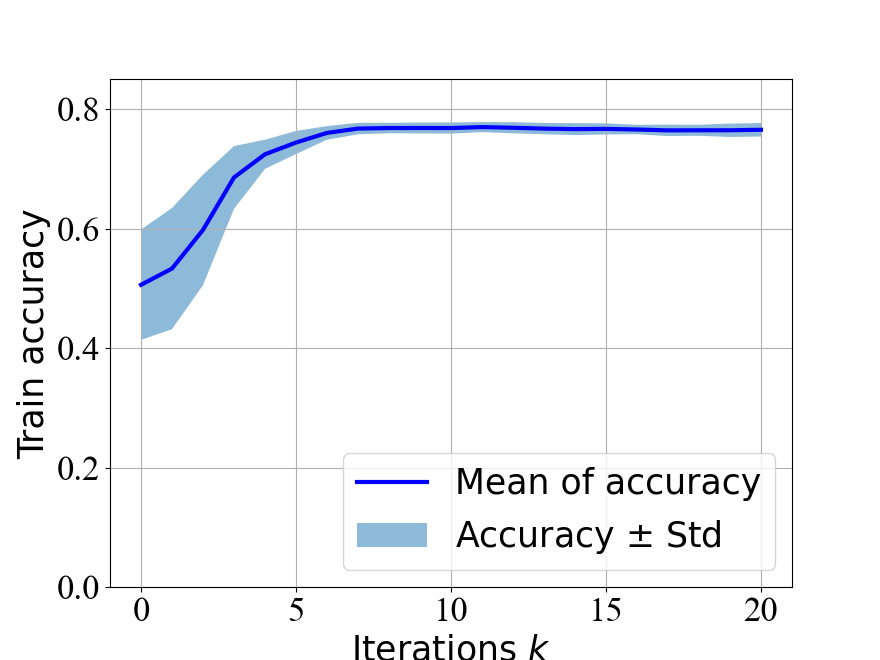}
    }
    \hfill
    \subfigure[Circular over train]{
    \includegraphics[width=0.3\columnwidth]{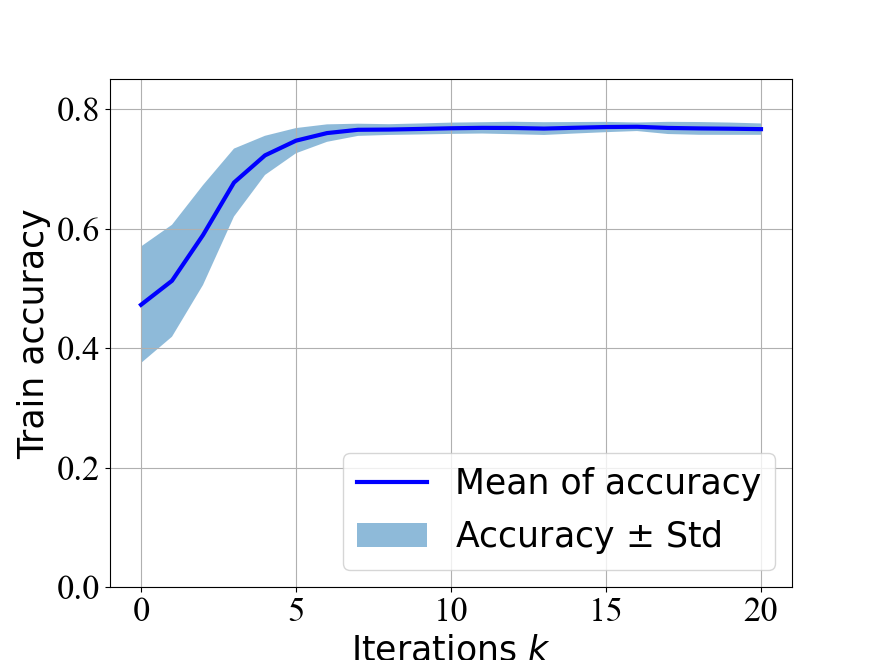}
    }
    \hfill
    \subfigure[Disconnected over train]{
    \includegraphics[width=0.3\columnwidth]{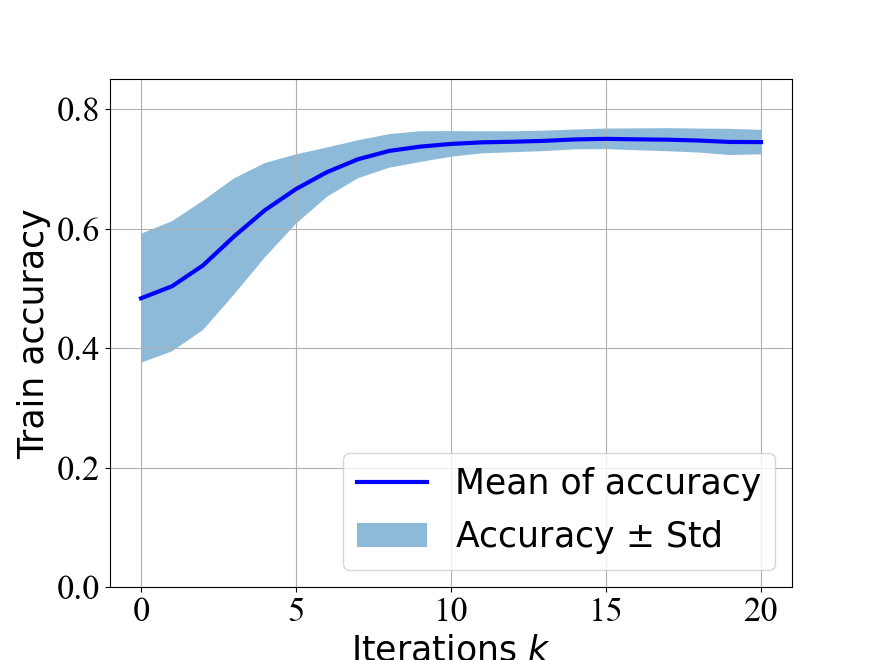}
    }
    \subfigure[Fully-connected over test]{
    \includegraphics[width=0.3\columnwidth]{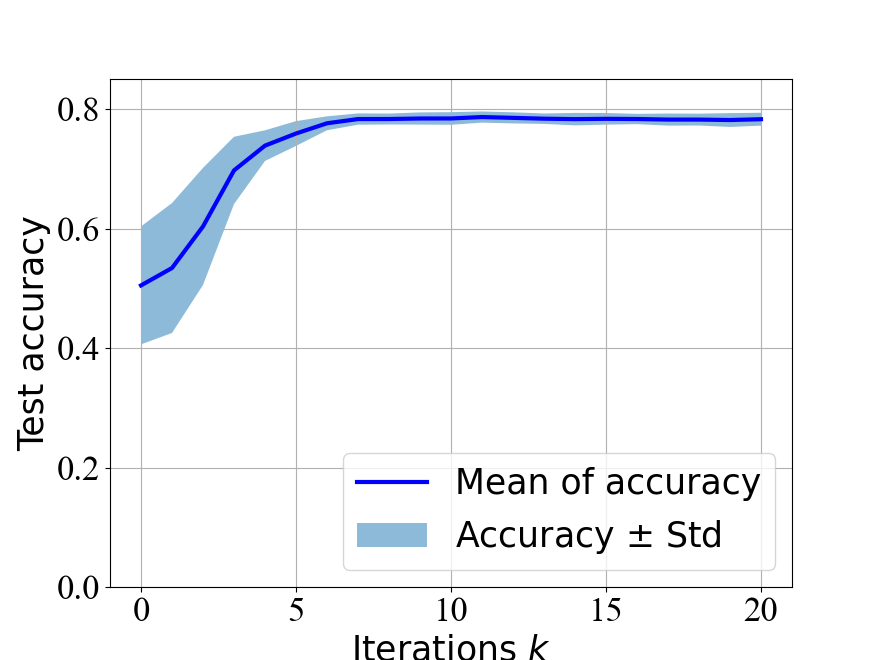}
    }
    \hfill
    \subfigure[Circular over test]{
    \includegraphics[width=0.3\columnwidth]{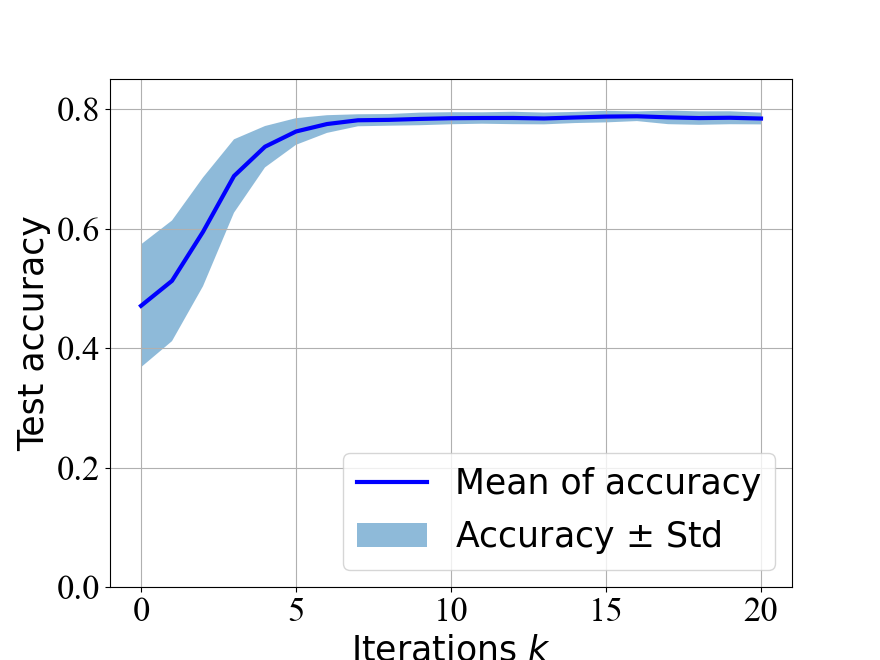}
    }
    \hfill
    \subfigure[Disconnected over test]{
    \includegraphics[width=0.3\columnwidth]{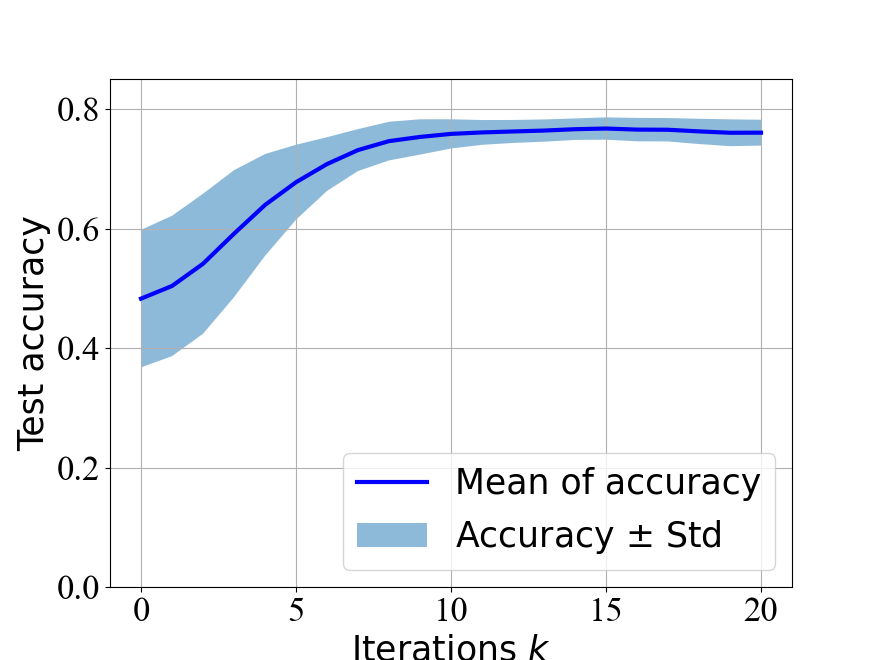}
    }
\caption{The plots show the accuracy over the data set versus number of iterations for the DE-SGHMC method on different network structures over training data and test data from the Telescope data set. Figures are based on one randomly picked agent. Here, the stepsize $\eta$ and the friction coefficient $\gamma$ are tuned to the data set where we take $\eta=0.07$, $\gamma=5$. We use batch size $b=100$ in the experiments.}
\label{fig:11exp}
\end{figure}

\subsection{Bayesian logistic regression with real data}\label{subsec-logistic-real}

In this section, we consider the Bayesian logistic regression problem on the
UCI ML Breast Cancer Wisconsin (Diagnostic) data set\protect\footnote{The corresponding data set is available online at  \url{https://archive.ics.uci.edu/ml/datasets/Breast+Cancer+Wisconsin+(Diagnostic)}.} and MAGIC Gamma Telescope data set\protect\footnote{The data set is available at \url{https://archive.ics.uci.edu/ml/datasets/magic+gamma+telescope}.}. The Breast Cancer data set contains 569 samples with dimension 31 and each sample describes characteristics of the cell nuclei present in a digitized image of a fine needle aspirate (FNA) of a breast mass. The Telescope data set contains 19,020 samples with dimension 11 and each sample describes the registration of high energy gamma particles in a ground-based atmospheric Cherenkov gamma telescope using the imaging technique.

Figure~\ref{fig:8exp} and Figure~\ref{fig:9exp} illustrate the results of using DE-SGLD and DE-SGHMC methods applied to the classification problem over the Breast Cancer data set. For Breast Cancer data set, we separate the data set into $N=6$ parts with approximately equal sizes and each agent can access to only one part of the whole data set. 
Similar to the previous section, we use the distribution of accuracy over the whole data set as the performance measure and report the performance of a randomly picked agent. The performance in the disconnected network setting is worse compared to the connected setting. We observe that the convergence of DE-SGHMC method displayed in Figure~\ref{fig:9exp} compared to DE-SGLD is slightly faster with a smaller standard deviation of accuracy in general.

Figure~\ref{fig:10exp} and Figure~\ref{fig:11exp} illustrate the results of using DE-SGLD and DE-SGHMC methods for classification over the Telescope data set. For Telescope data set, we separate the data set into training data and test data, where test data has $10\%$ data points. Then we separate the training data into 6 parts same as before. 
We report the accuracy over both training data and test data in the figures. We get similar results that illustrate that both methods perform better for fully-connected and circular networks compared to the disconnected network setting.
These results illustrate our theoretical results and show the performance of our methods for decentralized Bayesian logistic regression problems.

\section{Conclusion}
In this paper, we studied DE-SGLD and DE-SGHMC methods which allow scalable Bayesian inference for decentralized learning settings. For both methods, we show that 
the distribution of the iterate $x_i^{(k)}$ of node $i$ converges linearly (in $k$) to a neighborhood of the target distribution in the 2-Wasserstein metric when the target density $\pi(x) \propto e^{-f(x)}$ is strongly log-concave (i.e. $f$ is strongly convex) and $f$ is smooth. Our results are non-asymptotic and provide performance bounds for any finite $k$. We also illustrated the efficiency of our methods on the Bayesian linear regression and Bayesian logistic regression problems.  

\section*{Acknowledgements}
The authors are indebted to Umut \c{S}im\c{s}ekli for fruitful discussions and for his help with the experiments. 
The authors are also grateful to the Associate Editor and three anonymous referees for helpful suggestions
and comments.
Mert G\"{u}rb\"{u}zbalaban and Yuanhan Hu's research are supported in part by the grants Office of Naval Research Award Number
N00014-21-1-2244, National Science Foundation (NSF)
CCF-1814888, NSF DMS-2053485, NSF DMS-1723085. Xuefeng Gao acknowledges support from Hong Kong RGC Grants 14201117, 14201520 and 14201421. 
Lingjiong Zhu is grateful to the partial support from a Simons Foundation Collaboration Grant
and the grant NSF DMS-2053454 from the National Science Foundation.


\appendix

\section{Proofs of Technical Results in Section~\ref{proof:overdamped}}

\subsection{Proof of Lemma~\ref{lem:0}} 
\rev{In this proof, we aim to provide an $L^2$ bound on the gradients $\nabla F\left(x^{(k)}\right)$ that is uniform in $k$, where $F$ is defined in \eqref{eq:F}.}
Let us define
\begin{equation}\label{def-F-W-eta}
F_{\mathcal{W},\eta}(x):=\frac{1}{2\eta}x^{T}(I-\mathcal{W})x+F(x).
\end{equation}
Then $F_{\mathcal{W},\eta}$ is $\mu$-strongly convex and 
$L_{\eta}$-smooth with $L_{\eta}=\frac{1-\lambda_{N}^{W}}{\eta}+L$,
and we can re-write the DE-SGLD iterates as
\begin{equation} \label{eq:x-k-dyn-w}
x^{(k+1)}=x^{(k)}-\eta \nabla F_{\mathcal{W},\eta}\left(x^{(k)}\right)
-\eta\xi^{(k+1)}
+\sqrt{2\eta}w^{(k+1)}, \quad \mbox{with} \quad \mathcal{W} = W \otimes I_{d}.
\end{equation}
 Define $x_{\eta}^{\ast}$ as the minimizer of $F_{\mathcal{W},\eta}$.
Since $\nabla F(x)$ is $L$-Lipschitz, we have
\begin{align}
\mathbb{E}\left\Vert\nabla F\left(x^{(k)}\right)\right\Vert^{2}
&\leq
2\mathbb{E}\left\Vert\nabla F\left(x^{(k)}\right)- \nabla F\left(x_{\eta}^{\ast}\right)\right\Vert^{2}
+2\left\Vert \nabla F\left(x_{\eta}^{\ast}\right)\right\Vert^{2}
\nonumber
\\
&\leq 2L^{2}\mathbb{E}\left[\left\Vert x^{(k)}-x_{\eta}^{\ast}\right\Vert^{2}\right]
+2\left\Vert \nabla F\left(x_{\eta}^{\ast}\right)\right\Vert^{2}
\nonumber
\\
&\leq 2L^{2}\mathbb{E}\left[\left\Vert x^{(k)}-x_{\eta}^{\ast}\right\Vert^{2}\right]
+4\left\Vert \nabla F\left(x_{\eta}^{\ast}\right)-\nabla F\left(x^{\ast}\right)\right\Vert^{2}
+4\left\Vert \nabla F\left(x^{\ast}\right)\right\Vert^{2}
\nonumber
\\
&\leq 2L^{2}\mathbb{E}\left[\left\Vert x^{(k)}-x_{\eta}^{\ast}\right\Vert^{2}\right]
+4L^{2}\left\Vert x_{\eta}^{\ast}-x^{\ast}\right\Vert^{2}
+4\left\Vert \nabla F\left(x^{\ast}\right)\right\Vert^{2},\label{eq:g-F-decomp}
\end{align}
where we recall from \eqref{eq:x-ast-ND} that
$x^{\ast}=\left[x_{\ast}^{T},x_{\ast}^{T},\ldots,x_{\ast}^{T}\right]^{T}$,
where $x_{\ast}$ is the minimizer of $f(x)$.
Therefore, in order to derive Lemma~\ref{lem:0}, 
we need to have a control on $\| x_{\eta}^{\ast} - x^{\ast}\|$. 
This is provided in the following lemma, which follows from Corollary 9 in \citet{Yuan16}.

\begin{lemma}\label{lem:C1:gamma}
If $\eta \leq \min(\frac{1+\lambda_N^W}{L},\frac{1}{L+\mu})$, then
\begin{equation*}
\| x_{\eta}^{\ast} - x^{\ast}\| \leq C_1 \frac{\eta\sqrt{N}}{1-\bar{\gamma}}, \quad \mbox{with} 
\quad \bar{\gamma} := \max\left\{\left|\lambda_2^W\right|, \left|\lambda_N^W\right|\right\},
\end{equation*} 
where $C_1$ is defined in \eqref{def-C1}.
\end{lemma}

\textbf{Proof of Lemma~\ref{lem:C1:gamma}}
The proof of Lemma~\ref{lem:C1:gamma} will be provided in Appendix~\ref{sec:additional}.
\hfill $\Box$ 

\rev{Next, to continue with the proof of Lemma~\ref{lem:0}, we control the term $\mathbb{E}\left[\left\Vert x^{(k)}-x_{\eta}^{\ast}\right\Vert^{2}\right]$ in \eqref{eq:g-F-decomp} by deriving a recursion.}  
From \eqref{eq:x-k-dyn-w}, we get  
\begin{equation*}
x^{(k+1)}-x_{\eta}^{\ast}=x^{(k)}-x_{\eta}^{\ast}-\eta \nabla F_{\mathcal{W},\eta}\left(x^{(k)}\right)
-\eta\xi^{(k+1)}+\sqrt{2\eta}w^{(k+1)}, \quad \mbox{with} \quad \mathcal{W} = W \otimes I_{d}.
\end{equation*}
 Since $F_{\mathcal{W},\eta}$ is $L_{\eta}$-smooth and $\mu$-strongly convex, we have
 \begin{eqnarray} L_\eta \langle 
  \nabla F_{\mathcal{W},\eta}\left(z\right)- \nabla F_{\mathcal{W},\eta}\left(y\right), z - y \rangle 
 &\geq&
  \|\nabla F_{\mathcal{W},\eta}\left(z\right) - \nabla F_{\mathcal{W},\eta}\left(y\right)\|^2 \quad \forall z,y\in\mathbb{R}^d, \label{ineq-str-cvx-4}\\
 \langle z-y,\nabla F_{\mathcal{W},\eta}\left(z\right) - 
 \nabla F_{\mathcal{W},\eta}\left(y\right)
  \rangle 
  &\geq&
  \mu \|z-y\|^2 \quad \forall z,y\in\mathbb{R}^d.
   \label{ineq-str-cvx-5}
  \end{eqnarray}
 If we choose $z=x^{(k)}$ and $y=x_\eta^*$ and use the fact that $\nabla F_{\mathcal{W},\eta}\left(x_\eta^*\right)=0$, we obtain
\begin{align*}
\mathbb{E}\left[\left\Vert x^{(k+1)}-x_{\eta}^{\ast}\right\Vert^{2}\right]
&=\mathbb{E}\left[\left\Vert x^{(k)}-x_{\eta}^{\ast}\right\Vert^{2}\right]
-2\eta\mathbb{E}\left\langle x^{(k)}-x_{\eta}^{\ast},\nabla F_{\mathcal{W},\eta}\left(x^{(k)}\right)\right\rangle
\\
&\qquad\qquad\qquad
+\eta^{2}\mathbb{E}\left[\left\Vert\nabla F_{\mathcal{W},\eta}\left(x^{(k)}\right)\right\Vert^{2}\right]
+\eta^{2}\mathbb{E}\left\Vert\xi^{(k+1)}\right\Vert^{2}
+2\eta dN
\\
&\leq
\mathbb{E}\left[\left\Vert x^{(k)}-x_{\eta}^{\ast}\right\Vert^{2}\right]
-2\eta\left(1-\frac{\eta L_{\eta}}{2}\right)\mathbb{E}\left\langle x^{(k)}-x_{\eta}^{\ast},\nabla F_{\mathcal{W},\eta}\left(x^{(k)}\right)\right\rangle
\\
&\qquad\qquad\qquad\qquad+\eta^{2}\sigma^{2}N+2\eta dN
\\
&\leq
\left(1-2\mu\eta\left(1-\frac{\eta L_{\eta}}{2}\right)\right)\mathbb{E}\left[\left\Vert x^{(k)}-x_{\eta}^{\ast}\right\Vert^{2}\right]
+\eta^{2}\sigma^{2}N+2\eta dN
\\
&=\left(1-\mu\eta(1+\lambda_{N}^{W}-\eta L)\right)\mathbb{E}\left[\left\Vert x^{(k)}-x_{\eta}^{\ast}\right\Vert^{2}\right]
+\eta^{2}\sigma^{2}N+2\eta dN,
\end{align*}
where we used $\frac{\eta L_{\eta}}{2}<1$ and $\mu\eta(1+\lambda_{N}^{W}-\eta L)\in(0,1)$ which follows directly from our assumptions on the stepsize. Therefore, 
\begin{equation*}
\mathbb{E}\left[\left\Vert x^{(k)}-x_{\eta}^{\ast}\right\Vert^{2}\right]\leq
\left(1-\mu\eta(1+\lambda_{N}^{W}-\eta L)\right)^{k}\left\Vert x^{(0)}-x_{\eta}^{\ast}\right\Vert^{2}
+\frac{\eta\sigma^{2}N+2dN}{\mu(1+\lambda_{N}^{W}-\eta L)}.
\end{equation*}
Now we are ready to bound $\mathbb{E}\left\Vert\nabla F\left(x^{(k)}\right)\right\Vert^{2}$. We can compute from \eqref{eq:g-F-decomp} that 
\begin{align*}
\mathbb{E}\left\Vert\nabla F\left(x^{(k)}\right)\right\Vert^{2}
&\leq 2L^{2}\mathbb{E}\left[\left\Vert x^{(k)}-x_{\eta}^{\ast}\right\Vert^{2}\right]
+4L^{2}\left\Vert x_{\eta}^{\ast}-x^{\ast}\right\Vert^{2}
+4\left\Vert \nabla F\left(x^{\ast}\right)\right\Vert^{2}
\\
&\leq
2L^{2}\left(1-\mu\eta(1+\lambda_{N}^{W}-\eta L)\right)^{k}\left\Vert x^{(0)}-x_{\eta}^{\ast}\right\Vert^{2}
+\frac{2L^{2}(\eta\sigma^{2}N+2dN)}{\mu(1+\lambda_{N}^{W}-\eta L)}
\\
&\qquad\qquad
+4L^{2}\left\Vert x_{\eta}^{\ast}-x^{\ast}\right\Vert^{2}
+4\left\Vert \nabla F\left(x^{\ast}\right)\right\Vert^{2}
\\
&\leq
4L^{2}\left(1-\mu\eta(1+\lambda_{N}^{W}-\eta L)\right)^{k}\left\Vert x^{(0)}-x^{\ast}\right\Vert^{2}
\\
&\qquad\qquad\qquad
+4L^{2}\left(1-\mu\eta(1+\lambda_{N}^{W}-\eta L)\right)^{k}\left\Vert x^{\ast}-x_{\eta}^{\ast}\right\Vert^{2}
\\
&\qquad\qquad\qquad
+\frac{2L^{2}(\eta\sigma^{2}N+2dN)}{\mu(1+\lambda_{N}^{W}-\eta L)}+4L^{2}\left\Vert x_{\eta}^{\ast}-x^{\ast}\right\Vert^{2}
+4\left\Vert \nabla F\left(x^{\ast}\right)\right\Vert^{2}
\\
&\leq
4L^{2}\left(1-\mu\eta(1+\lambda_{N}^{W}-\eta L)\right)^{k}\left\Vert x^{(0)}-x^{\ast}\right\Vert^{2}
+8L^{2}\left\Vert x^{\ast}-x_{\eta}^{\ast}\right\Vert^{2}
\\
&\qquad\qquad\qquad
+\frac{2L^{2}(\eta\sigma^{2}N+2dN)}{\mu(1+\lambda_{N}^{W}-\eta L)}
+4\left\Vert \nabla F\left(x^{\ast}\right)\right\Vert^{2},
\end{align*}
where we recall from \eqref{eq:x-ast-ND} that $x^{\ast}=\left[x_{\ast}^{T},x_{\ast}^{T},\ldots,x_{\ast}^{T}\right]^{T}$ where $x_{\ast}$ is the minimizer of $f(x)$.
Finally, we apply Lemma~\ref{lem:C1:gamma} to
complete the proof \rev{of Lemma~\ref{lem:0}}.
\hfill $\Box$ 




\subsection{Proof of Lemma~\ref{lem:1}}
\rev{In this proof, we aim to provide uniform $L_2$ bounds between the iterates $x_i^{(k)}$ and their means $\bar{x}^{(k)}$.
First,} by the definition of $x^{(k)}$, we get
\begin{equation*}
x^{(k+1)}=(W\otimes I_{d})x^{(k)}-\eta \nabla F\left(x^{(k)}\right)
-\eta\xi^{(k+1)}+\sqrt{2\eta}w^{(k+1)}.
\end{equation*}
It follows that
\begin{align}
x^{(k)}&=(W^{k}\otimes I_{d})x^{(0)}-\eta\sum_{s=0}^{k-1}\left(W^{k-1-s}\otimes I_{d}\right)\nabla F\left(x^{(s)}\right)
\nonumber
\\
&\qquad\qquad
-\eta\sum_{s=0}^{k-1}\left(W^{k-1-s}\otimes I_{d}\right)\xi^{(s+1)}
+\sqrt{2\eta}\sum_{s=0}^{k-1}\left(W^{k-1-s}\otimes I_{d}\right)w^{(s+1)}.\label{x:k:expression}
\end{align}
Let us define $\mathbf{\bar{x}}^{(k)}:={[\bar{x}^{(k)}{}^{T},\cdots,\bar{x}^{(k)}{}^{T}]}^{T}\in\mathbb{R}^{Nd}$. 
Notice that
\begin{equation*}
\mathbf{\bar{x}}^{(k)}=\frac{1}{N}\left(\left(1_{N}1_{N}^{T}\right)\otimes I_{d}\right)x^{(k)}\,,
\end{equation*}
where $1_N \in \mathbb{R}^N$ is a vector of ones; i.e. it is a column vector with all entries equal to one and the superscript $^T$ denotes the vector transpose.
Therefore, we get
\begin{equation*}
\sum_{i=1}^{N}\left\Vert x_{i}^{(k)}-\bar{x}^{(k)}\right\Vert^{2}
=\left\Vert x^{(k)}-\mathbf{\bar{x}}^{(k)}\right\Vert^{2}
=\left\Vert x^{(k)}-\frac{1}{N}\left(\left(1_{N}1_{N}^{T}\right)\otimes I_d\right)x^{(k)}\right\Vert^{2}.
\end{equation*}
Note that it follows from \eqref{x:k:expression} that
\begin{align*}
&x^{(k)}-\frac{1}{N}\left(\left(1_{N}1_{N}^{T}\right)\otimes I_{d}\right)x^{(k)}
\\
&=(W^{k}\otimes I_{d})x^{(0)}-\frac{1}{N}\left(\left(1_{N}1_{N}^{T}W^{k}\right)\otimes I_{d}\right)x^{(0)}
\\
&-\eta\sum_{s=0}^{k-1}\left(W^{k-1-s}\otimes I_{d}\right)\nabla F\left(x^{(s)}\right)
+\eta\sum_{s=0}^{k-1}\frac{1}{N}\left(\left(1_{N}1_{N}^{T}W^{k-1-s}\right)\otimes I_{d}\right)\nabla F\left(x^{(s)}\right)
\\
&\quad
-\eta\sum_{s=0}^{k-1}\left(W^{k-1-s}\otimes I_{d}\right)\xi^{(s+1)}
+\eta\sum_{s=0}^{k-1}\frac{1}{N}\left(\left(1_{N}1_{N}^{T}W^{k-1-s}\right)\otimes I_{d}\right)\xi^{(s+1)}
\\
&\quad
+\sqrt{2\eta}\sum_{s=0}^{k-1}\left(W^{k-1-s}\otimes I_{d}\right)w^{(s+1)}
-\sqrt{2\eta}\sum_{s=0}^{k-1}\frac{1}{N}\left(\left(1_{N}1_{N}^{T}W^{k-1-s}\right)\otimes I_{d}\right)w^{(s+1)}.
\end{align*}
By the Cauchy-Schwarz inequality, we have
\begin{align*}
&\left\Vert x^{(k)}-\frac{1}{N}\left(\left(1_{N}1_{N}^{T}\right)\otimes I_{d}\right)x^{(k)}\right\Vert^{2}
\\
&\leq 
4\left\Vert(W^{k}\otimes I_{d})x^{(0)}-\frac{1}{N}\left(\left(1_{N}1_{N}^{T}W^{k}\right)\otimes I_{d}\right)x^{(0)}\right\Vert^{2}
\\
&\quad
+4\left\Vert-\eta\sum_{s=0}^{k-1}\left(W^{k-1-s}\otimes I_{d}\right)\nabla F\left(x^{(s)}\right)
+\eta\sum_{s=0}^{k-1}\frac{1}{N}\left(\left(1_{N}1_{N}^{T}W^{k-1-s}\right)\otimes I_{d}\right)\nabla F\left(x^{(s)}\right)\right\Vert^{2}
\\
&\quad
+4\left\Vert\eta\sum_{s=0}^{k-1}\left(W^{k-1-s}\otimes I_{d}\right)\xi^{(s+1)}
-\eta\sum_{s=0}^{k-1}\frac{1}{N}\left(\left(1_{N}1_{N}^{T}W^{k-1-s}\right)\otimes I_{d}\right)\xi^{(s+1)}\right\Vert^{2}
\\
&\quad
+4\left\Vert\sqrt{2\eta}\sum_{s=0}^{k-1}\left(W^{k-1-s}\otimes I_{d}\right)w^{(s+1)}
-\sqrt{2\eta}\sum_{s=0}^{k-1}\frac{1}{N}\left(\left(1_{N}1_{N}^{T}W^{k-1-s}\right)\otimes I_{d}\right)w^{(s+1)}\right\Vert^{2}
\\
&=4\left\Vert(W^{k}\otimes I_{d})x^{(0)}-\frac{1}{N}\left(\left(1_{N}1_{N}^{T}\right)\otimes I_{d}\right)x^{(0)}\right\Vert^{2}
\\
&\quad
+4\left\Vert-\eta\sum_{s=0}^{k-1}\left(W^{k-1-s}\otimes I_{d}\right)\nabla F\left(x^{(s)}\right)
+\eta\sum_{s=0}^{k-1}\frac{1}{N}\left(\left(1_{N}1_{N}^{T}\right)\otimes I_{d}\right)\nabla F\left(x^{(s)}\right)\right\Vert^{2}
\\
&\quad
+4\left\Vert\eta\sum_{s=0}^{k-1}\left(W^{k-1-s}\otimes I_{d}\right)\xi^{(s+1)}
-\eta\sum_{s=0}^{k-1}\frac{1}{N}\left(\left(1_{N}1_{N}^{T}\right)\otimes I_{d}\right)\xi^{(s+1)}\right\Vert^{2}
\\
&\quad
+4\left\Vert\sqrt{2\eta}\sum_{s=0}^{k-1}\left(W^{k-1-s}\otimes I_{d}\right)w^{(s+1)}
-\sqrt{2\eta}\sum_{s=0}^{k-1}\frac{1}{N}\left(\left(1_{N}1_{N}^{T}\right)\otimes I_{d}\right)w^{(s+1)}\right\Vert^{2},
\end{align*}
where we used the property that $W$ is doubly stochastic. Therefore, we get
\begin{align}
&\left\Vert x^{(k)}-\frac{1}{N}\left(\left(1_{N}1_{N}^{T}\right)\otimes I_{d}\right)x^{(k)}\right\Vert^{2}
\nonumber
\\
&\leq 4\left\Vert\left(\left(W^{k}-\frac{1}{N}1_{N}1_{N}^{T}\right)\otimes I_{d}\right)x^{(0)}\right\Vert^{2}
\nonumber
\\
&\qquad
+4\eta^{2}\left\Vert\sum_{s=0}^{k-1}\left(\left(W^{k-1-s}-\frac{1}{N}1_{N}1_{N}^{T}\right)\otimes I_{d}\right)\nabla F\left(x^{(s)}\right)\right\Vert^{2}
\nonumber \\
&\qquad\qquad+4\eta^{2}\left\Vert\sum_{s=0}^{k-1}\left(\left(W^{k-1-s}-\frac{1}{N}1_{N}1_{N}^{T}\right)\otimes I_{d}\right)\xi^{(s+1)}\right\Vert^{2}
\nonumber
\\
&\qquad\qquad\qquad\qquad
+8\eta\left\Vert\sum_{s=0}^{k-1}\left(\left(W^{k-1-s}-\frac{1}{N}1_{N}1_{N}^{T}\right)\otimes I_{d}\right)w^{(s+1)}\right\Vert^{2}. 
\label{eq:inter-estimate}
\end{align}
Note that
\begin{align}
&4\eta^{2}\left\Vert\sum_{s=0}^{k-1}\left(\left(W^{k-1-s}-\frac{1}{N}1_{N}1_{N}^{T}\right)\otimes I_{d}\right)\nabla F\left(x^{(s)}\right)\right\Vert^{2} \nonumber
\\
&\leq
4\eta^{2}\left(\sum_{s=0}^{k-1}\left\Vert\left(W^{k-1-s}-\frac{1}{N}1_{N}1_{N}^{T}\right)\otimes I_{d}\right\Vert
\cdot\left\Vert\nabla F\left(x^{(s)}\right)\right\Vert\right)^{2} \nonumber
\\
&\leq
4\eta^{2}\left(\sum_{s=0}^{k-1}\left\Vert W^{k-1-s}-\frac{1}{N}1_{N}1_{N}^{T}\right\Vert
\cdot\left\Vert\nabla F\left(x^{(s)}\right)\right\Vert\right)^{2} \nonumber
\\
&=
4\eta^{2}\left(\sum_{s=0}^{k-1}\bar{\gamma}^{k-1-s}
\cdot\left\Vert\nabla F\left(x^{(s)}\right)\right\Vert\right)^{2} \nonumber
\\
&=4\eta^{2}\left(\sum_{s=0}^{k-1}\bar{\gamma}^{k-1-s}\right)^{2}\left(\frac{\sum_{s=0}^{k-1}\bar{\gamma}^{k-1-s}
\cdot\left\Vert\nabla F\left(x^{(s)}\right)\right\Vert}{\sum_{s=0}^{k-1}\bar{\gamma}^{k-1-s}}\right)^{2} \nonumber
\\
&\leq
4\eta^{2}\left(\sum_{s=0}^{k-1}\bar{\gamma}^{k-1-s}\right)^{2}
\sum_{s=0}^{k-1}\frac{\bar{\gamma}^{k-1-s}}{\sum_{s=0}^{k-1}\bar{\gamma}^{k-1-s}}\left\Vert\nabla F\left(x^{(s)}\right)\right\Vert^{2}, \label{last-ineq-dev-from-mean}
\end{align}
where we used Jensen's inequality in the last step above,
and the fact that $W^{k-1-s}$ has eigenvalues
$(\lambda_{i}^{W})^{k-1-s}$
with $1=\lambda_{1}^{W}>\lambda_{2}^{W}\geq\cdots\geq\lambda_{N}^{W}>-1$, and hence
$\left\Vert W^{k-1-s}-\frac{1}{N}1_{N}1_{N}^{T}\right\Vert
=\max\{|\lambda_{2}^{W}|^{k-1-s},|\lambda_{N}^W|^{k-1-s}\}=\bar{\gamma}^{k-1-s}$.
Recall from Lemma~\ref{lem:0} that for every $k=0,1,2,\ldots$, 
$\mathbb{E}\left[\left\Vert\nabla F\left(x^{(k)}\right)\right\Vert^{2}\right]
\leq D^{2}$,
where $D$ is defined in \eqref{eqn:D1}.
Therefore, by \eqref{last-ineq-dev-from-mean}, we have
\begin{align*}
&4\eta^{2}\mathbb{E}\left[\left\Vert\sum_{s=0}^{k-1}\left(\left(W^{k-1-s}-\frac{1}{N}1_{N}1_{N}^{T}\right)\otimes I_{d}\right)\nabla F\left(x^{(s)}\right)\right\Vert^{2}\right]
\\
&\leq
4D^{2}\eta^{2}\left(\sum_{s=0}^{k-1}\bar{\gamma}^{k-1-s}\right)^{2}
\sum_{s=0}^{k-1}\frac{\bar{\gamma}^{k-1-s}}{\sum_{s=0}^{k-1}\bar{\gamma}^{k-1-s}}
\leq 
4D^{2}\eta^{2}\frac{1}{(1-\bar{\gamma})^{2}}.
\end{align*}
Similarly, we can show that
\begin{align*}
4\left\Vert\left(\left(W^{k}-\frac{1}{N}1_{N}1_{N}^{T}\right)\otimes I_{d}\right)x^{(0)}\right\Vert^{2}
&\leq
4\left\Vert\left(W^{k}-\frac{1}{N}1_{N}1_{N}^{T}\right)\otimes I_{d}\right\Vert^{2} \left\Vert x^{(0)}\right\Vert^{2}
\\
&\leq 4\bar{\gamma}^{2k}\left\Vert x^{(0)}\right\Vert^{2}.
\end{align*}
It follows from \eqref{eq:inter-estimate} that 
\begin{align*}
&\sum_{i=1}^{N}\mathbb{E}\left\Vert x_{i}^{(k)}-\bar{x}^{(k)}\right\Vert^{2} 
\\
&=\left\Vert x^{(k)}-\frac{1}{N}\left(\left(1_{N}1_{N}^{T}\right)\otimes I_{d}\right)x^{(k)}\right\Vert^{2}
\\
&\leq
4\bar{\gamma}^{2k}\mathbb{E}\left\Vert x^{(0)}\right\Vert^{2}
+4D^{2}\eta^{2}\frac{1}{(1-\bar{\gamma})^{2}}
+4\eta^{2}\sum_{s=0}^{k-1}\mathbb{E}\left\Vert\left(\left(W^{k-1-s}-\frac{1}{N}1_{N}1_{N}^{T}\right)\otimes I_{d}\right)\xi^{(s+1)}\right\Vert^{2}
\\
&\qquad
+8\eta\sum_{s=0}^{k-1}\mathbb{E}\left\Vert\left(\left(W^{k-1-s}-\frac{1}{N}1_{N}1_{N}^{T}\right)\otimes I_{d}\right)w^{(s+1)}\right\Vert^{2}
\\
&\leq
4\bar{\gamma}^{2k}\mathbb{E}\left\Vert x^{(0)}\right\Vert^{2}
+4D^{2}\eta^{2}\frac{1}{(1-\bar{\gamma})^{2}}
+4\eta^{2}\sum_{s=0}^{k-1}\left\Vert W^{k-1-s}-\frac{1}{N}1_{N}1_{N}^{T}\right\Vert^{2}
\mathbb{E}\left\Vert \xi^{(s+1)}\right\Vert^{2}
\\
&\qquad
+8\eta\sum_{s=0}^{k-1}\left\Vert W^{k-1-s}-\frac{1}{N}1_{N}1_{N}^{T}\right\Vert^{2}
\mathbb{E}\left\Vert w^{(s+1)}\right\Vert^{2}
\\
&\leq
4\bar{\gamma}^{2k}\mathbb{E}\left\Vert x^{(0)}\right\Vert^{2}
+4D^{2}\eta^{2}\frac{1}{(1-\bar{\gamma})^{2}}
+4\sigma^{2}N\eta^{2}\sum_{s=0}^{k-1}\bar{\gamma}^{2(k-1-s)}
+8dN\eta\sum_{s=0}^{k-1}\bar{\gamma}^{2(k-1-s)}  \\
&\leq
4\bar{\gamma}^{2k}\mathbb{E}\left\Vert x^{(0)}\right\Vert^{2}
+\frac{4D^{2}\eta^{2}}{(1-\bar{\gamma})^{2}}
+\frac{4\sigma^{2}N\eta^{2}}{(1-\bar{\gamma}^{2})}
+\frac{8dN\eta}{(1-\bar{\gamma}^{2})}.
\end{align*} 
The proof is complete.
\hfill $\Box$ 


\subsection{Proof of Lemma~\ref{lem:2}}
By Lemma~\ref{lem:1}, we can compute that
\begin{align*}
\mathbb{E}\left\Vert\mathcal{E}_{k+1}\right\Vert^{2}&=\mathbb{E}\left\Vert
\frac{1}{N}\sum_{i=1}^{N}\left(\nabla f_{i}\left(x_{i}^{(k)}\right)
-\nabla f_{i}\left(\bar{x}^{(k)}\right)\right)\right\Vert^{2}
\\
&\leq
\frac{1}{N^{2}}\sum_{i=1}^{N}
N\mathbb{E}\left\Vert
\nabla f_{i}\left(x_{i}^{(k)}\right)
-\nabla f_{i}\left(\bar{x}^{(k)}\right)\right\Vert^{2}
\\
&\leq\frac{1}{N}L^{2}\sum_{i=1}^{N}
\mathbb{E}\left\Vert
x_{i}^{(k)}
-\bar{x}^{(k)}\right\Vert^{2}
\\
&\leq
\frac{4L^{2}\bar{\gamma}^{2k}}{N}\mathbb{E}\left\Vert x^{(0)}\right\Vert^{2}
+\frac{4L^{2}D^{2}\eta^{2}}{N(1-\bar{\gamma})^{2}}
+\frac{4L^{2}\sigma^{2}\eta^{2}}{(1-\bar{\gamma}^{2})}
+\frac{8L^{2}d\eta}{(1-\bar{\gamma}^{2})}.
\end{align*} 
The proof is complete.
\hfill $\Box$ 

\subsection{Proof of Lemma~\ref{lem:3}}

\rev{In this proof, we aim to show that the mean of the iterates $\bar{x}^{(k)}$ which is defined in \eqref{eq:avg-error}
is close to $x_{k}$ in $L^{2}$ distance, where $x_{k}$ is defined in \eqref{eq:x_k} which is 
an Euler-Maruyama discretization of the continuous-time overdamped Langevin SDE in \eqref{eq:over-N}.}

\rev{First,} we can compute that
\begin{equation*}
\bar{x}^{(k+1)}-x_{k+1}=\bar{x}^{(k)}-x_{k}
-\frac{\eta}{N}\left[\nabla f\left(\bar{x}^{(k)}\right)-\nabla f(x_{k})\right]
+\eta\mathcal{E}_{k+1}
-\eta\bar{\xi}^{(k+1)},
\end{equation*}
where
$\mathcal{E}_{k+1}
=\frac{1}{N} \nabla f\left(\bar{x}^{(k)}\right)-\frac{1}{N}\sum_{i=1}^{N}\nabla f_{i}\left(x_{i}^{(k)}\right)$,
and this implies that
\begin{align}
&\left\Vert\bar{x}^{(k+1)}-x_{k+1}\right\Vert^{2}
\nonumber
\\
&=\left\Vert\bar{x}^{(k)}-x_{k}
-\frac{\eta}{N}\left[\nabla f\left(\bar{x}^{(k)}\right)-\nabla f(x_{k})\right]\right\Vert^{2}
+\eta^{2}\left\Vert\mathcal{E}_{k+1}-\bar{\xi}^{(k+1)}\right\Vert^{2}
\nonumber
\\
&\qquad\qquad
+2\left\langle\bar{x}^{(k)}-x_{k}
-\frac{\eta}{N}\left[\nabla f\left(\bar{x}^{(k)}\right)-\nabla f(x_{k})\right]
,\eta\mathcal{E}_{k+1}-\eta\bar{\xi}^{(k+1)}\right\rangle
\nonumber
\\
&=\left\Vert\bar{x}^{(k)}-x_{k}\right\Vert^{2}
+\eta^{2}\left\Vert\frac{1}{N}\left[\nabla f\left(\bar{x}^{(k)}\right)-\nabla f(x_{k})\right]\right\Vert^{2}
\nonumber
\\
&\qquad
-2\left\langle\bar{x}^{(k)}-x_{k},\eta\frac{1}{N}\left[\nabla f\left(\bar{x}^{(k)}\right)-\nabla f(x_{k})\right]\right\rangle
+\eta^{2}\left\Vert\mathcal{E}_{k+1}-\bar{\xi}^{(k+1)}\right\Vert^{2}
\nonumber
\\
&\qquad\qquad
+2\left\langle\bar{x}^{(k)}-x_{k}
-\eta\frac{1}{N}\left[\nabla f\left(\bar{x}^{(k)}\right)-\nabla f(x_{k})\right]
,\eta\mathcal{E}_{k+1}-\eta\bar{\xi}^{(k+1)}\right\rangle
\nonumber
\\
&\leq
\left\Vert\bar{x}^{(k)}-x_{k}\right\Vert^{2}
+\eta^{2}L\left\langle\bar{x}^{(k)}-x_{k},\frac{1}{N}\left[\nabla f\left(\bar{x}^{(k)}\right)-\nabla f(x_{k})\right]\right\rangle
\nonumber
\\
&\qquad
-2\left\langle\bar{x}^{(k)}-x_{k},\eta\frac{1}{N}\left[\nabla f\left(\bar{x}^{(k)}\right)-\nabla f(x_{k})\right]\right\rangle
+\eta^{2}\left\Vert\mathcal{E}_{k+1}-\bar{\xi}^{(k+1)}\right\Vert^{2}
\nonumber
\\
&\qquad\qquad
+2\left\langle\bar{x}^{(k)}-x_{k}
-\eta\frac{1}{N}\left[\nabla f\left(\bar{x}^{(k)}\right)-\nabla f(x_{k})\right]
,\eta\mathcal{E}_{k+1}-\eta\bar{\xi}^{(k+1)}\right\rangle
\nonumber
\\
&\leq
\left(1-2\eta\mu\left(1-\frac{\eta L}{2}\right)\right)\left\Vert\bar{x}^{(k)}-x_{k}\right\Vert^{2}
+\eta^{2}\left\Vert\mathcal{E}_{k+1}-\bar{\xi}^{(k+1)}\right\Vert^{2}
\nonumber
\\
&\qquad\qquad
+2\left\langle\bar{x}^{(k)}-x_{k}
-\eta\frac{1}{N}\left[\nabla f\left(\bar{x}^{(k)}\right)-\nabla f(x_{k})\right]
,\eta\mathcal{E}_{k+1}-\eta\bar{\xi}^{(k+1)}\right\rangle,\label{take:expect}
\end{align} 
where we used $L$-smoothness of $\frac{1}{N}f$ to obtain the second term after
the first inequality above 
and $\mu$-strongly convexity of $\frac{1}{N}f$ (inequalities \eqref{ineq-str-cvx-4}-\eqref{ineq-str-cvx-5} apply to the function $\frac{1}{N}f$ as well if we replace the smoothness constant $L_\eta$ by $L$) and the assumption that $\eta< 2/L$ to obtain the first term after the second inequality above. Note that according to \eqref{gradient:noise:i}, $\bar{\xi}^{(k+1)}$ has mean zero conditional 
on the natural filtration of the iterates till time $k$ and by Lemma~\ref{lem:2},
\begin{equation}\label{apply:tilde}
\mathbb{E}\left\Vert\mathcal{E}_{k+1}\right\Vert^{2}
\leq
\frac{4L^{2}\bar{\gamma}^{2k}}{N}\mathbb{E}\left\Vert x^{(0)}\right\Vert^{2}
+\frac{4L^{2}D^{2}\eta^{2}}{N(1-\bar{\gamma})^{2}}
+\frac{4L^{2}\sigma^{2}\eta^{2}}{(1-\bar{\gamma}^{2})}
+\frac{8L^{2}d\eta}{(1-\bar{\gamma}^{2})}. 
\end{equation}
Also, we recall from \eqref{bar:grad:noise} that
$\mathbb{E}\left\Vert\bar{\xi}^{(k+1)}\right\Vert^{2}
\leq\frac{\sigma^{2}}{N}$.
By taking expectations in \eqref{take:expect}
and applying \eqref{gradient:noise:i}, we get
\begin{align}
&\mathbb{E}\left\Vert\bar{x}^{(k+1)}-x_{k+1}\right\Vert^{2}
\nonumber
\\
&\leq
\left(1-2\eta\mu\left(1-\frac{\eta L}{2}\right)\right)\mathbb{E}\left\Vert\bar{x}^{(k)}-x_{k}\right\Vert^{2}
+\eta^{2}\mathbb{E}\left\Vert\mathcal{E}_{k+1}-\bar{\xi}^{(k+1)}\right\Vert^{2}
\nonumber
\\
&\qquad\qquad
+\mathbb{E}\left[2\left\langle\bar{x}^{(k)}-x_{k}
-\eta\frac{1}{N}\left[\nabla f\left(\bar{x}^{(k)}\right)-\nabla f(x_{k})\right]
,\eta\mathcal{E}_{k+1}-\eta\bar{\xi}^{(k+1)}\right\rangle\right]
\\
&=\left(1-2\eta\mu\left(1-\frac{\eta L}{2}\right)\right)\mathbb{E}\left\Vert\bar{x}^{(k)}-x_{k}\right\Vert^{2}
+\eta^{2}\mathbb{E}\left\Vert\mathcal{E}_{k+1}\right\Vert^{2}
+\eta^{2}\mathbb{E}\left\Vert\bar{\xi}^{(k+1)}\right\Vert^{2}
\nonumber
\\
&\qquad\qquad
+\mathbb{E}\left[2\left\langle\bar{x}^{(k)}-x_{k}
-\eta\frac{1}{N}\left[\nabla f\left(\bar{x}^{(k)}\right)-\nabla f(x_{k})\right]
,\eta\mathcal{E}_{k+1}\right\rangle\right]
\\
&\leq
\left(1-2\eta\mu\left(1-\frac{\eta L}{2}\right)\right)\mathbb{E}\left\Vert\bar{x}^{(k)}-x_{k}\right\Vert^{2}
+\eta^{2}\mathbb{E}\left\Vert\mathcal{E}_{k+1}\right\Vert^{2}
+\eta^{2}\frac{\sigma^{2}}{N}
\nonumber
\\
&\qquad\qquad
+2(1+\eta L)
\eta\mathbb{E}\left[\left\Vert\bar{x}^{(k)}-x_{k}\right\Vert\cdot\left\Vert\mathcal{E}_{k+1}\right\Vert\right], \label{last-ineq-to-ref-5}
\end{align}
where we used $L$-smoothness of $\frac{1}{N}f$. For any $x,y\geq 0$ and $c>0$, we have the inequality $2xy\leq cx^{2}+\frac{y^{2}}{c}$. Applying this inequality with $c= \frac{\mu(1-\frac{\eta L}{2})}{1+\eta L}$ to \eqref{last-ineq-to-ref-5}, we obtain
\begin{align*}
&\mathbb{E}\left\Vert\bar{x}^{(k+1)}-x_{k+1}\right\Vert^{2}
\\
&\leq
\left(1-2\eta\mu\left(1-\frac{\eta L}{2}\right)\right)\mathbb{E}\left\Vert\bar{x}^{(k)}-x_{k}\right\Vert^{2}
+\eta^{2}\mathbb{E}\left\Vert\mathcal{E}_{k+1}\right\Vert^{2}
+\eta^{2}\frac{\sigma^{2}}{N}
\\
&\qquad
+(1+\eta L)\eta
\left(\frac{\mu(1-\frac{\eta L}{2})}{1+\eta L}\mathbb{E}\left\Vert\bar{x}^{(k)}-x_{k}\right\Vert^{2}
+\frac{1+\eta L}{\mu(1-\frac{\eta L}{2})}\mathbb{E}\left\Vert\mathcal{E}_{k+1}\right\Vert^{2}\right)
\\
&=\left(1-\eta\mu\left(1-\frac{\eta L}{2}\right)\right)\mathbb{E}\left\Vert\bar{x}^{(k)}-x_{k}\right\Vert^{2}
+\eta\left(\eta+\frac{(1+\eta L)^{2}}{\mu(1-\frac{\eta L}{2})}\right)\mathbb{E}\left\Vert\mathcal{E}_{k+1}\right\Vert^{2}
+\eta^{2}\frac{\sigma^{2}}{N}\,,
\end{align*}
where we note that the leading term
$1-\eta\mu\left(1-\frac{\eta L}{2}\right)\in[0,1)$ by our assumption on stepsize $\eta$. By applying \eqref{apply:tilde}, we get
\begin{align*}
&\mathbb{E}\left\Vert\bar{x}^{(k+1)}-x_{k+1}\right\Vert^{2}
\\
&\leq\left(1-\eta\mu\left(1-\frac{\eta L}{2}\right)\right)
\mathbb{E}\left\Vert\bar{x}^{(k)}-x_{k}\right\Vert^{2}
\\
&\quad
+\eta\left(\eta+\frac{(1+\eta L)^{2}}{\mu(1-\frac{\eta L}{2})}\right)
\left(\frac{4L^{2}\bar{\gamma}^{2k}}{N}\mathbb{E}\left\Vert x^{(0)}\right\Vert^{2}
+\frac{4L^{2}D^{2}\eta^{2}}{N(1-\bar{\gamma})^{2}}
+\frac{4L^{2}\sigma^{2}\eta^{2}}{(1-\bar{\gamma}^{2})}
+\frac{8L^{2}d\eta}{(1-\bar{\gamma}^{2})}\right)
+\eta^{2}\frac{\sigma^{2}}{N},
\end{align*}
for every $k$. Note that $\mathbb{E}\left\Vert\bar{x}^{(0)}-x_{0}\right\Vert^{2}=0$.
By iterating the above equation, we get
\begin{align*}
&\mathbb{E}\left\Vert\bar{x}^{(k)}-x_{k}\right\Vert^{2}
\\
&\leq
\sum_{i=0}^{k-1}
\left(1-\eta\mu\left(1-\frac{\eta L}{2}\right)\right)^{i}
\\
&\qquad\qquad
\cdot\left(\eta\left(\eta+\frac{(1+\eta L)^{2}}{\mu(1-\frac{\eta L}{2})}\right)
\left(\frac{4L^{2}D^{2}\eta^{2}}{N(1-\bar{\gamma})^{2}}
+\frac{4L^{2}\sigma^{2}\eta^{2}}{(1-\bar{\gamma}^{2})}
+\frac{8L^{2}d\eta}{(1-\bar{\gamma}^{2})}
\right)+\eta^{2}\frac{\sigma^{2}}{N}\right)
\\
&\qquad
+\sum_{i=0}^{k-1}
\left(1-\eta\mu\left(1-\frac{\eta L}{2}\right)\right)^{i}
\eta\left(\eta+\frac{(1+\eta L)^{2}}{\mu(1-\frac{\eta L}{2})}\right)
\frac{4L^{2}\bar{\gamma}^{2(k-i)}}{N}\mathbb{E}\left\Vert x^{(0)}\right\Vert^{2}
\\
&=\frac{1-
\left(1-\eta\mu\left(1-\frac{\eta L}{2}\right)\right)^{k}}
{1-\left(1-\eta\mu\left(1-\frac{\eta L}{2}\right)\right)}
\\
&\qquad\qquad
\cdot\left(\eta\left(\eta+\frac{(1+\eta L)^{2}}{\mu(1-\frac{\eta L}{2})}\right)
\left(\frac{4L^{2}D^{2}\eta^{2}}{N(1-\bar{\gamma})^{2}}
+\frac{4L^{2}\sigma^{2}\eta^{2}}{(1-\bar{\gamma}^{2})}
+\frac{8L^{2}d\eta}{(1-\bar{\gamma}^{2})}\right)
+\eta^{2}\frac{\sigma^{2}}{N}\right)
\\
&\qquad
+\frac{\bar{\gamma}^{2k}-
\left(1-\eta\mu\left(1-\frac{\eta L}{2}\right)\right)^{k}}
{1-\left(1-\eta\mu\left(1-\frac{\eta L}{2}\right)\right)(\bar{\gamma})^{-2}}
\frac{4L^{2}}{N}\mathbb{E}\left\Vert x^{(0)}\right\Vert^{2}.
\end{align*}
Hence\footnote{We recall that the last term is proportional to the ratio $h(x,y) = \frac{x^k - y^k}{x-y}$ with $x=\bar{\gamma}^2$ and $y =\left(1-\eta\mu\left(1-\frac{\eta L}{2}\right)\right)$ and according to our notation (see Section~\ref{sec:background}), we interpret this ratio as $ky^{k-1}$ in the special case when $x=y$.}, for every $k$.
\begin{align*}
\mathbb{E}\left\Vert\bar{x}^{(k)}-x_{k}\right\Vert^{2}
&\leq\frac{\eta\left(\eta+\frac{(1+\eta L)^{2}}{\mu(1-\frac{\eta L}{2})}\right)
\left(\frac{4L^{2}D^{2}\eta^{2}}{N(1-\bar{\gamma})^{2}}
+\frac{4L^{2}\sigma^{2}\eta^{2}}{(1-\bar{\gamma}^{2})}
+\frac{8L^{2}d\eta}{(1-\bar{\gamma}^{2})}
\right)+\eta^{2}\frac{\sigma^{2}}{N}}{
1-\left(1-\eta\mu\left(1-\frac{\eta L}{2}\right)\right)}
\\
&\qquad\qquad
+\frac{\bar{\gamma}^{2k}-
\left(1-\eta\mu\left(1-\frac{\eta L}{2}\right)\right)^{k}}
{1-\left(1-\eta\mu\left(1-\frac{\eta L}{2}\right)\right)(\bar{\gamma})^{-2}}
\frac{4L^{2}}{N}\mathbb{E}\left\Vert x^{(0)}\right\Vert^{2}
\\
&=\frac{\eta\left(\eta+\frac{(1+\eta L)^{2}}{\mu(1-\frac{\eta L}{2})}\right)
\left(\frac{4L^{2}D^{2}\eta}{N(1-\bar{\gamma})^{2}}
+\frac{4L^{2}\sigma^{2}\eta}{(1-\bar{\gamma}^{2})}
+\frac{8L^{2}d}{(1-\bar{\gamma}^{2})}
\right)+\eta\frac{\sigma^{2}}{N}}{
\mu\left(1-\frac{\eta L}{2}\right)}
\\
&\qquad\qquad
+\frac{\bar{\gamma}^{2k}-
\left(1-\eta\mu\left(1-\frac{\eta L}{2}\right)\right)^{k}}
{\bar{\gamma}^{2}-1+\eta\mu\left(1-\frac{\eta L}{2}\right)}
\frac{4L^{2}\bar{\gamma}^{2}}{N}\mathbb{E}\left\Vert x^{(0)}\right\Vert^{2}.
\end{align*}
The proof is complete.
\hfill $\Box$ 


\subsection{Proof of Lemma~\ref{bound:Gibbs}}
The proof of Lemma~\ref{bound:Gibbs} will be provided in Appendix~\ref{sec:additional}.
\hfill $\Box$

\section{Proofs of Technical Results in Section~\ref{proof:underdamped}}

\subsection{Proof of Lemma~\ref{lemma-l2-underdamped}}

{In this proof, we aim to provide uniform $L^2$ bounds on the iterates $v^{(k)}, x^{(k)}$ in \eqref{eq:underdamped}--\eqref{eq:underdamped2}.} 

Based on the expression \eqref{eq:underdamped} for $v^{(k+1)}$, first we rewrite the DE-SGHMC iterates \eqref{eq:underdamped2} for $k\geq 1$ as
\begin{eqnarray} 
x^{(k+1)}&=&\mathcal{W}x^{(k)}+\eta v^{(k+1)},  \nonumber \\
&=& \mathcal{W}x^{(k)}+\eta \left( v^{(k)}-\eta\left[\gamma v^{(k)}+\nabla F\left(x^{(k)}\right)+\xi^{(k+1)}\right]+\sqrt{2\gamma\eta}w^{(k+1)}\right)\nonumber\\
&=& \mathcal{W}x^{(k)} - \eta^2 \nabla F\left(x^{(k)}\right) + \eta (1-\gamma \eta) v^{(k)} + \Delta^{(k+1)}\nonumber \\
&=& \mathcal{W}x^{(k)} - \eta^2 \nabla F\left(x^{(k)}\right) +  (1-\gamma \eta) \left(x^{(k)}-\mathcal{W}x^{(k-1)}\right) + \Delta^{(k+1)}\,, \label{eq-x-iter-under}
\end{eqnarray}
where 
 $$ \Delta^{(k+1)} := -\eta^2 \xi^{(k+1)} + \eta \sqrt{2\gamma\eta}w^{(k+1)}.$$
If we consider
\beq 
\alpha=\eta^2,
\label{def-alpha}
\eeq  
then \eqref{eq-x-iter-under} is equivalent to
\begin{eqnarray} 
x^{(k+1)} &=& \mathcal{W}x^{(k)} - \alpha \nabla F\left(x^{(k)}\right) + \beta \left(x^{(k)}-\mathcal{W}x^{(k-1)}\right) + \Delta^{(k+1)}\nonumber\\
&=&x^{(k)} - \alpha \nabla \bar{F}\left(x^{(k)}\right) + \beta \left(x^{(k)}-x^{(k-1)}\right) + \bar{\Delta}^{(k+1)}\,, \label{eq-perturbed-heavy-ball}
\end{eqnarray}
where $\beta=1-\gamma\eta$ and
\begin{equation*} 
\bar{F}(x) := F(x) + 
\frac{1}{2\alpha} x^T (I-\mathcal{W})x, \quad
  \bar{\Delta}^{(k+1)} := \Delta^{(k+1)} + \beta (I-\mathcal{W}) x^{(k-1)}.  
\end{equation*}
Let $x_\alpha^*$ be the unique minimizer of $\bar{F}(x)$. Since $\alpha>0$, the function $\bar{F}(x)$ is strongly convex with parameter $\mu$ and smooth with parameter 
\beq L_\alpha = L + \frac{1-\lambda_N^W}{\alpha}.
\label{def-Lalpha}
\eeq 
In the special case, $\bar{\Delta}^{(k+1)} = 0$, the iterations \eqref{eq-perturbed-heavy-ball} would exactly coincide with the iterations of the heavy-ball method of Polyak applied to the function $\bar{F}(x)$ with momentum parameter $\beta$. Therefore, we can view the iterations \eqref{eq-perturbed-heavy-ball} as a perturbed heavy-ball method with perturbation $\bar{\Delta}^{(k+1)}$ at iteration $k$. For the heavy-ball method, linear convergence to the optimum of $\bar{F}(x)$ is obtained if the parameters $\alpha$ and $\beta$ are properly chosen. In the rest of the proof, we will extend the proof technique of \citet{polyak-linear-conv} for the convergence of the heavy-ball method to allow perturbations $\bar{\Delta}^{(k+1)}$ and show that the second moments of the iterates remain bounded. First of all, we notice that the assumptions \eqref{ineq-step-cond}--\eqref{ineq-momentum-cond} on the choice of $\eta$ and $\beta$ can be restated in terms of conditions on $\alpha=\eta^2$ as follows:
\begin{align} 
&\alpha \in \bigg(0, \frac{1+\lambda_N^W}{2(L+\mu)}\bigg], \label{eq-cond-alpha}
\\
&0\leq \beta \leq \frac{1+\lambda_N^W - 4\alpha\mu}{4}, \label{eq-cond-beta}\\
&\beta^2 
\leq c_1\mu^3 \alpha^3 \frac{(1+\lambda_N^W)}{64}, \label{eq-cond-beta2}
\end{align}
where we see after a straightforward computation that the constants $c_1$ defined by \eqref{ineq-momentum-cond} can be rewritten in terms of the smoothness constant $L_\alpha$ as
\beq c_1 =  \frac{1}{2}\frac{\alpha \mu L_\alpha}{(1-\beta)(L_\alpha+\mu) +2 L_\alpha \beta}.
\label{eq-c1}
\eeq
In particular, the condition \eqref{eq-cond-beta} implies $\beta \in [0, \frac{1}{2})$ due to the fact that $\lambda_N^W < 1$ and $\alpha,\mu>0$. Next, we introduce
\beq\label{def-pk} 
\k{p} = \frac{\beta}{1-\beta} \left( \k{x} - \km{x}\right), 
\eeq
for $k\geq 1$. From the update rule \eqref{eq-perturbed-heavy-ball}, it follows that 
\begin{align*} 
\kp{x} + \kp{p} &=\frac{1}{1-\beta}x^{(k+1)} - \frac{\beta}{1-\beta}x^{(k)} 
\\
&=\k{x} + \k{p} - \frac{\alpha}{1-\beta} \nabla \bar{F}\left(\k{x}\right) + \frac{1}{1-\beta} \kp{\bar{\Delta}}. 
\end{align*}
This implies that
\begin{align*} 
&\left\|\kp{x} + \kp{p} - x_\alpha^*\right\|^2 
\\
&= \left\|x^{(k)} + p^{(k)} - x_\alpha^*\right\|^2 
- \frac{2\alpha}{1-\beta}\left\langle x^{(k)}-x_\alpha^*, \nabla \bar{F}\left(x^{(k)}\right)\right\rangle \\
&\qquad\qquad+ \frac{\alpha^2}{(1-\beta)^2} \left\|\nabla \bar{F}\left(x^{
(k)}\right)\right\|^2 - \frac{2\alpha\beta}{(1-\beta)^2}\left\langle x^{(k)}-x^{(k-1)}, \nabla \bar{F}\left(x^{(k)}\right)\right\rangle \\
&\qquad+ \frac{1}{(1-\beta)^2}\left\|\bar{\Delta}^{(k+1)}\right\|^2 + 2 \left\langle x^{(k)} + p^{(k)} -\frac{\alpha}{1-\beta}\nabla \bar{F}\left(x^{(k)}\right) - x_\alpha^*, \frac{1}{(1-\beta)}\bar{\Delta}^{(k+1)} \right\rangle\,,
\end{align*}
where we used the definition \eqref{def-pk} of $p^{(k)}$. Next, we bound the last two terms by applying the Cauchy-Schwarz inequality: 
\begin{align*} 
\mathbb{E}_k  \left[
 \frac{1}{(1-\beta)^2}\left\|\bar{\Delta}^{(k+1)}\right\|^2 \right]
 &\leq \mathbb{E}_k \left[ \frac{1}{(1-\beta)^2} \left(2\left\|{\Delta}^{(k+1)}\right\|^2 + 2\left\|\beta(I-\mathcal{W})x^{
 (k-1)}\right\|^2\right)\right] \\
 &\leq \frac{2}{(1-\beta)^2}\mathbb{E}_k \left\|{\Delta}^{(k+1)}\right\|^2 +
 \frac{2\beta^2}{(1-\beta)^2}\left(1-\lambda_N^W\right)^2  \left\|x^{
 (k-1)}\right\|^2 \\
 &\leq\frac{2}{(1-\beta)^2}\left(\eta^4 \sigma^2 N + \eta^3 2\gamma Nd\right) + \frac{2\beta^2}{(1-\beta)^2}\left(1-\lambda_N^W\right)^2 \left\|x^{
 (k-1)}\right\|^2\,,
\end{align*}
where $\mathbb{E}_k$ denotes the conditional expectation with respect to the natural filtration up to time $k$ (which includes the history of the iterations up to (and including) $x^{(k)}$). Similarly,
\begin{align*} 
&\mathbb{E}_k \left[ 2\left\langle x^{(k)} + p^{(k)} -\frac{\alpha}{1-\beta}\nabla \bar{F}\left(x^{(k)}\right) - x_\alpha^*, \frac{1}{(1-\beta)}\bar{\Delta}^{(k+1)} \right\rangle \right] \\
&=  2 \left\langle x^{(k)} + p^{(k)} -\frac{\alpha}{1-\beta}\nabla \bar{F}\left(x^{(k)}\right) - x_\alpha^*, \frac{1}{(1-\beta)}\beta (I-\mathcal{W})x^{(k-1)} \right\rangle   \\
& \leq c_1 \left\| x^{(k)} + p^{(k)} -\frac{\alpha}{1-\beta}\nabla \bar{F}\left(x^{(k)}\right) - x_\alpha^*\right\|^2
+ \frac{1}{c_1}\frac{\beta^2}{(1-\beta)^2}\left(1-\lambda_N^W\right)^2 \left\|x^{(k-1)}\right\|^2\,, 
\end{align*}
where we use Cauchy-Schwarz and $c_1$ is the constant given by \eqref{eq-c1}. Combining everything, 
\begin{align} 
&\mathbb{E}_k \left\|\kp{x} + \kp{p} - x_\alpha^*\right\|^2
\nonumber
\\
&\leq \left\|x^{(k)} + p^{(k)} - x_\alpha^*\right\|^2 - \frac{2\alpha}{1-\beta}\left\langle x^{(k)}-x_\alpha^*, \nabla \bar{F}\left(x^{(k)}\right)\right\rangle\nonumber \\
&\qquad+ \frac{\alpha^2}{(1-\beta)^2} \left\|\nabla \bar{F}\left(x^{
(k)}\right)\right\|^2 - \frac{2\alpha\beta}{(1-\beta)^2}\left\langle x^{(k)}-x^{(k-1)}, \nabla \bar{F}\left(x^{(k)}\right)\right\rangle
+ E^{(k+1)}\,, \label{ineq-dist-bound}
\end{align}
where 
\begin{align*}  
E^{(k+1)} &:=\frac{2}{(1-\beta)^2}(\eta^4 \sigma^2 N + \eta^3 2\gamma Nd) 
+ c_1 \left\| x^{(k)} 
 + p^{(k)} -\frac{\alpha}{1-\beta}\nabla \bar{F}\left(x^{(k)}\right) - x_\alpha^*\right\|^2\\
&\qquad\qquad\qquad+ \left(2 + \frac{1}{c_1}\right)\frac{\beta^2}{(1-\beta)^2}\left(1-\lambda_N^W\right)^2 \left\|x^{(k-1)}\right\|^2 \\
&\leq \frac{2}{(1-\beta)^2}\left(\eta^4 \sigma^2 N + \eta^3 2\gamma Nd\right) 
\\
&\qquad+ 2c_1 \left\| x^{(k)} 
 + p^{(k)} - x_\alpha^*\right\|^2 + 2c_1 \frac{\alpha^2}{(1-\beta)^2}\left\|\nabla \bar{F}\left(x^{(k)}\right)\right\|^2  \\
&\qquad\qquad+ \left(2 + \frac{1}{c_1}\right)\frac{\beta^2}{(1-\beta)^2}(1-\lambda_N^W)^2 
\left(2 \left\|x^{(k-1)} - x_\alpha^{*}\right\|^2 +  2\left\|x_\alpha^*\right\|^2\right)\,,
\end{align*}
and in the last step we used the Cauchy-Schwarz inequality, i.e.
$$ 
\left\|x^{
 (k-1)}\right\|^2  \leq 2 \left\|x^{(k-1)} - x_\alpha^{*}\right\|^2 + 2\left\|x_\alpha^*\right\|^2\,.  
$$ 
Since $\bar{F}$ is $\mu$-strongly convex and $L_\alpha$ smooth, we have also
\begin{align*}
&\frac{\mu L_\alpha}{ L_\alpha + \mu} \left\|x^{(k)}-x_\alpha^*\right\|^2 + \frac{1}{L_\alpha+\mu}\left\|\nabla \bar{F}\left(x^{(k)}\right)\right\|^2 \leq \left\langle x^{(k)}- x_\alpha^*,\nabla \bar{F}\left(x^{(k)}\right) \right\rangle\,,
\\
&\bar{F}\left(x^{(k)}\right) - \bar{F}\left(x^{(k-1)}\right) + \frac{\mu}{2}\left\|x^{(k)}- x^{(k-1)}\right\|^2 \leq \left\langle x^{(k)}-x^{(k-1)}, \nabla \bar{F}\left(x^{(k)}\right)\right\rangle\,,
\end{align*} 
(see e.g. \citet{nesterov2013introductory}). These inequalities combined with \eqref{ineq-dist-bound} implies
\begin{align*} 
&\frac{2\alpha \beta}{(1-\beta)^2}\left(\bar{F}\left(x^{(k)}\right)- \bar{F}^*\right) + 
\mathbb{E}_k \left\|\kp{x} + \kp{p} - x_\alpha^*\right\|^2  \\
&\leq \frac{2\alpha\beta}{(1-\beta)^2} \left(\bar{F}\left(x^{(k-1)}\right)- \bar{F}^*\right) 
+ \left\|\k{x} + \k{p} - x_\alpha^*\right\|^2 - \frac{2\alpha\mu L_\alpha}{(1-\beta)(L_\alpha+\mu)}\left\|x^{(k)}-x_\alpha^*\right\|^2 \\
&\qquad - \frac{\alpha\beta \mu}{(1-\beta)^2}\left\|x^{(k)}-x^{(k-1)}\right\|^2  + \frac{\alpha}{(1-\beta)}\left(\frac{\alpha}{1-\beta} - \frac{2}{L_\alpha + \mu}\right) \left\|\nabla \bar{F}\left(x^{(k)}\right)\right\|^2 + E^{(k+1)}\,.
\end{align*}
Plugging the upper bound for $E^{(k+1)}$, we obtain
\begin{align} 
&\frac{2\alpha\beta}{(1-\beta)^2}\left(\bar{F}\left(x^{(k)}\right)- \bar{F}^*\right) + 
\left\|\kp{x} + \kp{p} - x_\alpha^*\right\|^2 
\nonumber
\\
&\leq \frac{2\alpha\beta}{(1-\beta)^2}\left(\bar{F}\left(x^{(k-1)}\right)- \bar{F}^*\right) \nonumber
\\
&\qquad+ \left\|\k{x} + \k{p} - x_\alpha^*\right\|^2(1+2c_1) 
- \frac{2\alpha\mu L_\alpha}{(1-\beta)(L_\alpha+\mu)}\left\|x^{(k)}-x_\alpha^*\right\|^2 
\nonumber
\\
&\qquad - \frac{\alpha\beta \mu}{(1-\beta)^2}\left\|x^{(k)}-x^{(k-1)}\right\|^2  + \frac{\alpha}{(1-\beta)}\left(\frac{\alpha}{1-\beta} - \frac{2}{L_\alpha + \mu} + \frac{2c_1 \alpha}{1-\beta} \right) \left\|\nabla \bar{F}\left(x^{(k)}\right)\right\|^2  \nonumber
\\
&\qquad + \frac{2}{(1-\beta)^2} \left(\eta^4 \sigma^2 N + \eta^3 2\gamma Nd\right)
\nonumber
\\
&\qquad
+\frac{2\beta^2}{(1-\beta)^2}\left(1-\lambda_N^W\right)^2 \left(1+\frac{1}{2c_1}\right) \left(2 \left\|x^{(k-1)} - x_\alpha^{*}\right\|^2 + 2 \left\|x_\alpha^*\right\|^2 \right)\,. \label{ineq-pert-hb-estimate}
\end{align}
By Lemma~\ref{lem-negative-coef}, the coefficient in front of  $\left\|\nabla \bar{F}\left(x^{(k)}\right)\right\|^2$ is negative, i.e. 
 \beq k_{\alpha,\beta}:=\frac{\alpha}{(1-\beta)} \left( \frac{\alpha}{1-\beta} - \frac{2}{L_\alpha + \mu} + \frac{2c_1 \alpha}{1-\beta} \right) < 0.
 \label{def-k-ab}
 \eeq
We then move the term with $\left\|\nabla \bar{F}\left(x^{(k)}\right)\right\|^2$ to the left-hand side of \eqref{ineq-pert-hb-estimate} to obtain
\begin{align} 
&\frac{2\alpha\beta}{(1-\beta)^2}\left(\bar{F}\left(x^{(k)}\right)- \bar{F}^*\right) +
\left\|\kp{x} + \kp{p} - x_\alpha^*\right\|^2 - k_{\alpha,\beta}\left\|\nabla \bar{F}(x^{(k)})\right\|^2
\nonumber
\\
&\leq \frac{2\alpha\beta}{(1-\beta)^2}\left(\bar{F}\left(x^{(k-1)}\right)- \bar{F}^*\right) + \left\|\k{x} + \k{p} - x_\alpha^*\right\|^2(1+2c_1) 
\nonumber
\\
&\qquad
- \frac{2\alpha\mu L_\alpha}{(1-\beta)(L_\alpha+\mu)}\left\|x^{(k)}-x_\alpha^*\right\|^2 
- \frac{\alpha\beta \mu}{(1-\beta)^2}\left\|x^{(k)}-x^{(k-1)}\right\|^2 
\nonumber
\\
&\qquad\qquad+ \frac{2}{(1-\beta)^2} \left(\eta^4 \sigma^2 N + \eta^3 2\gamma Nd\right)
\nonumber
\\   
&\qquad\qquad\qquad+\frac{2\beta^2}{(1-\beta)^2}\left(1-\lambda_N^W\right)^2 \left(1+\frac{1}{2c_1}\right) \left(2 \left\|x^{(k-1)} - x_\alpha^{*}\right\|^2 + 2 \left\|x_\alpha^*\right\|^2 \right)\,. \label{ineq-hb-inexact}
\end{align} 
By standard inequalities for $\mu$-strongly convex functions from \citet[Section 2.1]{nesterov2013introductory}, also 
\beq 
2\mu \left(\bar{F}\left(x^{(k)}\right)- \bar{F}^*\right) 
&\leq& 
\left\|\nabla \bar F\left(x^{(k)}\right)\right\|^2, \label{ineq-str-cvx-1}\\
\left\|x^{(k-1)}-x_\alpha^*\right\|^2 &\leq& \frac{2}{\mu}\left[\bar{F}\left(x^{(k-1)}\right)- \bar{F}\left(x_\alpha^*\right)\right]\,, \label{ineq-str-cvx-2}
\eeq
where $\bar{F}^* := \bar{F}(x_\alpha^*)$ is the global minimum of $\bar{F}$. In particular, by multiplying both sides of the first inequality \eqref{ineq-str-cvx-1} with $-k_{\alpha,\beta}>0$ we obtain
\beq 
 -2k_{\alpha,\beta}\mu  \left(\bar{F}\left(x^{(k)}\right)- \bar{F}^*\right) 
\leq 
-k_{\alpha,\beta}\left\|\nabla \bar F\left(x^{(k)}\right)\right\|^2. \label{ineq-str-cvx-3}
\eeq
Inserting the estimates \eqref{ineq-str-cvx-2} and \eqref{ineq-str-cvx-3} into \eqref{ineq-hb-inexact}, we obtain 
\begin{align*} 
&b\left(\bar{F}\left(x^{(k)}\right)- \bar{F}^*\right) + 
\left\|\kp{x} + \kp{p} - x_\alpha^*\right\|^2  \\
&\leq a\left(\bar{F}\left(x^{(k-1)}\right)- \bar{F}^*\right) 
+ \left\|\k{x} + \k{p} - x_\alpha^*\right\|^2(1+2c_1) - \frac{2\alpha\mu L_\alpha}{(1-\beta)(L_\alpha+\mu)}\left\|x^{(k)}-x_\alpha^*\right\|^2 \\
&\qquad - \frac{\alpha\beta \mu}{(1-\beta)^2}\left\|x^{(k)}-x^{(k-1)}\right\|^2   
+ \frac{2}{(1-\beta)^2} \left(\eta^4 \sigma^2 N + \eta^3 2\gamma Nd\right) + c_2 \|x_\alpha^*\|^2 \,,
\end{align*}
where 
\begin{align} 
&a := \frac{2\alpha\beta}{(1-\beta)^2} + \frac{2c_2}{\mu} \,,\label{a:eqn}
\\
&b:= \frac{2\alpha\beta}{(1-\beta)^2} - 2 k_{\alpha,\beta}\mu
= \frac{2\alpha}{(1-\beta)} \left(\frac{\beta - \mu\alpha}{1-\beta}+ \frac{2\mu}{L_\alpha+\mu} - \frac{2c_1\alpha\mu}{1-\beta}
\right)\,,\label{b:eqn}
\end{align}
where $k_{\alpha,\beta}$ is defined by \eqref{def-k-ab}, and
\begin{equation*} 
c_2 :=  \frac{4\beta^2}{(1-\beta)^2}\left(1-\lambda_N^W\right)^2 \left(1+\frac{1}{c_1}\right).
\end{equation*}
We can also write 
$$ 
\left\|x^{(k)} + p^{(k)}-x_\alpha^*\right\|^2 = \left[z^{(k)}\right]^T M z^{(k)}\,,
$$
where 
$$
z^{(k)} = \left[x^{(k)}-x_\alpha^*, x^{(k)}-x^{(k-1)}\right]^T\,,
\qquad
\text{and} 
\qquad
M = \begin{bmatrix} 
I_d & \frac{\beta}{1-\beta} I_d \\
\frac{\beta}{1-\beta}I_d & \frac{\beta^2}{(1-\beta)^2}I_d
\end{bmatrix}\,.
$$
Therefore, we can write 
\begin{align*} 
&b\left(\bar{F}\left(x^{(k)}\right)- \bar{F}^*\right) + 
 \mathbb{E}_k\left[ \left(z^{(k+1)}\right)^T M \left(z^{(k+1)}\right)\right] \\
&\leq a\left(\bar{F}\left(x^{(k-1)}\right)- \bar{F}^*\right) +  \left(z^{(k)}\right)^T Q \left(z^{(k)}\right) 
+ \frac{2}{(1-\beta)^2} \left(\eta^4 \sigma^2 N + \eta^3 2\gamma Nd\right) 
+ c_2 \|x_\alpha^*\|^2\,, 
\end{align*}
where 
\begin{equation}
Q := \begin{bmatrix}
        \left(1+2c_1 -\frac{2\alpha\mu L_\alpha}{(1-\beta)(L+\mu)}\right)I_d & \frac{\beta}{(1-\beta)}(1+2c_1)I_d \\
        \frac{\beta}{(1-\beta)}(1+2c_1)I_d & \left(\frac{(1+2c_1)\beta^2 -\alpha\beta \mu}{(1-\beta)^2}  \right )I_d
    \end{bmatrix}\,. 
\end{equation}
By Lemma~\ref{lemma-q1} and Lemma~\ref{lemma-q2}, for a given positive scalar $s$, we have 
$s M \succeq Q$ as long as $s\geq q_2$ and $a\leq s b$ as long as $s\geq q_1$ where $q_1,q_2 \in [0,1)$ and are defined by \eqref{def-q1} and \eqref{def-q2} respectively. If we introduce $q:= \max(q_1,q_2)$ and choose $s=q$, then we obtain
\begin{align*}
&b\left(\bar{F}\left(x^{(k)}\right)- \bar{F}^*\right) + 
\mathbb{E}_k \left[ \left(z^{(k+1)}\right)^T M \left(z^{(k+1)}\right) \right] \\
&\leq q\left(b\left(\bar{F}\left(x^{(k-1)}\right)- \bar{F}^*\right) +  \left(z^{(k)}\right)^T M \left(z^{(k)}\right)\right) 
\\
&\qquad\qquad\qquad\qquad\qquad
+ \frac{2}{(1-\beta)^2} \left(\eta^4 \sigma^2 N + \eta^3 2\gamma Nd\right) 
+ c_2 \left\|x_\alpha^*\right\|^2. 
\end{align*}
Let us introduce the Lyapunov function 
\begin{align*}
V_{k+1} &:= \mathbb{E}\left[b\left(\bar{F}\left(x^{(k)}\right)- \bar{F}^*\right) +  \left(z^{(k+1)}\right)^T M \left(z^{(k+1)}\right)\right] 
\\
&= \mathbb{E}\left[b\left(\bar{F}\left(x^{(k)}\right)- \bar{F}^*\right) +  \left(z^{(k+1)}\right)^T M \left(z^{(k+1)}\right)\right]\,,
\end{align*}
for $k\geq 0$. Then, taking expectations, we get
$$ V_{k+1} \leq qV_k +  \frac{2}{(1-\beta)^2} \left(\eta^4 \sigma^2 N + \eta^3 2\gamma Nd\right) 
+ c_2 \left\|x_\alpha^*\right\|^2.  
$$
This recursion implies that
\begin{align}
 &b \mathbb{E}\left[\bar{F}\left(x^{(k)}\right)- \bar{F}^*\right] + \mathbb{E}\left\|x^{(k+1)} + p^{(k+1)}-x_\alpha^*\right\|^2   = V_{k+1}
 \nonumber
 \\
&\leq V_1 q^k + \frac{1}{1-q} 
\left( \frac{2}{(1-\beta)^2} \left(\eta^4 \sigma^2 N + \eta^3 2\gamma Nd\right) 
+ c_2 \left\|x_\alpha^*\right\|^2
\right)\,,
 \label{ineq-final-bound}
\end{align}
where we recall that
$ q = \max(q_1, q_2)$.
From the representation, \eqref{ineq-momentum-cond}, and the fact that $\alpha=\eta^2$ we observe that 
\beq \frac{1}{2}\frac{\alpha\mu}{(1+\beta) + (1-\beta)\left(\frac{\eta^2\mu}{1-\lambda_N^W + \eta^2 L }\right)} = c_1 \leq \frac{1}{2}\frac{\alpha\mu}{(1+\beta)}\,.
\label{ineq-ub-c1}
\eeq 
Using the fact that the function $h(\alpha):= \frac{\alpha \mu}{1-\lambda_N^W + \alpha L }$ is monotonically increasing on the positive real line and by the assumption \eqref{ineq-step-cond}, we obtain $$h(\alpha) \leq h\left(\frac{1+\lambda_N^W}{2(L+\mu)}\right) 
= \frac{(1+\lambda_N^W)\mu}{(1+\lambda_N^W)L + 2(1-\lambda_N^W)(L+\mu)} = \Theta(1).$$
Therefore, we also have the following lower bound for $c_1$:
\beq 
\frac{1}{2}\frac{\alpha\mu}{(1+\beta) + (1-\beta)\Theta(1)} \leq c_1\,.
\label{ineq-lb-c1}
\eeq
It follows then from \eqref{ineq-lb-c1} and \eqref{ineq-ub-c1} that
\beq c_1 = \Theta(\alpha) = \Theta\left(\eta^2\right).
\label{eq-order-c1}
\eeq
Consequently, by our assumption \eqref{ineq-momentum-cond}, we have 
\beq \beta = 1-\eta \gamma = \mathcal{O}\left(\eta^3 \sqrt{c_1}\right) = \mathcal{O}\left(\eta^4\right)= \mathcal{O}\left(\alpha^2\right),\quad \gamma \eta = 1-\beta = \Theta(1).
\label{eq-order-beta}
\eeq
Then, it follows from the definition of $c_2$ 
that 
$$c_2 = \Theta\left({\beta^2}\left(1+\frac{1}{\alpha}\right)\right) 
= \mathcal{O}\left(\beta^2 + \frac{\beta^2}{\alpha}\right)\,.$$
Since $\beta = \mathcal{O}(\alpha^2)$ by \eqref{eq-order-beta}, we obtain
\begin{equation} 
c_2 = \mathcal{O}\left(\alpha^3\right).
\label{eq-order-c2}
\end{equation}
This also implies that 
\begin{align} 
V_1 &= b \mathbb{E}\left[\bar{F}\left(x^{(0)}\right)- \bar{F}^*\right] + \mathbb{E} \left\|x^{(0)} + \frac{\beta}{1-\beta}\left(x^{(1)}-x^{(0)}\right) -x_\alpha^*\right\|^2 \\
&\leq b L_\alpha \left\|x^{(0)} -x_\alpha^*\right\|^2 + 2 \mathbb{E} \left\|x^{(0)}  -x_\alpha^*\right\|^2  +
\frac{2\beta^2}{(1-\beta)^2}\mathbb{E}\left\|x^{(1)}-x^{(0)} \right\|^2\,, \label{eq-V0-bound}
\end{align}
where
\begin{equation} 
b = \frac{2\alpha}{1-\beta}\left(\frac{\beta - \mu\alpha}{1-\beta}+ \frac{2\mu}{L_\alpha+\mu} - \frac{2c_1\alpha\mu}{1-\beta}
\right)= \Theta\left(\alpha^2\right)\,, \label{eq-order-b}
\end{equation}
where we used the fact that $\beta = \mathcal{O}(\alpha^2)$, $1-\lambda_N^W = \Theta(1)$ and \eqref{eq-order-c1}. Then,  
\begin{equation}\label{eq-order-bLalpha} 
bL_\alpha = b \left(\frac{1-\lambda_N^W}{\alpha}+L\right)=\mathcal{O}(\alpha)\,. 
\end{equation}
From the definition of $a$ given in \eqref{a:eqn}, we have also  
 $a= \mathcal{O}\left(\alpha\beta + c_2 
 \right) = \mathcal{O}\left(\alpha^3\right)\,$,
where we used $\beta = \Theta(\alpha^2)$ and $c_2 = \mathcal{O}(\alpha^3)$ obtained in \eqref{eq-order-c2}. Then, it follows from \eqref{eq-order-b} that 
\begin{equation}
q_1 = \frac{a}{b} = \mathcal{O}(\alpha).
\label{eq-order-q1}
\end{equation}
We also see that 
\beq  q_2 = \max(0, 1-2c_1) = 1-\Theta(\alpha),
\label{eq-order-q2}
\eeq
due to \eqref{eq-order-c1}. This estimate and \eqref{eq-order-q1} implies
\begin{align}
q = \max(q_1,q_2) =  1-\Theta(\alpha), \quad 
\frac{1}{1-q} = \frac{1}{\Theta(\alpha)} = \Theta\left( \frac{1}{\eta^2}\right). \label{eq-order-q}
\end{align}
On the other hand, a consequence of Lemma~\ref{lem:C1:gamma} is that
 \beq\|x_\alpha^*\| = \mathcal{O}(1).
 \label{eq-xalpha-bd}
 \eeq
Then, we get from \eqref{eq-order-bLalpha} and \eqref{eq-V0-bound} that 
    $V_1 = \mathcal{O}(1)$.
Combining this fact with the estimates \eqref{eq-order-c1}, \eqref{eq-order-c2}, \eqref{eq-order-beta}, \eqref{eq-order-q1}, \eqref{eq-order-q2}, \eqref{eq-order-q} 
we conclude that the terms on the right-hand side of \eqref{ineq-final-bound} satisfy
$ 
V_1 q^k = \Theta \left( \left(1-\Theta\left(\eta^2\right)\right)^k\right)\,, 
$
and 
\begin{align*}
&\frac{1}{1-q} 
\left( \frac{2}{(1-\beta)^2} \left(\eta^4 \sigma^2 N + \eta^3 2\gamma Nd\right) 
+ c_2 \|x_\alpha^*\|^2\right) = \frac{\mathcal{O}(\eta^3)}{\Theta(\eta^2)}=  \mathcal{O}(\eta)\,,
\\
&V_{k+1}= b \mathbb{E}\left[\bar{F}\left(x^{(k)}\right)- \bar{F}^*\right] + \mathbb{E} \left\|x^{(k+1)} + p^{(k+1)}-x_\alpha^*\right\|^2  \leq  c_3\,,
\end{align*}
for some constant $c_3 = \mathcal{O}(1)$ and every $k\geq 0$. Resorting to the estimate \eqref{eq-xalpha-bd}, we conclude that this proves \eqref{ineq-unif-l2}. The last inequality also implies that
\beq
\mathbb{E} \left\|x^{(k)}- x_\alpha^*\right\|^2w\leq \frac{1}{\mu}\mathbb{E}\left[\bar{F}\left(x^{(k)}\right)- \bar{F}^*\right]  \leq \frac{c_3}{\mu b}\,,
\label{ineq-c3-over-mub}
\eeq
as well as the inequality
\begin{align*}
\mathbb{E}\left\|
\left(1+\frac{\beta}{1-\beta}\right) x^{(k)} - \beta \frac{x^{(k-1)}}{1-\beta} \right\|^2 
&=\mathbb{E}\left\|
x^{(k)} + p^{(k)} \right\|^2
\\
&\leq 2 \mathbb{E} \left\|x^{(k)} + p^{(k)}-x_\alpha^*\right\|^2 
+ 2 \|x_\alpha^*\|^2 
\leq 2 c_3 + 2 \|x_\alpha^*\|^2\,, \nonumber
\end{align*}
where we applied the Cauchy-Schwarz inequality. If we apply the Cauchy-Schwarz inequality again, we obtain
\begin{align*}
\mathbb{E}\left\|x^{(k)}\right\|^2 
&\leq 
\mathbb{E}\left\|
\left(1+\frac{\beta}{1-\beta}\right) x^{(k)}\right\|^2 
\\
&\leq 
2\mathbb{E}\left\|
\left(1+\frac{\beta}{1-\beta}\right) x^{(k)} - \beta \frac{x^{(k-1)}}{1-\beta} \right\|^2 +  2 \mathbb{E}\left\|\beta \frac{x^{(k-1)}}{1-\beta} \right\|^2 \\
&\leq 2\mathbb{E}\left\|
\left(1+\frac{\beta}{1-\beta}\right) x^{(k)} - \beta \frac{x^{(k-1)}}{1-\beta} \right\|^2 
+ \frac{2\beta^2}{(1-\beta)^2} \left(2\mathbb{E}\left\|x^{(k-1)} - x_\alpha^* \right\|^2+2\|x_\alpha^*\|^2\right)\\
&\leq 4c_3 + 4\|x_\alpha^*\|^2 + \frac{2\beta^2}{(1-\beta)^2} \left(\frac{2c_3}{\mu b} + 2\|x_\alpha^*\|^2\right)\,,
\end{align*}
where we used \eqref{ineq-c3-over-mub}.

Within our assumptions on the stepsize and momentum $\beta$, $b = \mathcal{O}(\alpha^2)$ and $\|x_\alpha^*\| = \Theta(1)$. Furthermore, we have $\beta = \mathcal{O}(\alpha^2)$ as well as $b = \Theta(\alpha^2) = \Theta(\eta^4)$. Therefore, we conclude that
$\mathbb{E}\left\|x^{(k)} \right\|^2  = \mathcal{O}(1)$,
which is equivalent to $\eqref{ineq-unif-l2-iter}$. This implies that
\begin{equation*}
\mathbb{E} \left\|\nabla F\left(x^{(k)}\right)\right\|^2 
\leq \tilde{c}_4: = L \left(\|\nabla F(x^*)\|^2 + \sup_{k\geq 0}\mathbb{E} \left\|x^{(k)} - x^*\right\|^2\right),
\end{equation*}
where we used $L$-smoothness of $F$. Consequently, we find from the update equation~\eqref{eq:underdamped} that
  \begin{align} 
  \mathbb{E} \left\|v^{(k+1)}\right\|^2 &= \mathbb{E} \left\| \beta v^{(k)}- \eta \nabla F\left(x^{(k)}\right)\right\|^2
  +\eta^2 \sigma^2 N + {2\gamma\eta}Nd\\
    &\leq 2 \beta^2  \mathbb{E} \left\| v^{(k)}\right\|^2 + 2 \eta^2 \mathbb{E} \left\| \nabla F\left(x^{(k)}\right)\right\|^2 + \eta^2 \sigma^2 N + {2\gamma\eta}Nd,
\end{align}
which implies that for any $k$,
\beq
\mathbb{E} \left\|v^{(k+1)}\right\|^2  
\leq c_5:=\mathbb{E} \left\|v^{(0)}\right\|^2 
+\frac{2 \eta^2 \tilde{c}_4 + \eta^2 \sigma^2 N + {2\gamma\eta}Nd}{1-2\beta^2}  = \mathcal{O}(1),
  \eeq
where we used the fact that $2\beta^2 < 1$ by our assumptions. This completes the proof \rev{of Lemma~\ref{lemma-l2-underdamped}}.

\hfill $\Box$


\rev{The next three technical lemmas are used in the proof of Lemma~\ref{lemma-l2-underdamped}}.

\begin{lemma}\label{lem-negative-coef} 
In the setting of the proof of Lemma~\ref{lemma-l2-underdamped}; if the parameters $\alpha$ and $\beta$ satisfy the inequalities \eqref{eq-cond-alpha} and \eqref{eq-cond-beta}, then $k_{\alpha,\beta} < 0$ where $k_{\alpha,\beta}$ is defined by \eqref{def-k-ab}. 
\end{lemma}

\textbf{Proof of Lemma~\ref{lem-negative-coef}}
The proof of Lemma~\ref{lem-negative-coef} will be provided in Appendix~\ref{sec:additional}.
\hfill $\Box$


\begin{lemma}\label{lemma-q1} 
In the setting of Lemma~\ref{lemma-l2-underdamped}, let $\alpha$ and $\beta$ satisfy the conditions \eqref{eq-cond-alpha}, \eqref{eq-cond-beta} and \eqref{eq-cond-beta2} where $\alpha$ is defined by \eqref{def-alpha}. Then, we have 
\beq q_1 := \frac{a}{b}  \in (0,1)\,,
\label{def-q1}
\eeq
where $a$ and $b$ are defined in \eqref{a:eqn} and \eqref{b:eqn}.
\end{lemma}

\textbf{Proof of Lemma~\ref{lemma-q1}}
The proof of Lemma~\ref{lemma-q1} will be provided in Appendix~\ref{sec:additional}.
\hfill $\Box$ 

\begin{lemma}\label{lemma-q2} 
In the setting of the proof of Lemma~\ref{lemma-l2-underdamped}, we have $s M - {Q} \succeq 0$ if 
\beq \label{def-q2}
s\geq q_2  := \max\left(0, 1 - \frac{\alpha\mu L_\alpha}{(1-\beta)(L_\alpha+\mu) + 2L_\alpha \beta}\right)\,.	
\eeq
\end{lemma}

\textbf{Proof of Lemma~\ref{lemma-q2}}
The proof of Lemma~\ref{lemma-q2} will be provided in Appendix~\ref{sec:additional}.
\hfill $\Box$


\subsection{Proofs of Lemma~\ref{lem:1:under}--~\ref{lem:4:under}}

The proofs of Lemmas~\ref{lem:1:under} and~\ref{lem:2:under} are inspired
by the proofs of Lemmas~\ref{lem:1} and~\ref{lem:2} respectively, 
and the proofs of Lemmas~\ref{lem:3:under} and~\ref{lem:4:under} 
are inspired by the proof of Lemma~\ref{lem:3}.
We will present the proofs of Lemmas~\ref{lem:1:under}--\ref{lem:4:under} in Appendix~\ref{sec:additional}.
\hfill $\Box$


\section{Additional Technical Proofs}\label{sec:additional}

\subsection{Proof of Lemma~\ref{bound:Gibbs}}
Note that
$\mathbb{E}_{X\sim\pi}\Vert X-x_{\ast}\Vert^{2}
=\mathbb{E}\Vert X_{\infty}-x_{\ast}\Vert^{2}$,
where $X_{\infty}$ is the unique stationary distribution
of the overdamped Langevin diffusion:
\begin{equation*}
dX_{t}=-\frac{1}{N}\nabla f(X_{t})dt+\sqrt{2N^{-1}}dW_{t},
\end{equation*}
where $W_{t}$ is a standard $d-$dimensional Brownian motion.
By It\^{o}'s formula, we have
\begin{align*}
e^{\mu t}\Vert X_{t}-x_{\ast}\Vert^{2}
&=\Vert X_{0}-x_{\ast}\Vert^{2}
+2\sqrt{2N^{-1}}\int_{0}^{t}e^{\mu s}\langle X_{s}-x_{\ast},dW_{s}\rangle
\\
&\qquad\qquad
-2\int_{0}^{t}e^{\mu s}\left\langle X_{s}-x_{\ast},\frac{1}{N}\nabla f(X_{s})\right\rangle ds
\\
&\qquad\qquad\qquad
+2N^{-1}d\int_{0}^{t}e^{\mu s}ds
+\mu\int_{0}^{t}e^{\mu s}\Vert X_{s}-x_{\ast}\Vert^{2}ds
\\
&\leq
\Vert X_{0}-x_{\ast}\Vert^{2}
+2\sqrt{2N^{-1}}\int_{0}^{t}e^{\mu s}\langle X_{s}-x_{\ast},dW_{s}\rangle
+2dN^{-1}\int_{0}^{t}e^{\mu s}ds,
\end{align*}
where we used $\mu$-strongly convex property of $x\mapsto\frac{1}{N}f(x)$.
This implies that
\begin{equation*}
\mathbb{E}\Vert X_{t}-x_{\ast}\Vert^{2}
\leq e^{-\mu t}\Vert X_{0}-x_{\ast}\Vert^{2}+\frac{2dN^{-1}}{\mu},
\end{equation*}
and therefore $\mathbb{E}\Vert X_{\infty}-x_{\ast}\Vert^{2}
\leq\frac{2dN^{-1}}{\mu}$. The proof is complete.
\hfill $\Box$ 


\subsection{Proof of Lemma~\ref{lem:1:under}}
\rev{In this proof, we aim to provide uniform $L_2$ bounds between the iterates $x_i^{(k)}$ and their means $\bar{x}^{(k)}$.
First,} by the definitions of $x^{(k)}$, we get
\begin{equation*}
x^{(k+1)}=(W\otimes I_{d})x^{(k)}+\eta v^{(k+1)}.
\end{equation*}
It follows that
\begin{equation*}
x^{(k)}=\left(W^{k}\otimes I_{d}\right)x^{(0)}
+\eta\sum_{s=0}^{k-1}\left(W^{k-1-s}\otimes I_{d}\right)v^{(s+1)}.
\end{equation*}
Let us define $\mathbf{\bar{x}}^{(k)}:=[\bar{x}^{(k)},\cdots,\bar{x}^{(k)}]\in\mathbb{R}^{Nd}$.
Notice that
\begin{equation*}
\mathbf{\bar{x}}^{(k)}=\frac{1}{N}\left(\left(1_{N}1_{N}^{T}\right)\otimes I_{d}\right)x^{(k)}.
\end{equation*}
Therefore, we get
\begin{equation*}
\sum_{i=1}^{N}\left\Vert x_{i}^{(k)}-\bar{x}^{(k)}\right\Vert^{2}
=\left\Vert x^{(k)}-\mathbf{\bar{x}}^{(k)}\right\Vert^{2}
=\left\Vert x^{(k)}-\frac{1}{N}\left(\left(1_{N}1_{N}^{T}\right)\otimes I_{d}\right)x^{(k)}\right\Vert^{2},
\end{equation*}
and by the Cauchy-Schwarz inequality
\begin{align*}
&\left\Vert x^{(k)}-\frac{1}{N}\left(\left(1_{N}1_{N}^{T}\right)\otimes I_{d}\right)x^{(k)}\right\Vert^{2}
\\
&\leq
2\left\Vert\left(W^{k}\otimes I_{d}\right)x^{(0)}
-\frac{1}{N}\left(\left(1_{N}1_{N}^{T}W^{k}\right)\otimes I_{d}\right)x^{(0)}\right\Vert^{2}
\\
&\qquad
+2\left\Vert-\eta\sum_{s=0}^{k-1}\left(W^{k-1-s}\otimes I_{d}\right)v^{(s+1)}
+\eta\sum_{s=0}^{k-1}\frac{1}{N}\left(\left(1_{N}1_{N}^{T}W^{k-1-s}\right)\otimes I_{d}\right)v^{(s+1)}\right\Vert^{2}
\\
&=2\left\Vert\left(W^{k}\otimes I_{d}\right)x^{(0)}
-\frac{1}{N}\left(\left(1_{N}1_{N}^{T}\right)\otimes I_{d}\right)x^{(0)}\right\Vert^{2}
\\
&\qquad
+2\left\Vert-\eta\sum_{s=0}^{k-1}\left(W^{k-1-s}\otimes I_{d}\right)v^{(s+1)}
+\eta\sum_{s=0}^{k-1}\frac{1}{N}\left(\left(1_{N}1_{N}^{T}\right)\otimes I_{d}\right)v^{(s+1)}\right\Vert^{2}
\\
&=2\left\Vert\left(\left(W^{k}
-\frac{1}{N}1_{N}1_{N}^{T}\right)\otimes I_{d}\right)x^{(0)}\right\Vert^{2}
+2\eta^{2}\left\Vert\sum_{s=0}^{k-1}\left(\left(W^{k-1-s}-\frac{1}{N}1_{N}1_{N}^{T}\right)\otimes I_{d}\right)v^{(s+1)}\right\Vert^{2}.
\end{align*}

Note that
\begin{align*}
&\eta^{2}\left\Vert\sum_{s=0}^{k-1}\left(\left(W^{k-1-s}-\frac{1}{N}1_{N}1_{N}^{T}\right)\otimes I_{d}\right)v^{(s+1)}\right\Vert^{2}
\\
&\leq
\eta^{2}\left(\sum_{s=0}^{k-1}\left\Vert\left(W^{k-1-s}-\frac{1}{N}1_{N}1_{N}^{T}\right)\otimes I_{d}\right\Vert
\cdot\left\Vert v^{(s+1)}\right\Vert\right)^{2}
\\
&\leq
\eta^{2}\left(\sum_{s=0}^{k-1}\left\Vert W^{k-1-s}-\frac{1}{N}1_{N}1_{N}^{T}\right\Vert
\cdot\left\Vert v^{(s+1)}\right\Vert\right)^{2}
\\
&=
\eta^{2}\left(\sum_{s=0}^{k-1}\bar{\gamma}^{k-1-s}
\cdot\left\Vert v^{(s+1)}\right\Vert\right)^{2}
\\
&=\eta^{2}\left(\sum_{s=0}^{k-1}\bar{\gamma}^{k-1-s}\right)^{2}\left(\frac{\sum_{s=0}^{k-1}\bar{\gamma}^{k-1-s}
\cdot\left\Vert v^{(s+1)}\right\Vert}{\sum_{s=0}^{k-1}\bar{\gamma}^{k-1-s}}\right)^{2}
\\
&\leq
\eta^{2}\left(\sum_{s=0}^{k-1}\bar{\gamma}^{k-1-s}\right)^{2}
\sum_{s=0}^{k-1}\frac{\bar{\gamma}^{k-1-s}}{\sum_{s=0}^{k-1}\bar{\gamma}^{k-1-s}}\left\Vert v^{(s+1)}\right\Vert^{2},
\end{align*}
where we used Jensen's inequality in the last step above.

Recall from Lemma~\ref{lemma-l2-underdamped} that for every $s$,
$\mathbb{E}\left[\left\Vert v^{(s+1)}\right\Vert^{2}\right]
\leq c_{5}$.
Therefore, we have
\begin{align*}
&\eta^{2}\mathbb{E}\left[\left\Vert\sum_{s=0}^{k-1}\left(\left(W^{k-1-s}-\frac{1}{N}1_{N}1_{N}^{T}\right)\otimes I_{d}\right)
v^{(s+1)}\right\Vert^{2}\right]
\\
&\leq
c_{5}\eta^{2}\left(\sum_{s=0}^{k-1}\bar{\gamma}^{k-1-s}\right)^{2}
\sum_{s=0}^{k-1}\frac{\bar{\gamma}^{k-1-s}}{\sum_{s=0}^{k-1}\bar{\gamma}^{k-1-s}}
\leq
c_{5}\eta^{2}\frac{1}{(1-\bar{\gamma})^{2}}.
\end{align*}
Similarly, we have
\begin{equation*}
\left\Vert\left(\left(W^{k}
-\frac{1}{N}1_{N}1_{N}^{T}\right)\otimes I_{d}\right)x^{(0)}\right\Vert^{2}
\leq\bar{\gamma}^{2k}\left\Vert x^{(0)}\right\Vert^{2}.
\end{equation*}
The proof is complete.
\hfill $\Box$ 


\subsection{Proof of Lemma~\ref{lem:2:under}}
By Lemma~\ref{lem:1:under}, we can compute that
\begin{align*}
\mathbb{E}\left\Vert\mathcal{E}_{k+1}\right\Vert^{2}
&=\mathbb{E}\left\Vert
\frac{1}{N}\sum_{i=1}^{N}\left(\nabla f_{i}\left(x_{i}^{(k)}\right)
-\nabla f_{i}\left(\bar{x}^{(k)}\right)\right)\right\Vert^{2}
\\
&\leq
\frac{1}{N^{2}}\sum_{i=1}^{N}
N\mathbb{E}\left\Vert
\nabla f_{i}\left(x_{i}^{(k)}\right)
-\nabla f_{i}\left(\bar{x}^{(k)}\right)\right\Vert^{2}
\\
&\leq\frac{1}{N}L^{2}\sum_{i=1}^{N}
\mathbb{E}\left\Vert
x_{i}^{(k)}
-\bar{x}^{(k)}\right\Vert^{2}
\\
&\leq
\frac{2L^{2}\bar{\gamma}^{2k}}{N}\mathbb{E}\left\Vert x^{(0)}\right\Vert^{2}
+\frac{2L^{2}c_{5}\eta^{2}}{N(1-\bar{\gamma})^{2}}.
\end{align*} 
The proof is complete.
\hfill $\Box$

\subsection{Proof of Lemma~\ref{lem:3:under}}

\rev{In this proof, we aim to show that 
the average iterates $\bar{x}^{(k)}$ are close to the iterates $\tilde{x}_k$ which is defined in \eqref{eq:tilde-x-k}. 
First,} we can compute that
\begin{equation*}
\bar{x}^{(k+1)}-\tilde{x}_{k+1}=\bar{x}^{(k)}-\tilde{x}_{k}
-\frac{\eta^{2}}{N}\left[\nabla f\left(\bar{x}^{(k)}\right)-\nabla f(\tilde{x}_{k})\right]
+\beta\left(\bar{x}^{(k)}-\bar{x}^{(k-1)}\right)
+\eta^{2}\mathcal{E}_{k+1}-\eta^{2}\bar{\xi}^{(k+1)},
\end{equation*}
where
\begin{equation*}
\mathcal{E}_{k+1}=\frac{1}{N}\nabla f\left(\bar{x}^{(k)}\right)-\frac{1}{N}\sum_{i=1}^{N}\nabla f_{i}\left(x_{i}^{(k)}\right).
\end{equation*}
We also observe that under our assumptions $\eta \in (0,\sqrt{2/L})$, $\eta^{2}\mu(1-\frac{\eta^{2} L}{2})\leq 1$.
Then, it follows from the proof of Lemma~\ref{lem:3}
that we have
\begin{align}
\mathbb{E}\left\Vert\bar{x}^{(k+1)}-\tilde{x}_{k+1}\right\Vert^{2}
&\leq
\left(1-\eta^{2}\mu\left(1-\frac{\eta^{2} L}{2}\right)\right)\mathbb{E}\left\Vert\bar{x}^{(k)}-\tilde{x}_{k}\right\Vert^{2}
+\eta^{4}\frac{\sigma^{2}}{N}
\nonumber
\\
&\qquad\qquad
+\eta^{2}\left(\eta^{2}+\frac{(1+\eta^{2} L)^{2}}{\mu(1-\frac{\eta^{2} L}{2})}\right)
\mathbb{E}\left\Vert\frac{\beta}{\eta^{2}}\left(\bar{x}^{(k)}-\bar{x}^{(k-1)}\right)
+\mathcal{E}_{k+1}\right\Vert^{2}
\nonumber
\\
&\leq
\left(1-\eta^{2}\mu\left(1-\frac{\eta^{2} L}{2}\right)\right)\mathbb{E}\left\Vert\bar{x}^{(k)}-x_{k}\right\Vert^{2}
+\eta^{4}\frac{\sigma^{2}}{N}
\nonumber
\\
&\qquad\qquad\qquad
+2\eta^{2}\left(\eta^{2}+\frac{(1+\eta^{2} L)^{2}}{\mu(1-\frac{\eta^{2} L}{2})}\right)
\left(\mathbb{E}\left\Vert\frac{\beta}{\eta}\bar{v}^{(k)}\right\Vert^{2}
+\mathbb{E}\left\Vert\mathcal{E}_{k+1}\right\Vert^{2}\right),\label{take:under}
\end{align}
where we used $\bar{x}^{(k)}-\bar{x}^{(k-1)}=\eta\bar{v}^{(k)}$ and \eqref{bar:grad:noise}.
We recall from Lemma~\ref{lem:2:under} that
\begin{align}\label{take:under:2}
\mathbb{E}\left\Vert\mathcal{E}_{k+1}\right\Vert^{2}
\leq
\frac{2L^{2}\bar{\gamma}^{2k}}{N}\mathbb{E}\left\Vert x^{(0)}\right\Vert^{2}
+\frac{2L^{2}c_{5}\eta^{2}}{N(1-\bar{\gamma})^{2}},
\end{align}
and by Lemma~\ref{lemma-l2-underdamped}, we get
\begin{equation}\label{take:under:3}
\mathbb{E}\left\Vert\bar{v}^{(k)}\right\Vert^{2}
\leq\frac{1}{N}\sum_{i=1}^{N}\mathbb{E}\left\Vert v_{i}^{(k)}\right\Vert^{2}
=\frac{1}{N}\mathbb{E}\left\Vert v^{(k)}\right\Vert^{2}
\leq\frac{c_{5}}{N}.
\end{equation}
By applying \eqref{take:under:2}-\eqref{take:under:3} to \eqref{take:under}, we get
\begin{align*}
&\mathbb{E}\left\Vert\bar{x}^{(k+1)}-\tilde{x}_{k+1}\right\Vert^{2}
\\
&\leq\left(1-\eta^{2}\mu\left(1-\frac{\eta^{2} L}{2}\right)\right)
\mathbb{E}\left\Vert\bar{x}^{(k)}-\tilde{x}_{k}\right\Vert^{2}
+\eta^{4}\frac{\sigma^{2}}{N}
\\
&\qquad
+2\eta^{2}\left(\eta^{2}+\frac{(1+\eta^{2} L)^{2}}{\mu(1-\frac{\eta^{2} L}{2})}\right)
\left(\frac{\beta^{2}c_{5}}{\eta^{2}N}+\frac{2L^{2}\bar{\gamma}^{2k}}{N}\mathbb{E}\left\Vert x^{(0)}\right\Vert^{2}
+\frac{2L^{2}c_{5}\eta^{2}}{N(1-\bar{\gamma})^{2}}\right),
\end{align*}
for every $k$. Note that $\mathbb{E}\left\Vert\bar{x}^{(0)}-\tilde{x}_{0}\right\Vert^{2}=0$.
By our assumption on stepsize $\eta$, we have
$1-\eta^{2}\mu\left(1-\frac{\eta^{2} L}{2}\right)\in[0,1)$.
By following the same argument
as in the proof of Lemma~\ref{lem:3}, 
we conclude that for every $k$,
\begin{align*}
\mathbb{E}\left\Vert\bar{x}^{(k)}-\tilde{x}_{k}\right\Vert^{2}
&\leq\frac{2\eta^{2}\left(\eta^{2}+\frac{(1+\eta^{2} L)^{2}}{\mu(1-\frac{\eta^{2} L}{2})}\right)
\left(\frac{\beta^{2}c_{5}}{\eta^{2}N}
+\frac{2L^{2}c_{5}\eta^{2}}{N(1-\bar{\gamma})^{2}}
\right)+\eta^{4}\frac{\sigma^{2}}{N}}{
1-\left(1-\eta^{2}\mu\left(1-\frac{\eta^{2} L}{2}\right)\right)}
\\
&\qquad\qquad
+\frac{\bar{\gamma}^{2k}-
\left(1-\eta^{2}\mu\left(1-\frac{\eta^{2} L}{2}\right)\right)^{k}}
{\bar{\gamma}^{2}-1+\eta^{2}\mu\left(1-\frac{\eta^{2} L}{2}\right)}
\frac{4L^{2}\bar{\gamma}^{2}}{N}\mathbb{E}\left\Vert x^{(0)}\right\Vert^{2},
\end{align*}
which completes the proof.
\hfill $\Box$ 

\subsection{Proof of Lemma~\ref{lem:4:under}}

\rev{In this proof, we aim to show that the iterates $\tilde{x}_k$, which is defined in \eqref{eq:tilde-x-k},
is close to the iterates $x_{k}$ in \eqref{eq:under-over} obtained from an Euler-Maruyama discretization of an overdamped Langevin SDE. First,} we can compute that
\begin{equation*}
\tilde{x}_{k+1}-x_{k+1}=\tilde{x}_{k}-x_{k}
-\frac{\eta^{2}}{N}\left[\nabla f(\tilde{x}_{k})-\nabla f(x_{k})\right]
+\left(\sqrt{2(1-\beta)}-\sqrt{2}\right)\eta\bar{w}^{(k+1)}.
\end{equation*}
It follows from the arguments in the proof of Lemma~\ref{lem:3}
that we have
\begin{align*}
\left\Vert\bar{x}^{(k+1)}-\tilde{x}_{k+1}\right\Vert^{2}
&\leq
\left(1-\eta^{2}\mu\left(1-\frac{\eta^{2} L}{2}\right)\right)\left\Vert\bar{x}^{(k)}-\tilde{x}_{k}\right\Vert^{2}
\nonumber
\\
&\qquad
+\eta^{2}\left(\eta^{2}+\frac{(1+\eta^{2} L)^{2}}{\mu(1-\frac{\eta^{2} L}{2})}\right)
\left\Vert\frac{1}{\eta^{2}}\left(\sqrt{2(1-\beta)}-\sqrt{2}\right)\eta\bar{w}^{(k+1)}\right\Vert^{2}.
\end{align*}
By taking the expectations, we get
\begin{align*}
&\mathbb{E}\left\Vert\bar{x}^{(k+1)}-\tilde{x}_{k+1}\right\Vert^{2}
\\
&\leq\left(1-\eta^{2}\mu\left(1-\frac{\eta^{2} L}{2}\right)\right)
\mathbb{E}\left\Vert\bar{x}^{(k)}-\tilde{x}_{k}\right\Vert^{2}
+2\left(\eta^{2}+\frac{(1+\eta^{2} L)^{2}}{\mu(1-\frac{\eta^{2} L}{2})}\right)
\left(\sqrt{(1-\beta)}-1\right)^{2}\frac{d}{N},
\end{align*}
for every $k$. The rest of the proof follows similarly as in the proof of Lemma~\ref{lem:3:under}.
\hfill $\Box$

\subsection{Proof of Lemma~\ref{lem:C1:gamma}}
Note that $x_{\eta}^{\ast}$ by its definition 
coincides with the fixed point $\hat{x}^{\infty}$ of the decentralized gradient descent without noise:
\begin{equation*}
\hat{x}^{(k+1)}=\mathcal{W}\hat{x}^{(k)}-\eta\nabla F\left(\hat{x}^{(k)}\right),
\end{equation*}
i.e.
\begin{equation*}
\hat{x}^{\infty}=\mathcal{W}\hat{x}^{\infty}-\eta\nabla F\left(\hat{x}^{\infty}\right),
\end{equation*}
and $x_{\eta}^{\ast}=\hat{x}^{\infty}$. 
Since $x_{\eta}^{\ast}$ and $x^{\ast}$
do not depend on $\hat{x}^{(0)}$, 
to get a bound on $\Vert x_{\eta}^{\ast}-x^{\ast}\Vert$, 
we can assume that $\hat{x}^{(0)}=0$, and apply Corollary 9 in \citet{Yuan16}
which is re-stated in \citet{robust-network-asg}:
\begin{equation*}
\left\| \hat{x}_{i}^{\infty} - x^{\ast}\right\| \leq C_1 \frac{\eta}{1-\bar{\gamma}}, \quad \mbox{where} 
\quad \bar{\gamma} := \max\left\{\left|\lambda_2^W\right|, \left|\lambda_N^W\right|\right\},
\end{equation*} 
where $x^{\ast}=(x_{\ast}^{T},x_{\ast}^{T},\ldots,x_{\ast}^{T})^{T}$,
where $x^{\ast}$ is the minimizer of $f(x)$,
which yields that
\begin{equation*}
\left\| x_{\eta}^{\ast} - x^{\ast}\right\| \leq C_1 \frac{\eta\sqrt{N}}{1-\bar{\gamma}}, \quad \mbox{where} 
\quad \bar{\gamma} := \max\left\{\left|\lambda_2^W\right|, \left|\lambda_N^W\right|\right\}.
\end{equation*} 
The proof is complete. 
\hfill $\Box$ 


\subsection{Proof of Lemma~\ref{lem-negative-coef}}
By the definition of $L_\alpha$ given by \eqref{def-Lalpha}, we have 
 \beq k_{\alpha,\beta}:=\frac{\alpha}{(1-\beta)} \left( \frac{\alpha}{1-\beta} - \frac{2\alpha}{1-\lambda_N^W + (L+\mu)\alpha} + \frac{2c_1 \alpha}{1-\beta} \right)\,.
 \eeq
 Due to \eqref{eq-cond-alpha}, we have $(L+\mu)\alpha \leq \frac{1+\lambda_N^W}{2}$, therefore
\beq 
 k_{\alpha,\beta}\leq \frac{\alpha}{(1-\beta)} \left( \frac{\alpha}{1-\beta} - \frac{2\alpha}{1-\lambda_N^W + (1+\lambda_N^W)/2} + \frac{2c_1 \alpha}{1-\beta} \right)\,.
 \label{ineq-k-ab}
\eeq
Furthermore, 
\beq c_1 = \frac{1}{2}\frac{\alpha\mu}{(1-\beta)(1+\mu/L_\alpha) + 2\beta} < \frac{1}{2} \frac{\alpha\mu}{1+\beta} < \frac{1}{2}\alpha\mu\,,
\label{ineq-c1-ub}
\eeq
where we used the fact that $\mu/L_\alpha>0$. Therefore, by replacing $c_1$ in \eqref{ineq-k-ab} with its upper bound \eqref{ineq-c1-ub}, we obtain 
\begin{align*}
 k_{\alpha,\beta}&\leq \frac{\alpha^2}{(1-\beta)} \left( \frac{1}{1-\beta} - \frac{2}{1-\lambda_N^W + (1+\lambda_N^W)/2} + \frac{ \alpha\mu}{1-\beta} \right) \\
  &= \frac{\alpha^2}{(1-\beta)} \left( \frac{1}{1-\beta} - \frac{4}{3-\lambda_N^W} + \frac{ \alpha\mu}{1-\beta} \right)\,. 
\end{align*}
Since $\alpha>0$ and $\beta < 1/2$ by our assumptions, $\frac{\alpha^2}{1-\beta}>0$ and $k_{\alpha,\beta} < 0$ if and only if 
$$\frac{1}{1-\beta} - \frac{4}{3-\lambda_N^W} + \frac{ \alpha\mu}{1-\beta} < 0\,,$$
which is equivalent to 
 \beq \beta < \frac{1+\lambda_N^W +\alpha\mu\lambda_N^W - 3\alpha\mu}{4}\,.
 \label{ineq-positivity-cond}
 \eeq
By our assumption \eqref{eq-cond-beta} on $\beta$, we have 
$$\beta \leq \frac{1+\lambda_N^W - 4\alpha\mu}{4}\,,$$
and noticing that $\lambda_N^W>-1$ and $\alpha\mu \lambda_N^W > -\alpha\mu$, we conclude that the inequality \eqref{ineq-positivity-cond} holds. Hence, we obtain $k_{\alpha,\beta}<0$ and the proof is complete.
\hfill $\Box$


\subsection{Proof of Lemma~\ref{lemma-q1}}
Using the definitions of $a$ and $b$ from \eqref{a:eqn} and \eqref{b:eqn}, we have 
$$q_1 =\frac{a}{b}=\frac{2\alpha\beta + 4\beta^2 (1-\lambda_N^W)^2 (1 + 1/c_1)/\mu}{2\alpha\left(\beta-\mu\alpha + 2\mu \frac{(1-\beta)}{L_\alpha+\mu} - 2c_1\alpha\mu\right)}\,,$$
where $c_1$ is given by \eqref{eq-c1}. Therefore, the condition $q_1\in(0,1)$ is equivalent to 
\beq 
b(1-\beta)^2 = 2\alpha\left(\beta-\mu\alpha + 2\mu \frac{(1-\beta)}{L_\alpha+\mu} - 2c_1\alpha\mu\right)> 0\,,
\label{ineq-first-cond}
\eeq 
where $b$ is defined by \eqref{b:eqn}
and 
\beq 
2\alpha\beta + 4\beta^2 \left(1-\lambda_N^W\right)^2 (1 + 1/c_1)/\mu < 2\alpha\left(\beta-\mu\alpha + 2\mu \frac{(1-\beta)}{L_\alpha+\mu} - 2c_1\alpha\mu\right)\,.
\label{ineq-sec-cond}
\eeq 
It suffices to show that under our assumptions on $\alpha$ and $\beta$, these two conditions are satisfied. The first condition \eqref{ineq-first-cond} is satisfied because $b = \frac{2\alpha\beta}{(1-\beta)^2} - 2 k_{\alpha,\beta}\mu > \frac{2\alpha\beta}{(1-\beta)^2}>0$ by Lemma~\ref{lem-negative-coef}.
We next prove that the second condition
\eqref{ineq-sec-cond} holds. We re-organize \eqref{ineq-sec-cond} as
\beq   
4\beta^2 \left(1-\lambda_N^W\right)^2 (c_1 + 1) < 2 c_1\mu \alpha\left(-\mu\alpha + 2\mu \frac{(1-\beta)}{L_\alpha+\mu} - 2c_1\alpha\mu\right)\,. 
\eeq 
We note that 
$$c_1 = \frac{1}{2}\frac{\alpha\mu}{(1-\beta)(1+\mu/L_\alpha) + 2\beta}\leq \frac{\alpha\mu}{2(1+\beta)} < 1\,,$$ 
where in the first inequality we used the fact that $\mu/L_\alpha>0$ whereas in the second inequality we used the assumptions \eqref{eq-cond-alpha} and \eqref{eq-cond-beta}. 
Therefore, $c_1 + 1 < 2$ and 
$2c_1\alpha\mu \leq 2 \alpha \mu \frac{\alpha\mu}{2(1+\beta)}$. Hence it suffices to have
\begin{align*}   
8 \beta^2 \left(1-\lambda_N^W\right)^2  &\leq 2 c_1\mu \alpha\left(-\mu\alpha + 2\mu \frac{(1-\beta)}{L_\alpha+\mu} - 2 \alpha \mu \frac{\alpha\mu}{2(1+\beta)}\right) \\
&= 2 c_1\mu \alpha\left(-\mu\alpha + 2\mu \frac{(1-\beta)\alpha}{1-\lambda_N^W +( L +\mu)\alpha} - 2 \alpha \mu \frac{\alpha\mu}{2(1+\beta)}\right)\,,  
\end{align*}
where we used the definition of $L_\alpha$ given in \eqref{def-Lalpha}. By assumption \eqref{eq-cond-alpha}, we have $\alpha\leq (1+\lambda_N^W)/(2(L+\mu))$; therefore it suffices to have 
\begin{align}   
8 \beta^2 \left(1-\lambda_N^W\right)^2 &\leq 2 c_1\mu \alpha\left(-\mu\alpha + 2\mu \frac{(1-\beta)\alpha}{1-\lambda_N^W + \frac{1+\lambda_N^W}{2}} - 2 \alpha \mu \frac{\alpha\mu}{2(1+\beta)}\right)\\
&= 2 c_1\mu \alpha\left(-\mu\alpha + 4\mu \frac{(1-\beta)\alpha}{3-\lambda_N^W} - 2 \alpha \mu \frac{\alpha\mu}{2(1+\beta)}\right) \\
&= 2c_1\mu^2 \alpha^2 \left(-1 + 4 \frac{(1-\beta)}{3-\lambda_N^W} -  \frac{\alpha\mu}{(1+\beta)}\right)\,.\label{ineq-to-be-relaxed}
\end{align}
By differentiating the right-hand side of \eqref{ineq-to-be-relaxed} with respect to $\beta$, it is easy to see that the right-hand side is a decreasing function of $\beta$ under our assumptions. 
Therefore, by plugging in the largest allowed value $\frac{1+\lambda_N - 4\alpha\mu}{4}$ for $\beta$ on the right-hand side of this inequality, we can relax condition \eqref{ineq-to-be-relaxed} to
\begin{align*}   
8 \beta^2 \left(1-\lambda_N^W\right)^2 
\leq  2c_1\mu^2 \alpha^2 \left(-1 + 4 \frac{1 - \frac{1+\lambda_N^W - 4\alpha\mu}{4}}{3-\lambda_N^W} - \alpha\mu\right)
= 2c_1\mu^2 \alpha^2 \left( \frac{\alpha\mu(1+\lambda_N^W)}{3-\lambda_N^W} \right)\,.
\end{align*}
Since $\lambda_N^W \in (-1,1)$, it suffices to have 
\begin{equation*}   
8\beta^2 \left(1-\lambda_N^W\right)^2 
\leq 2c_1\mu^2 \alpha^2 \left( \frac{\alpha\mu(1+\lambda_N^W)}{4} \right)\,,
\end{equation*}
which holds if and only if 
\begin{equation*}   
\beta^2
\leq c_1\mu^3 \alpha^3 \left( \frac{(1+\lambda_N^W)}{16  (1-\lambda_N^W)^2 } \right)\,.
\end{equation*} 
Since $\lambda_N^W \in (-1,1)$, it suffices to have 
\begin{equation*}   
\beta^2
\leq c_1\mu^3 \alpha^3 \left( \frac{1+\lambda_N^W}{64 } \right)\,, 
\end{equation*} 
which is exactly the condition \eqref{eq-cond-beta2} we assumed in the statement of the lemma. We conclude that the inequality \eqref{ineq-sec-cond} is also satisfied. Finally, we infer from \eqref{ineq-first-cond} and \eqref{ineq-sec-cond} that $q_1 \in (0,1)$ completing the proof.
\hfill $\Box$ 


\subsection{Proof of Lemma~\ref{lemma-q2}}
Consider the matrix pencil $S_s = sM - {Q}$ with $s\geq 0$. We have 
\begin{align*}
S_s = \begin{bmatrix}
        \left(s - 1 - 2c_1  + \frac{2\alpha\mu L_\alpha}{(1-\beta)(L_\alpha+\mu)}\right)I_d & \frac{\beta}{(1-\beta)}(s-1-2c_1)I_d \\
        \frac{\beta}{(1-\beta)}(s-1-2c_1)I_d &  \left(\frac{(s-1-2c_1)\beta^2 + \alpha\beta \mu}{(1-\beta)^2} \right)I_d
    \end{bmatrix}
    = A_s\otimes I_d\,,
\end{align*}
where $\otimes$ denotes the Kronecker product of matrices and $A_s$ is the $2\times 2$ matrix
$$A_s =  \begin{bmatrix}
    s - 1 - 2c_1+ \frac{2\alpha\mu L_\alpha}{(1-\beta)(L_\alpha+\mu)} & \frac{\beta}{(1-\beta)}(s-1-2c_1)\\
    \frac{\beta}{(1-\beta)}(s-1-2c_1) & \frac{(s-1-2c_1)\beta^2 + \alpha\beta \mu}{(1-\beta)^2}
    \end{bmatrix}\,.   $$
By the properties of the Kronecker product, the symmetric matrix $S_s$ has the same eigenvalues with the $2\times 2$ matrix $A_s$ and $S_s$ is positive semi-definite if and only if as $A_s$ is positive semi-definite. 
Therefore, $S_s$ is positive definite if and only if the principal minors of $A_s$ are non-negative, i.e.
$$
s - 1 - 2c_1 + \frac{2\alpha\mu L_\alpha}{(1-\beta)(L_\alpha+\mu)} \geq 0\,,
$$
and 
$$ 
\left(s - 1 - 2c_1 + \frac{2\alpha\mu L_\alpha}{(1-\beta)(L_\alpha+\mu)}\right) \left( \frac{(s-1-2c_1)\beta^2 +\alpha\mu\beta}{(1-\beta)^2}  \right)  \geq \left(\frac{\beta}{(1-\beta)}(s-1-2c_1)\right)^2\,.
$$ 
After some computations we observe that the last inequality is equivalent to 
$$
s - 1 - 2c_1 + \frac{2\alpha\mu L_\alpha}{(1-\beta)(L_\alpha+\mu) + 2L_\alpha \beta}\geq 0\,.
$$
We conclude that $S_s$ is positive semi-definite if and only if 
$$ 
s\geq 1 + 2c_1 - \frac{2\alpha\mu L_\alpha}{(1-\beta)(L_\alpha+\mu) + 2L_\alpha \beta} = 1 - \frac{\alpha\mu L_\alpha}{(1-\beta)(L_\alpha+\mu) + 2L_\alpha \beta}\,,
$$
where we used the definition \eqref{eq-c1} of $c_1$ in the last equality. This completes the proof.
\hfill $\Box$ 

\section{Discussions on Gradient Noise Assumptions}\label{sec:gradient:noise:assump}

\rev{In our analysis, we assumed that the variance of the gradient noise is bounded (Assumption~\ref{assumption}). 
This is a reasonable assumption since it can be shown that if the stepsize $\eta>0$ is small enough the variance of the gradients will stay bounded and satisfy our assumptions on the gradient noise (Assumption~\ref{assumption}) with an analysis similar to \citet[Section K]{universally-optimal-sgd}. We can illustrate this point in detail as follows.} 

\rev{Consider a more general gradient noise setting than Assumption~\ref{assumption}:
    \begin{equation}\label{cond-noise}
        \mathbb{E} \left[\tilde\nabla f_i(x) - \nabla f_i(x) \Big| x \right] = 0, 
        \qquad
         \mathbb{E} \left[\left\|\tilde\nabla f_i(x) - \nabla f_i(x) \right\|^2\Big| x\right]  \leq C\left(1+ \|x\|^2\right), 
    \end{equation}
(see e.g. \cite{jain2018accelerating}) where $C$ is a positive constant.
The assumption \eqref{cond-noise} is satisfied for a wide class of $f_{i}$ functions
when gradients are estimated over mini-batches.
Consider the linear regression example in the empirical risk minimization setting, where the stochastic gradients $\tilde\nabla f_i(x)$ are estimated from mini-batches of size $b$ at a point $x$, i.e. 
$$\tilde\nabla f_i(x) = \frac{2n_i}{b}\sum_{k=1}^b (y_{j_k}^i-x^TX_{j_k}^i) + \frac{1}{ \lambda N}x, $$
where $j_1, j_2, \dots, j_b$ are selected uniformly random with replacement over the index set $\{1,2.\dots,n_i\}$ of the data points where $n_i$ are finite and fixed.  
In this setting, it is well-known that the gradient error satisfies \eqref{cond-noise}.
The $L_2$-regularized logistic regression case will be similar.}

\rev{In the following, we will show that for DE-SGLD, when the stepsize $\eta$ is sufficiently small, 
the assumption \eqref{cond-noise} implies that:
    \begin{equation}\label{cond-noise-2}
            \mathbb{E} \left[\tilde\nabla f_i(x) - \nabla f_i(x) \Big| x \right] = 0, 
            \qquad
         \mathbb{E} \left\|\tilde\nabla f_i(x) - \nabla f_i(x)\right\|^2 \leq \sigma^2, 
    \end{equation}
for some $\sigma>0$, which is the precisely the assumption we had for gradient noise (Assumption~\ref{assumption})
and hence our main result (Theorem~\ref{thm:overdamped}) holds
under the assumption \eqref{cond-noise}. This is primarily because the second moments of the iterates are uniformly bounded and taking expectation with respect to $x$ in \eqref{cond-noise} would result in a condition like \eqref{cond-noise-2}. To see this in more detail, we recall that in the proof of Lemma~\ref{lem:0}, by assuming 
$\mathbb{E} \left\|\tilde\nabla f_i(x) - \nabla f_i(x)\right\|^2 \leq \sigma^2$, 
as in Assumption~\ref{assumption}, we had
\begin{equation}\label{instead:eqn}
\mathbb{E}\left[\left\Vert x^{(k+1)}-x_{\eta}^{\ast}\right\Vert^{2}\right]
\leq\left(1-\mu\eta\left(1+\lambda_{N}^{W}-\eta L\right)\right)\mathbb{E}\left[\left\Vert x^{(k)}-x_{\eta}^{\ast}\right\Vert^{2}\right]
+\eta^{2}\sigma^{2}N+2\eta dN,
\end{equation}
provided that the stepsize $\eta$ is sufficiently small, 
where $x^{(k)}=\left[(x_{1}^{(k)})^{T},\ldots,(x_{N}^{(K)})^{T}\right]^{T}$, 
and $x_{\eta}^{\ast}$ is the minimizer of $F_{\mathcal{W},\eta}(x)=\frac{1}{2\eta}x^{T}(I-\mathcal{W})x+F(x)$,
where $F(x)=\sum_{i=1}^{N}f_{i}(x_{i})$. 
Now, if we assume \eqref{cond-noise}, instead of \eqref{cond-noise-2}, we will obtain
\begin{align}
\mathbb{E}\left[\left\Vert x^{(k+1)}-x_{\eta}^{\ast}\right\Vert^{2}\right]
&\leq\left(1-\mu\eta\left(1+\lambda_{N}^{W}-\eta L\right)\right)\mathbb{E}\left[\left\Vert x^{(k)}-x_{\eta}^{\ast}\right\Vert^{2}\right]
\nonumber
\\
&\qquad\qquad\qquad
+\eta^{2}CN+\eta^{2}C\mathbb{E}\left[\left\Vert x^{(k)}\right\Vert^{2}\right]+2\eta dN\nonumber
\\
&\leq\left(1-\mu\eta\left(1+\lambda_{N}^{W}-\eta L\right)\right)\mathbb{E}\left[\left\Vert x^{(k)}-x_{\eta}^{\ast}\right\Vert^{2}\right]
\nonumber
\\
&\qquad
+\eta^{2}CN+2\eta^{2}C\mathbb{E}\left[\left\Vert x^{(k)}-x_{\eta}^{\ast}\right\Vert^{2}\right]
+2\eta^{2}C\left\Vert x_{\eta}^{\ast}\right\Vert^{2}
+2\eta dN.
\nonumber
\end{align}
By Lemma~\ref{lem:C1:gamma}, for sufficiently small stepsize $\eta>0$,
$\Vert x_{\eta}^{\ast}-x^{\ast}\Vert\leq C_{1}\frac{\eta\sqrt{N}}{1-\bar{\gamma}}$, where $C_{1},\bar{\gamma}$ are constants defined in \eqref{def-C1} and \eqref{eq:gamma}, and 
we recall from \eqref{eq:x-ast-ND} that $x^{\ast}=\left[x_{\ast}^{T},\ldots,x_{\ast}^{T}\right]^{T}$ where $x_*$ is the minimizer of $f(x) = \sum_{i=1}^N f_i(x)$ which is unique by strong convexity. Therefore, we obtain
\begin{align}
&\mathbb{E}\left[\left\Vert x^{(k+1)}-x_{\eta}^{\ast}\right\Vert^{2}\right]
\nonumber
\\
&\leq\left(1-\mu\eta\left(1+\lambda_{N}^{W}-\eta L\right)\right)\mathbb{E}\left[\left\Vert x^{(k)}-x_{\eta}^{\ast}\right\Vert^{2}\right]
\nonumber
\\
&\qquad
+\eta^{2}CN+2\eta^{2}C\mathbb{E}\left[\left\Vert x^{(k)}-x_{\eta}^{\ast}\right\Vert^{2}\right]
+4\eta^{2}C\left\Vert x^{\ast}\right\Vert^{2}
+4\eta^{4}CC_{1}^{2}\frac{N}{(1-\bar{\gamma})^{2}}
+2\eta dN
\nonumber
\\
&\leq\left(1-\frac{1}{2}\mu\eta\left(1+\lambda_{N}^{W}-\eta L\right)\right)\mathbb{E}\left[\left\Vert x^{(k)}-x_{\eta}^{\ast}\right\Vert^{2}\right]
\nonumber
\\
&\qquad\qquad\qquad\qquad\qquad
+\eta^{2}CN
+4\eta^{2}C\left\Vert x^{\ast}\right\Vert^{2}
+4\eta^{4}CC_{1}^{2}\frac{N}{(1-\bar{\gamma})^{2}}
+2\eta dN,
\end{align}
for sufficiently small $\eta$. This implies that
for sufficiently small $\eta$, such that
$\frac{1}{2}\mu\eta(1+\lambda_{N}^{W}-\eta L)\in(0,1)$, 
we have the following uniform $L_{2}$
bound:
\begin{equation}
\mathbb{E}\left[\left\Vert x^{(k)}-x_{\eta}^{\ast}\right\Vert^{2}\right]
\leq\tilde{C}_{1},\qquad\text{for any $k\in\mathbb{N}$},
\end{equation}
for some constant $\tilde{C}_{1}>0$. 
Finally, by taking the expectation
w.r.t. $x$ in \eqref{cond-noise} and applying
the tower property, we conclude that
\eqref{cond-noise-2} holds for some $\sigma>0$.
Hence, the assumption \eqref{cond-noise} implies
the assumption \eqref{cond-noise-2} which is used in Assumption~\ref{assumption}.
This argument shows that our assumption on the finiteness of gradient noise variance is satisfied
when data is sampled with mini-batches such that \eqref{cond-noise} holds.}

\section{Discussions on the Lipschitz Gradient Assumption}\label{sec:assump:numerical}

\rev{In our analysis, we consider sampling from the target distribution with density $\pi(x) \propto e^{-f(x)} = e^{-\sum_{i=1}^N f_i(x)}$, where $f_i$ is the loss function of the agent $i$ for $i=1,2,\dots, N$. We assume that the gradients $\nabla f_i(x)$ are (uniformly) Lipschitz with some Lipschitz constant $L$. In the linear regression example in Section~\ref{sec:numerical}, we consider the empirical risk minimization setting, where the number of data points $n_i$ that agent $i$ possesses is finite and the dataset is given and fixed. The Lipschitz constant $L$ will in general depend on the dataset, 
but will be finite as long as the number of data points $n$ is finite. For example, in the case of linear regression, we have 
\begin{equation} 
f_i(x) =  \sum\nolimits_{j=1}^{n_i}\left(y_j^i-x^TX_j^i\right)^2 + \frac{1}{2 \lambda N} \|x\|^2\,,
\end{equation}
where agent $i$ possesses a dataset $\mathcal{D}_i :=\{(X_j^i,y_j^i)\}_{j=1}^{n_i}$ of $n_i$ data points. The Hessian of $f_i$ satisfies
$$\nabla^2 f_i(x) = 2 \sum\nolimits_{j=1}^{n_i}X_j^i \left(X_j^i\right)^T + \frac{1}{\lambda N} I\,,$$
where $I$ is the identity matrix. Therefore, Hessian of $f_i$ is uniformly bounded satisfying $\|\nabla^2 f_i(x)\| \leq 2 \sum_{j=1}^{n_i} \|X_j^i\|^2 + \frac{1}{\lambda N}$. Therefore, we can take the Lipschitz constant to be 
    $$ L = 2 \max_{i=1,2,\dots, N} \left(  \sum\nolimits_{j=1}^{n_i} \left\|X_j^i\right\|^2 \right)  + \frac{1}{\lambda N}, $$
and this constant is finite because the number of samples $n =\sum_{i=1}^{N} n_i$ is finite and the data points $X_j^i$ are given and fixed. This is the setting considered in our paper, and therefore our uniformly Lipschitz assumptions are satisfied.}

\rev{More generally, one can try to bound $L$ almost surely, i.e. for almost every realization of the dataset. If we take $X_j^i$ to be random without a compact support (i.e. when data is Gaussian), then $L$ will not be bounded almost surely. 
However, if the input data is bounded (which can often hold in machine learning practice naturally after normalizing/preprocessing data if necessary), then we will have $L$ finite almost surely and our Lipschitz assumption will hold for almost every realization of the dataset. By similar computations, we can have the same conclusions for logistic regression. 
In other words, in the empirical risk minimization setting that we consider when each agent has finitely many data points and the dataset is fixed, our uniform Lipschitz gradient assumption will hold although the Lipschitz constant $L$ will depend on the dataset. If we assume further that data has compact support, our Lipschitz assumption will hold almost surely.}

\rev{It is worth noting that the recent elegant approach in \cite{Barkhagen2021} applies even if the data does not have compact support and when $n$ goes to infinity and can handle non-i.i.d. $L$-mixing data streams. 
However, \cite{Barkhagen2021} considers the centralized setting and distributed setting is not discussed. It is not clear how to apply their techniques to the distributed setting but this would definitely be an interesting future research direction. Also, \cite{Barkhagen2021} does not discuss the stochastic gradient Hamiltonian Monte Carlo case, whereas our analysis framework provides a uniform approach where we study the stochastic gradient Hamiltonian Monte Carlo as well in the distributed setting.
}

\bibliography{langevin}

\end{document}